%% file: paper.tex
\documentclass[smallextended,referee,envcountsect,]{svjour3}
\smartqed
\date{}

\usepackage[utf8]{inputenc} 
\usepackage[T1]{fontenc}    
\usepackage{url}            
\usepackage{booktabs}       
\usepackage{amsfonts}       
\usepackage{nicefrac}       
\usepackage{microtype}      
\usepackage{braket}
\usepackage{amsmath}
\usepackage{thmtools}
\usepackage{amssymb}
\usepackage{color}
\usepackage{enumerate}

\usepackage[pdftex]{graphicx}
\usepackage{listings}
\usepackage{multirow}
\usepackage{rotating}
\usepackage[ruled,algosection]{algorithm2e}
\usepackage{algpseudocode}
\usepackage[square,numbers]{natbib}

\usepackage{chngcntr}
\usepackage{apptools}
\usepackage{tabularx}

\LinesNumbered 
\SetAlgoNoLine

\makeatletter
\renewcommand{\@algocf@capt@plain}{above}
\makeatother

\newtheorem{thm}{Theorem}[section]
\newtheorem{asu}[thm]{Assumption}
\newtheorem{cor}[thm]{Corollary}
\newtheorem{prop}[thm]{Proposition}

\newtheorem{lem}[thm]{Lemma}

\spnewtheorem*{namedprop}{Proposition}{\bfseries}{\itshape}
\spnewtheorem*{namedthm}{Theorem}{\bfseries}{\itshape}
\spnewtheorem*{namedlemma}{Lemma}{\bfseries}{\itshape}
\spnewtheorem*{namedcor}{Corollary}{\bfseries}{\itshape}

\input{mycmds.tex}

\newenvironment{pf}{\begin{proof}}{%
  \end{proof}\ignorespacesafterend
}

\newenvironment{eq}{\begin{equation}}{%
  \end{equation}\ignorespacesafterend
}

\newcommand{\addr}[1]{#1}
\newcommand{\name}[1]{#1}

\title{
   The
  duality structure gradient descent
  algorithm: analysis and applications to neural networks.  
}
\authorrunning{Thomas Flynn}
\titlerunning{Duality Structure Gradient Descent}

\author{Thomas Flynn}
\institute{
  Thomas Flynn \at
  \addr Brookhaven National Laboratory\\
  Computational Science Initiative \\
        Upton, NY 11973, USA \\
        \email tflynn@bnl.gov }

\begin{document}
\maketitle
\begin{abstract}
   The training of machine learning models is typically carried out using some form of gradient descent, often with great success. However, non-asymptotic analyses of first-order optimization algorithms typically employ a gradient smoothness assumption (formally, Lipschitz continuity of the gradient) that is too strong to be applicable in the case of deep neural networks. To address this, we propose an algorithm named duality structure gradient descent  (\textproc{dsgd})
  that is amenable to non-asymptotic performance analysis, under mild assumptions on the training set and network architecture. The algorithm can be viewed as a form of layer-wise coordinate descent, where at each iteration the algorithm chooses one layer of the network to update. The decision of what layer to update is done in a greedy fashion, based on a rigorous lower bound on the improvement of the objective function for each choice of layer. In the analysis, we bound the time required to reach approximate stationary points, in both the deterministic and stochastic settings. The convergence is measured in terms of a parameter-dependent family of norms that is derived from the network architecture and designed to confirm a smoothness-like property on the gradient of the training loss function. We empirically demonstrate the behavior of \textproc{dsgd} in several neural network training scenarios. 
\end{abstract}
\section{Introduction}
Gradient descent and its variants are often used to train machine learning models, and these algorithms have led to impressive results in many different applications.
These include training neural networks for tasks such as representation learning 
\citep{nlnca,lee2009}, image classification \citep{hintonscience,imagenet}, scene labeling \citep{farabet2013}, and multimodal signal processing \citep{srivastava2014multimodal}, just to name a few.
In each case, these systems employ some form gradient based optimization, and the algorithm settings must be carefully tuned to guarantee success.
For example, choosing small step-sizes leads to slow optimization, and step-sizes that are too large result in unstable algorithm behavior.
Therefore it would be useful to have a theory that provides a rule for the step-sizes and other settings that guarantees the success of optimization. 
     
  A common assumption in analyses of gradient descent for non-convex functions, some of which are reviewed below in Section \ref{sect:rel-wrk},  is Lipschitz continuity of the gradient of the objective function, or related criteria. This assumption, also termed ``gradient smoothness'' or simply ``smoothness'', typically states that the objective function has a Lipschitz gradient. Assuming the function is twice continuously differentiable, this is equivalent to requiring that the second derivative is bounded. For example, if the objective has a Lipschitz continuous gradient, with Lipschitz constant $L$, then gradient descent with a step-size of $1/L$ is guaranteed to find approximate stationary points \cite[Section~1.2.3]{nesterov2013introductory}. However, it is doubtful that this approach to the analysis can be applied to multi-layer neural networks, since there are very simple neural network training problems where the objective function does not have a Lipschitz continuous gradient. In Section \ref{funsect} we present an example problem of this type. 
     
     Our approach has three main components.
The starting point is a layer-wise Lipschitz property that is satisfied by the neural network optimization objective.  Motivated by this we design our algorithm to choose one layer of the network to update at each iteration, using a lower bound for the amount of function decrease that each choice of layer would yield.
The second component is an analytical framework based on parameter-dependent norms 
that is used to prove the convergence of our algorithm, and that we believe may also be of general interest.
Thirdly, the geometric point of view is not just a tool for analysis but offers flexibility, as a variety of algorithms with convergence guarantees can be described this way: by defining the search directions within each layer according to a possibly non-Euclidean norm one can generate a variety of different update rules, and these can all be accommodated in our analysis framework.
Combining these three components leads to an optimization procedure that is provably convergent for a significant class of multi-layer neural network optimization problems. Note that the focus of this paper is mostly theoretical, and we do not explore practical variants of DSGD in depth in the present work.
We  now describe these three components in more detail. 
\paragraph{Layer-wise Lipschitz property}     A  network with no hidden layers presents a relatively straight forward optimization problem (under mild assumptions on the loss function). Typically, the resulting objective function has a Lipschitz continuous gradient; for instance, this follows from Proposition \ref{hess-bd-prop} by restricting to a network with a single layer. When multiple layers are connected in the typical feed-forward fashion, the result is a hierarchical system that as a whole generally does not have a Lipschitz continuous gradient. This is rigorously established in Proposition \ref{prop:nn-hard} below. 
   However, if we focus our attention to only one layer of the multi-layer network, then the task is somewhat simplified. Specifically, consider a neural network with the weight matrices ordered from input to output as  $w_1,w_2,\hdots,w_{L}$. Then under mild assumptions, the magnitude of the second derivative (Hessian matrix) of the objective function restricted to the weights in layer $i$ can be bounded by a polynomial in the norms of the successive matrices
   $w_{i+1},\hdots,w_{L}$, which is a sufficient for Lipschitz continuity of the gradient. This is formalized below in Proposition \ref{hess-bd-prop}.
This fact can be used to infer a lower bound on the function decrease that will happen when taking a gradient descent step on layer $i$.
   By computing this bound for each possible layer
   $i \in \{1,\hdots, L\}$ we can choose which update to perform using the greedy heuristic of picking the layer that maximizes the lower bound. The pseudocode for the procedure is presented in Algorithm \ref{algo:finslernn} below.

\paragraph{Duality Structure Gradient Descent}
A widely used success criterion for non-convex optimization is that the algorithm yields a point where the Euclidean norm of the  derivative of the objective is small. This is motivated by the fact that points where the derivative is zero are stationary points. It appears difficult to establish this sort of guarantee in the situation described above, where one layer at a time is updated according to a greedy criteria. However, the analysis becomes simpler if we are willing to adjust the geometric framework used to define convergence.
In the geometry we introduce, the norm at each point in the parameter space is determined by the weights of the neural network, and
the convergence criterion we use is that the algorithm generates a point with a small derivative as measured using the local norm. 

Our analysis is based on a continuously varying family of norms that is designed in response to the structure of the neural network, taking into account
our bound on the Lipschitz constants, and our greedy ``maximum update'' criterion.
The family of norms encodes our algorithm in the sense that one step of the algorithm  corresponds to  taking a step in the steepest descent direction as defined by the family of norms. The steepest descent directions are computed by solving a secondary optimization problem at each iteration, in order to identify the layer maximizing the lower bound on the function decrease. Formally, a \textit{duality structure} represents the solutions to these sub-problems.  

\paragraph{Intralayer update rules}
A third component of our approach, which turns out to be key to obtaining an algorithm that is not only theoretically convergent but also effective in practice, is to consider the geometry within each layer of the weight matrices. Typically in first order gradient descent, the update direction is the vector of partial derivatives of the objective function. This can be motivated using Taylor's theorem:
if it is known that the spectral norm of the Hessian matrix of a given function
$f$
is bounded by a constant
$L$,
then Taylor's theorem provides a quadratic upper bound for the objective of the form
$f(w-\Delta) \leq f(w) - \tfrac{\partial f}{\partial w}(w)\cdot\Delta + \tfrac{L}{2}\|\Delta\|_{2}^{2}$,
and setting
$\Delta = \tfrac{1}{L}\tfrac{\partial f}{\partial w}(w)$
results in a function decrease of magnitude at least
$\tfrac{1}{2L}\|\tfrac{\partial f}{\partial w}(w)\|_2^{2}$.
Using a different norm when applying Taylor's theorem
results in a different quadratic upper bound, and 
a general theorem about gradient descent for arbitrary norms is stated in
Proposition \ref{simple-gd-lem}.
The basic idea is that if
$\|\cdot\|$
is an arbitrary norm and
$L$
is a global bound on the norm of the bilinear maps
$\tfrac{\partial^{2} f}{\partial w^{2}}(w)$,
as measured with respect to
$\|\cdot\|$,
then
$f$ satisfies a quadratic bound of the form
$f(w-\Delta) \leq f(w) - \tfrac{\partial f}{\partial w}(w)\cdot\Delta + \frac{L}{2}\|\Delta\|^{2}$.
Using the notion of a duality map $\rho$ for the norm $\|\cdot\|$ (see Equation \eqref{dmapprops} for a formal definition), the update 
$\Delta = \frac{1}{L}\rho(\frac{\partial f}{\partial w}(w))$
leads to a decrease of magnitude at least
$ \frac{1}{2L} \|\tfrac{\partial f}{\partial w}(w)\|^2$.
For example, when the argument
$w$
has a matrix structure and
$\|\cdot\|$
is the spectral norm, then the update direction is a spectrally-normalized version of the Euclidean gradient, in which the matrix of partial derivatives has its non-zero singular values set to unity, a fact that is recalled in Proposition \ref{one-space-dual} below.
 The choice of norm for the weights can be encoded in the overall family of norms, and each norm leads to a different provably convergent variant of the algorithm.  In our experiments we considered update rules based on the matrix norms induced by $\|\cdot\|_q$ for $q=1,2$, and $\infty$.

Despite the possible complexity of the family of norms, the analysis is straight forward and mimics the standard proof of convergence for Euclidean gradient descent. In the resulting convergence theory, we study how quickly the norm of the gradient tends to zero, measured with respect to the local norms 
   $\|\cdot\|_{w(t)}$.
   Roughly speaking, the quantity that is proved to tend to zero is 
   $\|\tfrac{\partial f}{\partial w}(w(t))\|/p(\|w(t)\|)$,
   where
   $\|w(t)\|$ is the norm of the network parameters and
   $p$
   is an polynomial that depends on the architecture of the neural network. 
   This is in contrast to the usual Euclidean non-asymptotic performance analysis, which tracks the gradient measured with respect to a fixed norm, that is, 
   $\|\frac{\partial f}{\partial w}(w(t))\|$. See Proposition \ref{hess-bd-prop} and the discussion following it for more details on how the local norms are defined in the case of neural networks.  To finish this section, we would like to comment on the applicability of our results. We consider the case of multi-layer neural networks that use activation functions that are bounded and differentiable (up to 2nd order). This excludes neural networks with rectified linear units (ReLU), which constitute a popular class of models. The main reason for this is that ReLu units do not have second derivatives at every point in their domain, which is a requirement in our theory.
 
\subsection{Outline}
After reviewing some related work, in Section \ref{funsect} we present an example of a neural network training problem where the objective function does not have bounded second derivatives. In Section \ref{sect:gd} we introduce the abstract duality structure gradient descent (\textproc{dsgd}) algorithms and the convergence analyses.
The main result in this section is
Theorem \ref{prop:sgd-prop},
concerning the expected number of iterations needed to reach an approximate stationary point in \textproc{dsgd}.
Corollaries \ref{non-asympt-sgd} and \ref{non-asympt-gd}
consider special cases, including batch gradient descent and that of a trivial family of norms, in the latter case recovering the known rates for for standard stochastic gradient descent.  In Section \ref{sect:nnapp} we show how \textproc{dsgd} may be applied to neural networks with multiple hidden layers. The main results in this section are convergence analyses for the neural network training procedure presented in Algorithm \ref{algo:finslernn}, both in the deterministic case
(Theorem \ref{main-thm-app})
and a corresponding analysis for the mini-batch variant of the algorithm
(Theorem \ref{main-thm-app-sgd}).
Numerical experiments on the MNIST, Fashion-MNIST, CIFAR-10, and SVHN benchmark data
sets are presented in Section \ref{sect:num}.  We finish with a discussion in Section \ref{sect:discuss}.
Several proofs are deferred to an appendix. 
\subsection{Related work\label{sect:rel-wrk}}
There are a number of performance analyses of gradient descent for non-convex functions which utilize the assumption that one or more  higher derivatives are  bounded. Although we are specifically concerned with non-convex optimization, it is worth mentioning that \textproc{sgd} for convex functions is more well understood, and there are analyses that bypass the smoothness assumption \citep{sgd-and-hogwild, bachmoulines11}. 

For non-convex optimization, gradient-descent using a step-size proportional to $1/L$ achieves a convergence guarantee on the order of  $1/T$, where $T$ is the running time of the algorithm  \cite[Section~1.2.3]{nesterov2013introductory}.
Note the inverse relationship between the Lipschitz constant $L$ and the step-size $1/L$, which is characteristic of results that rely on a Lipschitz property of the gradient for non-asymptotic analysis.
 Most practical algorithms in machine learning are stochastic variants of gradient descent.  The Randomized Stochastic Gradient (RSG) algorithm  is one such example \citep{ghadimi-lan}. In RSG, a stochastic gradient update is run for $T$ steps,  and then a random iterate is returned.
In  \citep{ghadimi-lan} it was proved that the expected squared-norm  of the returned gradient tends to zero at rate of $1/\sqrt{T}$.
 Their assumptions include a Lipschitz gradient and uniformly bounded variance of gradient estimates. 
 
A variety of other, more specialized algorithms have also been analyzed under the Lipschitz-gradient assumption.
The Stochastic Variance Reduced Gradient (SVRG) algorithm combines features of deterministic and stochastic gradient descent, alternating between full gradient calculations and \textproc{sgd} iterations \citep{svrg}. Notably, it was shown that SVRG for non-convex functions requires fewer gradient evaluations on average compared to RSG \citep{zeyuan-svrg} \citep{reddi2016stochastic}. The step-sizes follow a $1/L$ rule, and the variance assumptions are weaker compared to RSG. For machine learning on a large scale, distributed and decentralized algorithms become of interest. Decentralized \textproc{sgd} was analyzed in \citep{decentralized}, leading to a  $1/L$-type result for this setting. 

Adaptive gradient methods, including Adagrad \citep{duchi2011adaptive},
RMSProp \citep{Tieleman2012} and ADAM \citep{kingma2014adam}
  define another important variant of gradient descent. These methods update learning rates on the fly based on the trajectory of observed (possibly stochastic) gradients. Convergence bounds for Adagrad-style updates in the context of non-convex functions have recently been derived \citep{li2018convergence, ward19a, jin2022on}. A key difference  between these adaptive gradient methods and our work, aside from the relaxation of the smoothness assumption that we pursue, is that in \textproc{dsgd}, gradients are scaled by the norm of the iterates, rather than the sum of the norms of the gradients. High probability bounds for AdaGrad style algorithms have also been derived \citep{kavis2022high}. Another form of adaption is clipping, whereby updates are rescaled if their magnitude is too large. Convergence rates for Clipped \textproc{gd} and Clipped \textproc{sgd} have recently been derived in \citep{Zhang2020Why}. It was shown that Clipped \textproc{gd} converges for a broader class of functions than those having a Lipschitz gradient. However, it is not clear if their generalized smoothness condition holds in the setting of deep neural networks.   Clipped \textproc{sgd} also admits convergence guarantees in the case of convex functions with rapidly growing subgradients, and weakly convex functions  \citep{mai21a}. While the class of weakly convex functions includes those with a Lipschitz gradient, it is too restrictive to include neural networks with multiple layers.
  
One approach to extend the results on gradient descent is to augment or replace the assumption on the second derivative with an analogous assumption on third order derivatives. In an analysis of cubic regularization methods, \citet{cartis2011adaptive}  proved a bound on the asymptotic rate of convergence for non-convex functions that have a Lipschitz-continuous Hessian.  In a non-asymptotic analysis of a trust region algorithm in \citep{curtis2017trust}, convergence was shown to points that approximately satisfy a second order optimality condition, assuming a Lipschitz gradient and Lipschitz Hessian.  

A natural question is whether these results can be generalized to exploit the Lipschitz properties of derivatives of arbitrary order. This question was taken up by \citet{Birgin2017},  where it is assumed that the derivative of order $p$ is Lipschitz continuous, for arbitrary $p\geq 1$ .
They consider an algorithm that constructs a $p+1$ degree polynomial majorizing the objective at each iteration, and the next iterate is obtained by approximately minimizing this polynomial.
The algorithm in a sense generalizes first order gradient descent and well as cubic regularization methods.
A remarkable feature of the analysis is that the convergence rate improves as $p$ increases. Note that the trade off is that higher values of $p$ lead to subproblems of minimizing potentially high degree multivariate polynomials. 

Another approach to generalizing smoothness assumptions uses the concept of \textit{relative smoothness}, defined by  \citet{relativelysmooth} and closely related to the condition $LC$ proposed by \citet{beyondlip}. Roughly speaking, a function $f$ is defined to be relatively smooth relative to a reference function $h$ if the Hessian of $f$ is upper bounded by the Hessian of $h$ (see Proposition 1.1 in \citep{relativelysmooth}.) In the optimization procedure, one solves sub-problems that involve the function $h$ instead of $f$, and if $h$ is significantly simpler than $f$ the procedure can be practical.
A non-asymptotic convergence guarantee is established under an additional relative-convexity condition.   We note that recent works have explored relative smoothness in non-convex  \citep{bolte} \citep{Bauschke2019} and stochastic settings \citep{davis2018stochastic} \citep{zhang2018convergence}.
 Our work in this paper is also concerned with generalized gradient smoothness condition; however, there are two primary differences. Firstly, our primary assumption (Assumption \ref{f-asu}) differs in that it does not require bounding the action of the Hessian on all possible update directions, but only in those directions relevant to the algorithm update steps (this is manifested by the presence of the duality mapping in the criterion). We posit that this enables step-sizes that are less conservative. In addition, motivated by applications to neural networks, in our analysis of the stochastic settings we confirm that the minibatch gradients satisfy a generalized bound on the variance as well. 

In this work we utilize the notion of a continuously varying family of norms in the convergence analysis of \textproc{dsgd}, ideas that are also used in variable metric methods \citep{davidon, davidon1991} and optimization on manifolds more generally. 
Notable instances of optimization on manifolds include optimizing over spaces of structured matrices \citep{absil2009optimization}, and parameterized probability distributions, as in information geometry \citep{amari1998natural}.
In the context of neural networks, natural gradient approaches to optimization have been explored \citep{kurita1993,amari1998natural}, and recently \citet{yann} considered some practical variants of the approach, while also extending it to networks with multiple hidden layers.

We note that several heuristics for step-size selection in the specific case of gradient descent for neural networks have been proposed, including \citep{pesky,duchi2011adaptive,kingma2014adam}, but the theoretical analyses in these works is limited to convex functions.  Other heuristics include  forcing Lipschitz continuity of the gradient by constraining the parameters to a bounded set, for instance using weight clipping, although this leads to the problem of how to choose an appropriate bounded region, and how to determine learning rate and other algorithm settings based on the size of this region. Finally, we point out that in certain cases one can establish that SGD converges to a local minimum, rather than a stationary point, given an initial point in a small enough neighborhood to such a minimum \citep{sgdminimum}.

 \begin{table}[!htb]\label{tbl:notation}
   \caption{Notation}

   \begin{minipage}{.49\linewidth}

      \begin{tabular}{|ll|}
        \hline
            $f$& an objective function\\
$f^*$ & a lower bound on values of $f$ \\
$w$& optimization variable\\
$n$ & dimension of parameter space\\
$t$& iteration number \\
$\mc{L}(\reals^d,\reals)\hspace{-0.5em}$& the linear maps from $\reals^d$ to  $\reals$\\
$\ell$& an element in $\mc{L}(\reals^d,\reals)$\\
$\epsilon$& step-size for optimization\\
$g(t)$ & an approximate derivative\\
$\delta$ & error of a derivative estimate\\
$L$& Lipschitz-type constant\\
$K$& number of layers in a network\\
$n_{k}$& number of nodes in layer $k$\\
$y^{k}$& state of layer $k$ of network\\
$x$& input to  network \\
        $z$& output target for a network \\
             \hline
        \end{tabular}
    \end{minipage}%
    \begin{minipage}{.5\linewidth}
      \begin{tabular}{|ll|}
\hline
        $m$& number of items in a training set\\
        $f_{i}$& loss function for training example $i$\\
$\|\cdot\|$& a norm\\
$\|\cdot\|_{w}\hspace{-5em}$& norm depending on a parameter $w$\\
$\rho$& a duality map\\
$\rho_{w}$& duality map as a function of $w$\\
$h$& the function computed by a  layer\\
$q$& a choice of norm in $\{1,2,\infty\}$\\
$\tr$& trace of a matrix\\
$b$& batch size\\
$B(t)$& items in the batch at time $t$\\
$T$& final iterate of  optimization\\
$r,v,s\hspace{-5em}$& functions used to bounds derivatives\\
$J$& loss function used for network\\
$\sgn$& $\sgn(x) = 1_{x\geq 0} -1_{x<0}$ \\
        $w_{1:k}$& a subvector $w_{1:k} = (w_{1},\hdots,w_k)$\\
        \hline
        \end{tabular}
      \end{minipage}
     
  \end{table}
  \paragraph{Notation}
The notation we will use is listed  in Table \ref{tbl:notation}.
In addition,   we will use the following definitions and conventions.
Given two linear maps 
$A_1: Z \to U$ and 
$A_2: Z \to U$, the direct sum 
$A_1 \oplus A_2$ is the linear map from 
$Z\times Z$ to 
$U \times U$ that maps a vector 
$(z_1, z_2)$ to 
$(A_1z_1, A_2z_2)$. $C(x,y)$ is the  result of applying bilinear map $C$ to the vectors $x,y$; in terms of components, $C(x,y) = \sum_{i=1}^{n}\sum_{j=1}^{m}C_{i,j}x_iy_j$.
The norm of a bilinear map $C:X \times Y\to Z$, is 
$\|C\| = \sup_{\|x\|_{X}=\|y\|_{Y}=1}\|C(u_1,u_2)\|_{Z}$.
If
$\ell :\reals^n\to\reals$
is a linear functional, then we represent the value of $\ell$ at the point
$u\in\reals^n$ by $\ell\cdot u$, following the notation used by \citet{abraham2012manifolds}. The derivative of a function
$f : \reals^n \to \reals^m$ at point
$x_0 \in \reals^n$ is a linear map from $\reals^n$ to $\reals^m$, denoted by
$\tfrac{\partial f}{\partial x}(x_0)$.
The result of applying this linear map to a vector $u \in \reals^n$ is a vector in $\reals^m$ denoted
$\tfrac{\partial f}{\partial x}(x_0)\cdot u$.
The second derivative of a function $f$ at $x_0$ is a bilinear map from $\reals^n\times\reals^n$ to $\reals^m$, denoted by
$\tfrac{\partial^2 f}{\partial x^2}(x_0)$, and we use the notation
$\tfrac{\partial^2 f}{\partial x^2}(x_0)\cdot (u,v)$ to represent the $\reals^m$-valued result of applying this bilinear map to the pair of vectors $(u,v)$.

\section{Motivating example}\label{funsect}
 For completeness,  this section details   a simple neural network training problems where the gradient of the objective function does not have a Lipschitz gradient.
Consider the  network depicted in Figure \ref{fig:nn}. This network maps a real-valued input to a hidden layer with two nodes and produces a real-valued output. Suppose that the sigmoid activation function
$\sigma(u) = 1/(1+e^{-u})$
is used, so that the function computed by the network is
\begin{equation}\label{nn-fn-comp}
  y(w_1,w_2,w_3,w_4;x) = \sigma(w_3\sigma(w_1 x) + w_4\sigma(w_2 x) )
  \end{equation}
Consider training the network to map the input
$x=1$
to the output
$0$,
using a squared-error loss function. The optimization objective $f:\reals^4\to\reals$ will be
\begin{equation}\label{opt-obj}
f(w_1,w_2,w_3,w_4) = |y(w_1,w_2,w_3,w_4;1)|^{2}.
\end{equation}
Proposition \ref{prop:nn-hard}  establishes that $f$ does not have bounded second derivatives, and hence cannot have a  Lipschitz continuous gradient.
\begin{prop}\label{prop:nn-hard}
The function $f$ defined in Equation \eqref{opt-obj} has unbounded second derivatives:
        $\sup_{w\in\reals^4}\|\tfrac{\partial^2 f}{\partial w^{2}}(w)\| = \infty.$
\end{prop}
\begin{figure}
\centering
   \begin{tikzpicture}[
            > = stealth, 
            shorten > = 1pt, 
            auto,
            node distance = 2.5cm, 
            semithick 
        ]

        \tikzstyle{every state}=[
            draw = black,
            thick,
            fill = white,
            minimum size = 17pt
        ]
        \node[state,draw=none] (x) {$x$};
        \node[state] (y1) [above right of=x] {};
        \node[state] (y2) [above left of=x] {};
        \node[state] (z) [above left of=y1] {};

        \path[->] (x) edge node[swap] {$w_1$} (y1);
        \path[->] (x) edge node {$w_2$} (y2);
        \path[->] (y1) edge node[swap] {$w_3$} (z);
        \path[->] (y2) edge node {$w_4$} (z);
    \end{tikzpicture}
    


    \caption{The small network used as a motivating example in Section \ref{funsect}. We show that the training problem of mapping the input $1$ to the output $0$, using the logistic activation function and squared-error loss, leads to an objective where the gradient is not Lipschitz continuous. \label{fig:nn}}
\end{figure}
The proof is deferred to an appendix.

A consequence of this proposition is that analyses assuming the objective function has a Lipschitz gradient  cannot be used to guarantee the convergence of gradient descent for this (and related) functions. 
Intuitively, when the parameters tend towards regions of space where the second derivative is larger, the steepest descent curve could be changing direction very quickly, and this means  first-order methods may have to use ever smaller step-sizes to avoid over-stepping and increasing the objective function. Note our example can be extended to show that third and higher-order derivatives of the objective $E$ are also not globally bounded, and therefore convergence analysis that shift the requirement of a derivative bounded onto such higher-order  derivatives would also not be applicable. On the other hand, the \textproc{dsgd} algorithm introduced below allows us to prove convergence for a variety of neural network training scenarios, including the function  \eqref{opt-obj}.
Finally, the negative conclusion in Proposition \ref{prop:nn-hard} does not mean that  algorithms like stochastic gradient descent would fail in practice, but it does suggest that the theory would be needed to be extended in order analyze the convergence in the context of training neural networks.

\section{Duality Structure Gradient Descent}\label{sect:gd}
We begin by assuming there is a user-defined family of norms parameterized by elements of the search space. 
The norm of a vector $u$ at parameter $w$ is denoted by  $\|u\|_{w}$.  The family of norms is subject to a continuity condition:
\begin{asu}\label{continuous-norms}The function
$(w,u) \mapsto \| u \|_w$ is continuous on $\reals^n\times \mathbb{R}^{n}$. 
\end{asu}
 Intuitively, the Assumption stipulates  that two norms $\|\cdot\|_{w_1}$ and $\|\cdot\|_{w_2}$ should be similar if $w_1$ and $w_2$ are close. 
This implies that the family of norms defines a \textit{Finsler structure} on the search space \cite[Definition~27.5]{deimling1985nonlinear}. In the remainder we use the terminology ``family of norms'' to refer to any collection of norms satisfying  Assumption \ref{continuous-norms}.  An example that the reader may keep in mind is $\|u\|_w = \frac{1}{1+\|w\|_2}\|u\|_2$, where  $\|\cdot\|_2$ is the standard Euclidean norm on $\mathbb{R}^n$. 

The family of norms  induces a norm on the dual space
$\mathcal{L}(\reals^{n},\reals)$ at each $w \in \reals^n$; if $\ell \in
\mathcal{L}(\mathbb{R}^{n},\mathbb{R})$ then  
\begin{equation}\label{dual-norm-def}
\|\ell\|_{w} =
\sup_{\|u\|_{w}=1}\ell\cdot u. 
\end{equation}
It is the case that for any family of norms the dual norm map 
$(w,\ell) \mapsto \|\ell\|_{w}$
is continuous on
$\reals^n\times \mathcal{L}(\mathbb{R}^{n},\mathbb{R})$.
This follows from \citep[Proposition~27.7]{deimling1985nonlinear}.

A vector $u$ achieving the supremum in
Equation \eqref{dual-norm-def} always exists, as it is the maximum of a continuous function over a compact set.
We represent scaled versions of vectors achieving the supremum in \eqref{dual-norm-def} using a duality map: 
\begin{definition}\label{dualitymap-def}A \textit{duality map} at $w$ is a function  $\rho_{w} : \mathcal{L}(\mathbb{R}^{n},\reals) \to \reals^{n}$ such that for all $\ell \in \mc{L}(\reals^n,\reals)$,
\begin{subequations}\label{dmapprops}
\begin{align} 
\| \rho_{w}(\ell)\|_{w} &= \|\ell\|_{w} \label{dual-map-prop-a}, \\
\ell\cdot \rho_{w}(\ell) &= \|\ell\|^{2}_{w}. \label{dual-map-prop-b}
\end{align}
\end{subequations}
\end{definition}
If the underlying norm
$\|\cdot\|_{w}$
is an inner product norm, then it can be shown that there is a unique choice for the duality map at $w$.
In detail, let
$Q_w$
be the positive definite matrix such that 
$\|u\|_{w} = \sqrt{u \cdot (Q_wu)}$ for all vectors $u$.
Then the duality map for this norm is
$\rho_w(\ell) = Q_{w}^{-1}\ell.$ 
However, in general there might be more than one choice for the duality map. For instance, consider the norm $\|\cdot\|_{\infty}$ on $\reals^2$, and let $\ell$ be the linear functional $\ell(x_1,x_2) = x_1$. Then we could set $\rho_{\infty}(\ell)$ to be either of the vectors $(1,1)$ or $(1,-1)$, and in both cases properties \eqref{dual-map-prop-a}, \eqref{dual-map-prop-b} would be satisfied.  

A duality structure assigns a duality map to each $w \in \reals^n$: 
\begin{definition}\label{dual-struct-def}
A \textit{duality structure} is
a function 
$\rho: \reals^n \times \mathcal{L}(\mathbb{R}^{n},\reals)
\to \reals^{n}$
such that for all $w \in \reals^n$,
the function
$\rho_{w}  \in \mc{L}(\reals^n,\reals)$
satisfies the two properties \eqref{dual-map-prop-a} and \eqref{dual-map-prop-b}.
\end{definition}
The simplest family of norms is the one that assigns the Euclidean norm to each parameter. In this situation  the dual norm is also the Euclidean norm and the duality map at each point is simply the identity function. For the family of norms $\|u\|_w = (1+\|w\|_2)\|u\|_2$, the reader may verify that a duality structure is $\rho_{w}(\ell) = \frac{1}{(1+\|w\|_2)^2}\ell$.
Before continuing, let us consider a less trivial example. 
\begin{example}\label{less-trivial}
  Let $h:\reals\to\reals$ and $g:\reals\to\reals$ be continuous.
  Consider the following family of norms on $\reals^2$:
  $$
  \|(u,v)\|_{(x,y)}
  =
  \sqrt{1 + |h(y)|}\, |u| + \sqrt{1 + |g(x)|}\, |v|.$$
   As the function $(x, y, u, v) \mapsto \|(u, v)\|_{(x,y)}$ is continuous,
   this family of norms is a well defined.
   Denoting a linear functional on $\reals^2$ by $\ell = (\ell_1,\ell_2)$, it follows from Proposition \ref{duality-product} below that the dual norm is
$$
\|(\ell_1,\ell_2)\|_{(x,y)}
=
\max\left\{ \frac{|\ell_1|}{\sqrt{1+|h(y)|}}, \frac{|\ell_2|}{\sqrt{1+|g(x)|}} \right\}.
$$
and a duality map is
$$
\rho_{(x,y)}(\ell_1,\ell_2)
=
\begin{cases}
\left( \frac{\ell_1}{1+|h(y)|}, 0 \right)
&\text{ if }
 \frac{|\ell_1|}{\sqrt{1+|h(y)|}} \geq
 \frac{|\ell_2|}{\sqrt{1+|g(x)|}}, \\
\left(0, \frac{\ell_2}{1+|g(x)|} \right)
&\text{ else. }
\end{cases}
$$
This concludes the example. 
\end{example}
As the example of the norm $\|\cdot\|_{\infty}$ shows, there may be multiple duality structures that can be chosen for a  family of norms. In this case, any one of them can be chosen without affecting the convergence bounds of the algorithms.

Given the definition of duality structure, we can now explain the steps of the \textproc{dsgd} algorithm, shown in Algorithm \ref{algo:finsgd}. Each iteration of this algorithm uses an estimate $g(t)$ of the derivative.
 The algorithm computes the duality map on this estimate, and the result serves as the update direction. A step-size $\epsilon(t)$ determines how far to go in this direction. Note that for a trivial family of norms ($\|\cdot\|_w= \|\cdot\|_2$ for all $w$), the algorithm reduces to standard \textproc{sgd}.

Our analysis seeks to bound the expected number of iterations until the algorithm generates an approximate stationary point for the function $f$, measured relative to the local norms. Formally, for $\gamma >0$ we define the stopping time $\tau$  as the first time the gradient has a norm less than or equal to $\gamma$: 
\begin{equation}  \label{taudef}
\tau
=
\inf
\Set{
t \geq 1 \, | \Big\|\frac{\partial f}{\partial w}(w(t))\Big\|_{w(t)}^{2} \leq \gamma
},
\end{equation}
and the goal of our analysis is to find an upper bound for $\mathbb{E}[\tau]$.
In our definition of $\tau$, the magnitude of the gradient is measured relative to the local norms $\|\cdot\|_{w(t)}$. This criterion for success is a standard notion in the literature of optimization on manifolds (see Theorem 4 of \citep{boumal2016globalrates}, and Theorem 2 of \citep{riemannian-svrg}.)  Note that this stopping rule is not part of the algorithms described in this paper, but is an analytical tool to prove something about the behavior of the iterates generated by the algorithms.

Next, we describe the conditions on the function $f$ and the derivative estimates $g(t)$ that will be used in the convergence analysis.
The conditions on the objective $f$ are that the function is differentiable and obeys a quadratic bound along each ray specified by the duality map. 
\begin{asu}\label{f-asu}
The function $f :\rn \to \reals$ is continuously differentiable,  bounded from below by
 $f^{*} \in \reals$,
and there is an
$L\geq 0$
such that, for all
$w\in \rn$, all
$\eta \in \mc{L}(\reals^n,\reals)$,
and all $\epsilon \in \reals$,
\begin{displaymath}
\left|
f(w + \epsilon \rho_{w}(\eta)) - f(w) - \epsilon \frac{\partial f}{\partial w}(w)\cdot\rho_{w}(\eta)
\right|
\leq
\frac{L}{2}\epsilon^2\|\eta\|_{w}^{2}.
\end{displaymath}
\end{asu}
 Note that when the family of norms is simply the Euclidean norm, $\|\cdot\|_w = \|\cdot\|_2$ for all $w$, then this condition, often called \textit{smoothness}, is satisfied by functions that have a Lipschitz continuous gradient \citep[Lemma~1.2.3]{nesterov2013introductory}.
While superficially appearing similar to the standard Lipschitz gradient assumption, by allowing for a family of norms in the definition (and also by involving the duality structure), a larger set of functions can be seen to satisfy the criteria, and hence this is assumption is strictly more general than the standard smoothness assumption.
Most importantly for us, in Proposition \ref{lip-prop-nn} below, we show that Assumption \ref{f-asu} is satisfied for the empirical loss function of a deep neural network,  by appropriate choice of the family of norms.  This setting of course includes the simple network with two hidden nodes presented in Section \ref{funsect}, and we refer the reader to Proposition \ref{simple-ex-lip} in the appendix which explicitly establishes Assumption \ref{f-asu} for the example of Section \ref{funsect}.
Another function class that does not satisfy Euclidean smoothness but does satisfy Assumption \ref{f-asu} is given below in Example \ref{another-example}.
Roughly speaking, if the family of norms is defined using a scalar function like $\|\cdot\|_w = g(x)\|\cdot\|$ for some underlying norm $\|\cdot\|$ and a real-valued function $g$, then the condition \eqref{f-asu} can be interpreted as requiring that the 2nd derivative of $f$ grows slower than $g(x)^2$.
Finally, this assumption can also be compared with condition A3 in \citep{boumal2016globalrates}, except it concerns the simple search space of $\reals^n$, and it is adapted to  use a a family of norms on Euclidean space instead of a Riemannian structure. 

\SetAlgoNoLine

\begin{algorithm}[t]
\caption{Duality structure gradient descent (\textproc{dsgd})\label{algo:finsgd}}
\DontPrintSemicolon
\textbf{input:} Initial point $w(1) \in \reals^n$ and step-size sequence $\epsilon(t)$. \\ \vspace{0.5em}

\For{$t=1,2,\hdots$}{
\vspace{0.5em}
$\blacktriangleright$ Obtain derivative estimate $g(t)$ \label{stoch-g-def} \\
\vspace{0.5em}
$\blacktriangleright$ Compute the search direction $\Delta(t) = \rho_{w(t)}\left(g(t)\right)$ \label{stoch-delta-def} \\
\vspace{0.5em}
$\blacktriangleright$  Update the parameter $w(t+1) = w(t) - \epsilon(t) \Delta(t)$  \label{stoch-w-def}
\vspace{0.5em}
}
\end{algorithm}

To begin the analysis of \textproc{dsgd}, we  define the filtration
$\{\mc{F}(t)\}_{t=0,1,\hdots}$, where $\mc{F}(0) = \sigma( w(1))$ and for $t\geq 1$,
$\mc{F}(t) = \sigma(w(1),g(1),g(2),\hdots,g(t))$.
We assume that the derivative estimates $g(t)$ are unbiased, and have
bounded variance relative to the family of norms: 
\begin{asu}\label{bias-asu}
  For $t=1,2,\hdots,$ define 
$\delta(t) = g(t) - \frac{\partial f}{\partial w}(w(t))$.
The $\delta(t)$ must satisfy
\begin{subequations}
\begin{align}
\mathbb{E}\left[
\delta(t)
\mid \mc{F}(t-1)
\right]
&=
0 \label{bias-eq},
\\
 \mathbb{E}\left[\|\delta(t)\|^{2}_{w(t)} \mid \mc{F}(t-1)
   \right] &\leq \sigma^2 < \infty. \label{var-ineq}
\end{align}
\end{subequations}
\end{asu}
When  $\|\cdot\|_w = \|\cdot\|_2$ for all $w$, Equation \eqref{var-ineq} states that, conditioned on past observations, the gradient estimates have uniformly bounded variance in the usual sense. In ML applications, the stochastic gradient is typically computed  with respect to a random minibatch of examples. 
However, in the context of using \textproc{sgd} to train deep neural nets, it may be problematic to require  the right hand side of \eqref{var-ineq} be bounded uniformly over all $t$ and all $w(1)$, since, as we demonstrate in Proposition \ref{nn-unbounded-var} below, it is possible to construct training data sets for the neural network model of Section \ref{funsect} such that the variance of the minibatch gradient estimator with unit batch-size is unbounded as a function of $\|w\|$.
By allowing for a family of norms, Assumption \ref{bias-asu} avoids this problem; as shown below in
Lemma \ref{nn-grad-var},
the variance of mini-batch gradient estimator for  \textproc{dsgd}  is bounded relative to the appropriate family of norms. 

Compared to the analysis of \textproc{sgd} in the Euclidean case (e.g, that of \citep{ghadimi-lan}) the analysis of \textproc{dsgd} is more involved because of the duality map, which may be a nonlinear function.
This means that even if $g(t)$ is unbiased in the sense of
Assumption \ref{bias-asu}, it may be the case that
$\mathbb{E}[\rho_{w(t)}(g(t))] \neq  \rho_{w(t)}(\tfrac{\partial f}{\partial w}(w(t)))$.
To address this, we quantify the bias in the update directions in terms of a convexity parameter of the family of norms, defined as  follows: 
\begin{definition}\cite[Section~4.1]{chidume}
        A family of norms $\|\cdot\|_w$ is  \textit{$2$-uniformly convex with parameter $c$} if there is a constant $c\geq 0$ satisfying, for all $w, x,y \in \reals^n$,
        \begin{eq}\label{uni-cvx-eqn}
         \l\|\frac{x + y}{2}\r\|_w^2
         \leq \frac{1}{2}\l\|x\r\|_w^2 + \frac{1}{2}\|y\|_w^2 - \frac{c}{4}\|x-y\|_w^2.
        \end{eq}
\end{definition}
Equivalently, this definition states that the function $x\mapsto \|x\|^2_w$ is strongly convex with parameter $2c$, uniformly over $w$. For example, if the family of norms is such that each $\|\cdot\|_w$ is an inner product norm, then $c=1$, while if $\|\cdot\|_w=\|\cdot\|_{p}$ for some $1<p<2$  we can take $c= p-1$ \citep[Proposition~3]{Lieb1994}. Note that  Equation \eqref{uni-cvx-eqn} always holds with $c=0$, although positive values of $c$ lead to better convergence rates, as we show below. 

The following Lemma states bounds we shall use to relate bias and convexity:
\begin{lem}\label{the-k-lemma}
Let
$\|\cdot\|$
be a norm on
$\reals^n$ that is $2$-uniformly convex with parameter $c$,
and let
$\rho$
be a corresponding duality map. 
If
$\delta$
is a 
$\reals^n$-valued
random variable such that
$\mathbb{E}[\delta] = 0$ then
\begin{subequations}
\begin{align}
\mathbb{E}\left[\ell \cdot \rho\left(\ell + \delta\right)\right]
&\geq
\left(\frac{1+c}{2}\right)\|\ell\|^{2} -
\left(\frac{1-c}{2}\right)\mathbb{E}\left[\|\delta\|^{2}\right],
\label{bias-expect}
\\
\mathbb{E}\left[\|\ell +\delta\|^2\right]
&\leq (2-c)\|\ell\|^{2} + \left(2-c^2\right)\mathbb{E}\l[\|\delta\|^2\r]. 
\label{bias-norm}
\end{align}
\end{subequations}

\end{lem}
The proof of this lemma is in the appendix. Note that when $\|\cdot\|$ is an inner product norm, the convexity coefficient is $c=1$ and both relations in the lemma are qualities. 

We can now proceed to our analysis of \textproc{dsgd}.
This theorem gives some conditions on $\epsilon$ and  $\gamma$ that guarantee finiteness of the expected amount of time to reach a $\gamma$-approximate stationary point.  
\begin{thm}\label{prop:sgd-prop}
Let
Assumptions \ref{f-asu} and \ref{bias-asu} hold, and suppose the family of norms has convexity parameter $c\geq 0$.
Consider running
Algorithm \ref{algo:finsgd} using constant step-sizes $\epsilon(t) := \epsilon > 0$.
Suppose that $\gamma$ and $\epsilon$ satisfy
\begin{equation}\label{step-k-cond}\begin{split}
\gamma &> \left(\frac{1-c}{1+c}\right)\sigma^2, \\
\epsilon &<
\frac{1}{L}\times
  \frac{(1+c)\gamma - (1-c)\sigma^2}{ (2-c)\gamma + (2-c^2)\sigma^2}
\end{split}
\end{equation}
\noindent Define  $G = f(w(1)) - f^{*}$.  If $\tau$ is defined as in
Equation \eqref{taudef}, then,
\begin{equation}\label{stoch-lemma}
\mathbb{E}[\tau]
\leq
\frac{
2G + \left(1+c - L\epsilon(2-c)\right)
\gamma
}{
 \epsilon
\left(1+c - L\epsilon(2-c)\right)
\gamma - \epsilon
\left(L\epsilon(2-c^2) + 1-c\right)\sigma^2
}.
\end{equation}
\end{thm}
 The proof of Theorem \ref{prop:sgd-prop} is deferred to appendix. Let us consider some special cases of this result. 
\begin{cor}\label{non-asympt-sgd}
Under Assumptions \ref{f-asu} and \ref{bias-asu}, the following special cases of Theorem \ref{prop:sgd-prop} hold:
\begin{enumerate}
\item (Standard SGD) Suppose
$\|\cdot\|_w = \|\cdot\|_2$ for all $w$, and let
$\rho_w$ be the identity mapping
$\rho_w(\ell) = \ell$. Then for any
$\gamma > 0$, setting
$\epsilon = \frac{1}{L}\left(\frac{\gamma}{\gamma + \sigma^2}\right)$ leads to 
$$\mathbb{E}[\tau] \leq
  \frac{2GL\sigma^2}{\gamma^2} + \frac{4L(G+\sigma^2)}{\gamma} +
  8L.
$$
\item
More generally, suppose $\|\cdot\|_w$ has parameter of convexity $c\geq 0$.
Then for any $\gamma > \left(\frac{1-c}{1+c}\right)\sigma^2$, setting
$\epsilon = \frac{1}{2L}\left( \frac{(1+c)\gamma - (1-c)\sigma^2}{(2-c)\gamma + (2-c^2)\sigma^2}\right)$ leads to 
$$
  \mathbb{E}[\tau]
  \leq
    \frac{8LG(2-c^2)\sigma^2}
    {( (1+c)\gamma + (1-c)\sigma^2)^2}
    +
    \frac{8L(2-c^2)(G+\sigma^2)}{
       (1+c)\gamma + (1-c)\sigma^2
    }
    +
    8L(2-c).
    $$
    
\end{enumerate}                   
\end{cor}
\begin{pf}
The claimed inequalities follow by plugging the given values of $\epsilon$ into Equation \ref{stoch-lemma}. The details are given in the appendix. 
\end{pf}
Note that the batch case, represented by $\sigma^2=0$, admits a simpler proof that avoids dependence on the convexity parameter, leading to the following: 
\begin{cor}\label{non-asympt-gd}
Let
Assumption \ref{f-asu} hold, and assume
$g(t) = \tfrac{\partial f}{\partial w}(w(t))$.
Then for any family of norms $\|\cdot\|_w$, if $\gamma = \frac{1}{L}$ then
$$
\min_{1\leq t \leq T}
\left\|\frac{\partial f}{\partial w}(w(t))\right\|_{w(t)}^{2} 
\leq 
\frac{2GL}{T}.
$$
Expressed using the stopping variable $\tau$,  this says
$\tau \leq \left\lceil 2GL/\gamma\right\rceil$.
Furthermore, any accumulation point
 $w^{*}$
 of the algorithm is a stationary point of $f$, meaning
 $\frac{\partial f}{\partial w}(w^{*}) = 0$.
\end{cor}
\begin{proof}
The proof closely follows that of Theorem \ref{prop:sgd-prop}, and the details are deferred to an appendix. 
\end{proof}

For the trivial family of norms that simply assigns the Euclidean norm to each point in the space, Algorithm \ref{algo:finsgd} reduces to standard gradient descent and we recover the known $1/T$ convergence rate for GD \citep{nesterov2013introductory}.
In the general case, the non-asymptotic performance guarantee concerns the quantities
$\|\tfrac{\partial f}{\partial w}(w(t))\|_{w(t)}$,
where the gradient magnitude is measured relative to the local norms $w(t)$.
We leave to future work the interesting question of under what conditions a relation can be established between the convergence of
$\|\tfrac{\partial f}{\partial w}(w(t))\|_{w(t)}$
and the convergence of
$\|\tfrac{\partial f}{\partial w}(w(t))\|$,
where the norm is fixed. 

The main application of Theorem \ref{prop:sgd-prop} will come in the following section, where it is used to prove the convergence of a layer-wise training algorithm for neural networks
(Theorems \ref{main-thm-app} and \ref{main-thm-app-sgd}).
For another example of an optimization problem where this theory applies, consider the following. 
\begin{example}\label{another-example}
Consider applying \textproc{dsgd} to the function
$f:\reals^2 \to \reals$ defined by
$f(x,y) = g(x)h(y)$, where
$g :\reals\to\reals$ and
$h : \reals\to\reals$ are functions that have bounded second derivatives.
For simplicity, assume that
$\|g''\|_{\infty} \leq 1$ and
$\|h''\|_{\infty} \leq 1$. Furthermore, assume that
$\sup_{(x,y)\in\reals^{2}}g(x)h(y) \geq f^*$ for some $f^* \in \reals$  (for instance, this occurs if $g$ and $h$ are non-negative).
The function $f$ need not have a Lipschitz continuous gradient, as the example of $g(x) = x^2$ and $h(y)= y^2$ demonstrates.

Let us denote pairs in $\reals^2$ by $w=(x,y)$.
Define the family of norms
$$\|(\delta x,\delta y)\|_{w} = \sqrt{1+|h(y)|}|\delta x| + \sqrt{1+|g(x)|}|\delta y|.$$
The dual norm and duality map are as previously defined in
Example \ref{less-trivial}.

Let us show that the conditions of
Assumption \ref{f-asu} are satisfied  with $L=1$.
Let
$\eta = (\eta_1,\eta_2)$
be any vector. 
If
$\frac{\eta_1}{\sqrt{1+|h(y)|}} \geq \frac{\eta_2}{\sqrt{1+|g(x)|}}$,
then
$\|\eta\|_{w} = \frac{|\eta_1|}{\sqrt{1+|h(y)|}}$ and 
$\rho(\eta) = \left( \frac{\eta_1}{1+|h(y)|}, 0\right)$.
Then, since the function $x\mapsto f(x,y)$ has a  second derivative that is bounded by $|h(y)|$, we
can apply a standard quadratic bound (Proposition \ref{simple-gd-lem}) to conclude that
\begin{align*}
\left|
f(w + \epsilon \rho_{w}(\eta)) - f(w) - \epsilon \frac{\partial f}{\partial w}(w)\cdot\rho_w(\eta)
\right|
&\leq \epsilon^2\frac{1}{(1+|h(y)|)^{2}}\frac{1}{2}|h(y)|\|\eta_1\|^{2} \\
&= \epsilon^2\frac{|h(y)|}{1+|h(y)|}\frac{1}{2}\|\eta\|_{w}^{2} \\
&\leq
\frac{\epsilon^2}{2}\|\eta\|_{w}^{2}. 
\end{align*}
The case  $\frac{\eta_1}{\sqrt{1+|h(y)|}} < \frac{\eta_2}{\sqrt{1+|g(x)|}}$ is similar.
Hence
Assumption \ref{f-asu} is satisfied.

According to Corollary \ref{non-asympt-gd}, convergence will be guaranteed in batch \textproc{dsgd} with $\epsilon=1$. In more details, in the first step
(Line \ref{stoch-delta-def}) the algorithm computes the duality map on the  derivative $g(t) = \tfrac{\partial f}{\partial w}(w(t))$. If
\begin{equation}\label{the-case}
\left|\frac{\partial f}{\partial x}(x(t),y(t))\right|\frac{1}{\sqrt{1+|h(y(t))|}} \geq 
\left|\frac{\partial f}{\partial y}(x(t),y(t))\right|\frac{1}{\sqrt{1+|g(x(t))|}}
\end{equation}
then the update direction is
$\Delta(t) = \left( \tfrac{\partial f}{\partial x}(x(t),y(t)) \frac{1}{1+|h(y(t))|}, 0\right)$.
In this case, at the next step
(Line \ref{stoch-w-def}) the next point is computed by keeping $y$ the same
($y(t+1) = y(t)$)
and updating $x$ as
$x(t+1) = x(t) - \epsilon\tfrac{\partial f}{\partial x}(x(t),y(t)) \frac{1}{1+|h(y(t))|}$.
If \eqref{the-case} does not hold, then $y$ is updated instead: $x(t+1) = x(t)$ and
$y(t+1) = y(t) - \epsilon\tfrac{\partial f}{\partial y}(x(t),y(t)) \frac{1}{1+|g(x(t))|}$.
The resulting convergence guarantee associated with the algorithm is that
$$
\max\left\{
\frac{|\frac{\partial f}{\partial x}(x(t),y(t))|}{\sqrt{1+|h(y(t))|}} ,
\frac{|\frac{\partial f}{\partial y}(x(t),y(t))|}{\sqrt{1+|g(x(t))|}}
\right\} \rightarrow 0 $$
as $t\to \infty$.

\end{example}
This example could be extended easily to the case where $f$ is defined as the product of arbitrarily many functions that have  bounded second derivatives.

\section{Application to Neural Networks with Multiple Layers}\label{sect:nnapp}
In order to implement and analyze the \textproc{dsgd} algorithm for minimizing a particular objective function, there are three tasks: 1) Define the family of norms for the space, 2) Identify a duality structure to use, and 3) Verify the generalized gradient smoothness condition of
  Assumption \ref{f-asu}.
In this section we carry out these steps in the context of a neural network with multiple layers. 

We first define the parameter space and the objective function. The network consists of an input layer and $K$ non-input layers. We are going to consider the case that each layer is fully connected to the previous one and uses the same activation function. 
Networks with heterogeneous layer types (consisting for instance of convolutional layers, softmax layers, etc.) and networks with biases at each layer can also be accommodated in our theory. 

Let the input to the network be of dimensionality
$n_{0}$, 
and let
$n_{1},\hdots, n_{K}$ specify the number of nodes in each of
$K$ non-input layers. 
For $k=1,\hdots,K$ define 
$W_{k} = \mathbb{R}^{n_{k}\times n_{k-1}}$
to be the space of
$n_{k}\times n_{k-1}$ matrices; 
a matrix in
$w_k \in W_{k}$ specifies weights from nodes in layer
$k-1$
to nodes in layer $k$. The overall parameter space is then 
$W = W_1 \times \hdots \times W_{K}$.
We define the output of the network as follows.
For an input
$x \in \mathbb{R}^{n_0}$, and weights $w = (w_1,\hdots, w_K) \in W$, the output is
$y^{K}(w;x) \in \mathbb{R}^{n_{K}}$ where 
$y^0(w;x) = y$ and for $1 \leq k \leq K$,
\begin{equation*}
y^{k}_{i}(w;x) 
= 
\sigma
\bigg(
  \textstyle
  \sum\limits_{j=1}^{n_{k-1}}
    w_{k,i,j}y^{k-1}_{j}(w;x)
\bigg), 
\quad 
i = 1,2,\hdots,n_{k}.
\end{equation*}
Given
$m$ input/output pairs
$(x_1,z_1),(x_2,z_2),\hdots,(x_m,z_m)$, where
$(x_n,z_n) \in \mathbb{R}^{n_0}\times\mathbb{R}^{n_K}$,
 we seek to minimize
 the \textit{empirical error}
\begin{equation}\label{emp-err}
\begin{split}
f(w) = \frac{1}{m}\textstyle\sum\limits_{i=1}^{m}f_{i}(w), 
\end{split}
\end{equation}
where the $f_i$  are
\begin{equation}\label{the-fi}
  f_{i}(w) = \|y^{K}(w;x_i) - z_i\|_2^{2}, \quad i=1,2,\hdots,m.
  \end{equation}
and $\|\cdot\|_{2}$ is the Euclidean norm.

\subsection{Layer-wise gradient smoothness for Neural Networks}
Our assumptions on the
activation $\sigma$,
the inputs $x_i$, and
targets $z_i$, are as follows: 
\begin{asu}\label{nonlin-asu}\quad 
\begin{enumerate}
   \renewcommand{\theenumi}{\roman{enumi}}
\item \label{nlbd}(Activation bounds)  The activation function $\sigma$ and its first two derivatives are bounded. Formally, 
$
\|\sigma\|_{\infty} \leq 1
$, 
$
\|\sigma'\|_{\infty} < \infty$, and $\|\sigma''\|_{\infty} < \infty$.
\item \label{iobd}(Input/Target bounds) 
 $\|x_i\|_{\infty} \leq 1$ and $\|z_i\|_{\infty} \leq 1$ for $i=1,2,\hdots,m$.
\end{enumerate}
\end{asu}
Note that the first part of the assumption is satisfied by the sigmoid function
$\sigma(u) = 1/(1+e^{-u})$
and also the hyperbolic tangent function
$\sigma(u) =-1 + 2/(1+e^{-2u})$. Note that this assumption is not satisfied by the ReLU function $\sigma(u)= \max\{0,x\}$.
The second part of Assumption \ref{nonlin-asu} states that the components of the inputs and targets are between $-1$ and $1$. 

In
Proposition \ref{hess-bd-prop} we establish that, under  Assumption \ref{nonlin-asu},
  the restriction of the objective function to the weights in any particular layer is a function with a bounded second derivative. 
  To confirm the boundedness of the second derivative, any norm on the weight matrices can be used, because on finite dimensional spaces all norms are strongly equivalent. However, different norms will lead to different specific bounds. For the purposes of gradient descent, each norm and Lipschitz bound implies a different quadratic upper bound on the objective, which in general may lead to a variety of update steps, hence defining different algorithms.
   Our construction considers the induced matrix norms corresponding to the vector norms
$\|\cdot\|_{q}$ for
$1 \leq q \leq \infty$. 

\begin{asu}\label{q-asu}
  For some 
  $q \in [1,\infty]$ and all   $1\leq i \leq  K$, 
  each space $\reals^{n_{i}}$
  has the norm $\|\cdot\|_{q}$.
\end{asu}
We are going to be working with the matrix norm induced by the given choice of $q$. Recall that, for an $r \times c$ matrix $A$, 
if $q=1$, then,
$$\|A\|_1 = \max\limits_{1\leq j\leq c}\,\sum\limits_{i=1}^{r}|A_{i,j}|,$$ 
while if $q=2$,
  $$\|A\|_{2} =
  \max\limits_{1\leq i\leq \min\{r,c\}}\,\sigma_{i}(A), 
  $$
  where for a matrix $A$, we define
  $\sigma(A) = \left(\sigma_1(A),\hdots,\sigma_{\min\{r,c\}}(A)\right)$
  to be the vector  of singular values of $A$. Lastly,
  if
$q=\infty$,
  $$\|A\|_{\infty} =
    \max\limits_{1\leq i\leq r}\,\sum\limits_{j=1}^{c}|A_{i,j}|. 
  $$
  That is, when
  $q = 1$ the matrix norm is the maximum absolute column sum, when
  $q = 2$ the norm is the largest singular value, also known as the spectral norm, and when
  $q = \infty$ the norm is the largest absolute row sum \citep{matanal}. 

  For ease of notation,  throughout this section, we will assume that all the layers have the same number of nodes. Formally, this means
  $n_{i} = n_{K}$  for $i=0,\hdots,K$. Also, at times we use the subvector notation $z_{1:i}$ to denote the first $i$ components of a vector $z$.
  Next, let us introduce a set of functions that will be used to express our bounds on the second derivatives.
    Let $r_0=1$, and for $1 \leq n \leq K-1$ the function $r_n$ is
  \begin{equation*}
    r_{n}(z_1,\hdots,z_n) 
    = 
    \|\sigma'\|_{\infty}^{n}
    \textstyle\prod_{i=1}^{n}z_i.
  \end{equation*}
  Then define $v_n$ recursively, with
  $v_0 = 0$,
  and for
  $1 \leq n \leq K-1$, the function $v_n$ is
  \begin{equation*}
    v_n(z_1,\hdots,z_n) 
    = 
    \|\sigma''\|_{\infty}\|\sigma'\|_{\infty}^{2(n-1)}
    \textstyle\prod_{i=1}^{n}z_i^{2} 
    + 
    \|\sigma'\|_{\infty} z_n v_{n-1}(z_1,\hdots,z_{n-1}).
  \end{equation*}
  Define constants
  $d_{q,1}$,
  $d_{q,2}$ and
  $c_{q}$
  as in
  Table \ref{table:defs}. 
   Then for $ 0 \leq n \leq K-1$ the function $s_{n}$ is
  \begin{equation*}
    s_{i}(z_1,\hdots,z_i) = 
    d_{q,2}c_{q}^{2}\|\sigma'\|_{\infty}^{2}
    r_{i}^{2}(z_{1:i})  
    + 
    d_{q,1}c_{q}^{2}\|\sigma'\|_{\infty}^{2}v_{i}(z_{1:i}) 
    + 
    d_{q,1}c_{q}^{2}\|\sigma''\|_{\infty}r_{i}(z_{1:i}).
  \end{equation*}
\begin{prop}\label{hess-bd-prop}
  Let
  Assumptions \ref{nonlin-asu} and \ref{q-asu} hold, and let $q$ be the constant from 
  Assumption \ref{q-asu}.
  Let the spaces 
  $W_1,\hdots,W_{K}$ 
  have the norm induced by 
  $\|\cdot\|_{q}$
  and define functions
   $p_1,\hdots,p_K$ as
  \begin{align}\label{p-definition}
    p_{i}(w) &= \sqrt{ s_{K-i}\left(\|w_{i+1}\|_{q}, \hdots, \|w_{K}\|_{q}\right) + 1}.
\end{align}
Let $f$ be as in \eqref{emp-err}.
Then for all $w \in W$ and $1 \leq i \leq K$, the bound 
$
\|
 \tfrac{\partial^{2} f}
      {\partial w_{i}^{2}}(w)\|_{q} \leq p_{i}(w)^{2}
$ holds.
\end{prop}
  \begin{table}[]
    \centering
    \begin{tabular}[t]{lcccc}
      \toprule
               & $q=1$ & $1<q<2$      & $2 \leq q < \infty$ & $q=\infty$ \\
      \midrule
      $c_q$    & $n$     & $n^{1/q}$     & $n^{1/q}$            &  $1$   \\
      $d_{q,1}$ &  $4$  & $4n^{(q-1)/q}$ & $4n^{(q-1)/q}$        & $4n$   \\
      $d_{q,2}$ &  $2$  & $2$          & $2n^{(q-2)/q}$       & $2n$ \\
      \bottomrule
    \end{tabular}
    \caption{The definitions of constants used in Proposition \ref{hess-bd-prop}. $c_q$ represents the magnitude of the vector $(1,1,\hdots,1)$ in the norm $\|\cdot\|_q$, and $d_{q,1}, d_{q,2}$ are bounds on the first and second derivatives, respectively, of the function $J(x) = \|x-z\|_2^2$,  measured with the norm $\|\cdot\|_q$. \label{table:defs}}
  \end{table}%

\noindent For example,  a network with one hidden layer
yields polynomials
$p_1,p_2$ where
\begin{subequations}
  \begin{equation}\label{p1-one-hidden}
    \begin{split}
  p_{1}(w)
  &= \sqrt{s_{1}(\|w_{2}\|_{q}) + 1} \\
  &= \sqrt{d_{q,2}c_{q}^{2}\|\sigma'\|^2_{\infty}r_{1}(\|w_2\|_{q})^{2}
    +
    d_{2,1}c_{q}^{2}\|\sigma'\|_{\infty}^{2}v_{1}(\|w_{2}\|_{q})
    +
    d_{q,1}c_{q}^{2}r_{1}(\|w_2\|_{q})
    +
    1
    } \\
  &=\sqrt{
    \left( d_{q,2}c_{q}^{2}\|\sigma'\|_{\infty}^{4}
    +
    d_{q,1}c_{q}^{2}\|\sigma'\|_{\infty}^{2}\|\sigma''\|_{\infty}\right)\|w_{2}\|_q^{2}
    +
    d_{q,1}c_{q}^{2}\|\sigma''\|_{\infty}\|\sigma'\|_{\infty} \|w_2\|_q
    +
    1
    },
    \end{split}
  \end{equation}
  \begin{equation}\label{p2-one-hidden}
    \begin{split}
  p_{2}(w)
  &=
  \sqrt{s_0 +  1} \\
  &=
  \sqrt{
    d_{q,2}c_{q}^{2}\|\sigma'\|_{\infty}^{2}r_0^2
    +
    d_{q,1}c_{q}^{2}\|\sigma'\|_{\infty}^{2}v_0
    +
    d_{q,1}c_{q}^{2}\|\sigma''\|_{\infty}r_0
    +
    1
  } \quad\quad\quad\quad\quad\quad\quad\quad\quad\quad\\
  &=
  \sqrt{
    c_{q}^{2}( d_{q,2}\|\sigma'\|_{\infty}^{2} + d_{q,1}\|\sigma''\|_{\infty})
    +1
  }.
    \end{split}
  \end{equation}
\end{subequations} 

 Note that in Proposition \ref{hess-bd-prop}, the norm of the Hessian matrix is bounded by a polynomial in the norms of the weights, and terms of a similar form appear in norm-based complexity measures for deep networks \citep{liang2019fisher,Neyshabur15,bartlettspectrally}.  

Proposition  \ref{hess-bd-prop} enables us to analyze algorithms that update only one layer at a time. Specifically, if we update a layer in the direction of the image of the gradient under the duality map, then an appropriate step-size guarantees improvement of the objective. This is a consequence of the following Lemma: 
\begin{prop}\label{simple-gd-lem}
  Let
  $f:\reals^n\to\reals$
  be a function with continuous derivatives up to 2nd order.
  Let
  $\|\cdot\|$
  be an arbitrary norm on $\reals^n$ and let $\rho$ be a corresponding duality map. 
  Suppose that 
  $
  \sup_{w}\sup_{\|u_1\|=\|u_2\|=1}
  \left\|\tfrac{\partial^2 f}{\partial w^{2}}(w)\cdot(u_1,u_2)\right\|
  \leq L
  $.
  Then for any $\epsilon>0$ and any $\Delta \in \reals^n$,
  $f(w-\epsilon\Delta) \leq  f(w) -
  \epsilon \frac{\partial f}{\partial w}(w)\cdot\Delta +
  \epsilon^2\frac{L}{2}\|\Delta\|^{2}.$
  In particular,
  $
  f\left(w-\epsilon\rho\left(\frac{\partial f}{\partial w}(w)\right)\right)
  \leq
  f(w)
  -
  \epsilon\left(1- \frac{L}{2}\epsilon\right)
  \|\frac{\partial f}{\partial w}(w)\|^{2}.
  $
\end{prop}
This proposition motivates the following greedy algorithm: Identify a layer $i^{*}$ such that
$
i^{*}
=
\argmax_{1\leq i \leq K}\frac{1}{p_{i}(w)}\|\tfrac{\partial f}{\partial w_i}(w)\|
$
and  update parameter $w_{i^*}$, using a step-size
$\tfrac{1}{p_{i^{*}}(w)^{2}}$
in the direction
$\rho(\tfrac{\partial f}{\partial w}(w))$.
As a consequence of
Proposition \ref{hess-bd-prop} and
Proposition \ref{simple-gd-lem}, this
update will lead to a decrease in the objective of at least
$\tfrac{1}{2 p_{i^{*}}(w)^2}\|\tfrac{\partial f}{\partial w_{i^{*}}}(w)\|^2$.
This greedy algorithm is depicted (in a slightly generalized form) in
Algorithm \ref{algo:finslernn}.
In the remainder of this section, we will show how this sequence of operations can be explained with a particular duality structure on $\reals^n$, in order to apply the convergence theorems of
Section \ref{sect:gd}. 
\subsection{Family of norms and duality structure}\label{subsect-fins}
In this section we define a family of norms and  associated duality structure that encodes the layer-wise update criteria.
The family of norms is constructed using the functions $p_i$ from \eqref{p-definition} as follows.
For any $w = (w_1,\hdots,w_K) \in W$ and any $(u_1, \hdots, u_{K}) \in W$,
define  $\|(u_1,\hdots,u_K)\|_{w}$ as 
\begin{equation}\label{the-finsler-structure}
\|(u_1,\hdots,u_K)\|_{w} 
= 
 p_{1}(w)\|u_1\|_{q} + \hdots + p_{K}(w)\|u_{K}\|_{q}.
\end{equation}
Note that the family of norms and the polynomials $p$ also depend on the user-supplied parameter $q$ from
Assumption \ref{q-asu}, although we omit this from the notation for clarity. 

To obtain the duality structure, we derive duality maps for matrices with the norm
$\|\cdot\|_{q}$, and then use a general construction for product spaces. 
The first part
is summarized in the following Proposition. Note that when we use the $\argmax$ to find the index of the largest entry of a vector, any tie-breaking rule can be used in case there are multiple maxima. For instance, the $\argmax$ may be defined to return the smallest such index.

\begin{prop}\label{one-space-dual}
Let $\ell \in \mathcal{L}(\mathbb{R}^{r\times c},\reals)$ be a linear functional defined on
a space of matrices with the norm $\|\cdot\|_{q}$ for
$
q
\in
\{
1,
2,
\infty\}
$. Then the dual norm is 
\begin{equation}\label{mat-dual-norm}
  \|\ell\|_{q} =
  \begin{cases}
    \sum\limits_{j=1}^{c}\max\limits_{1\leq i\leq r}|\ell_{i,j}| &\text{ if } q = 1, \vspace{0.3em}\\ 
     \sum\limits_{i=1}^{\min\{r,c\}}\sigma_i(\ell) &\text{ if } q = 2, \vspace{0.3em} \\
     \sum\limits_{i=1}^{r}\max\limits_{1\leq j\leq c}|\ell_{i,j}| &\text{ if } q = \infty \vspace{0.3em}.
  \end{cases}
\end{equation}
Possible choices for duality maps are as follows:\\
For $q=1$, the duality map $\rho_{1}$ sends $\ell$ to a matrix that picks out a maximum in each column: $ \rho_{1}(\ell)  = \|\ell\|_{1}m$  where $m$ is the $r\times c $ matrix
\begin{equation}\label{mat-duality-1}
 m_{i,j} = \begin{cases}
  \sgn(\ell_{i,j}) &\text{ if } i = \argmax_{1\leq k \leq r} |\ell_{k,j}|, \\
  0 &\text{ otherwise. }
  \end{cases}
\end{equation}
For $q=2$ the duality map  $\rho_{2}$ normalizes the singular values of  $\ell$:
If $\ell = U\Sigma V^{T}$ is the singular value decomposition of $\ell$, written in terms of column vectors as  $ U = [ u_1,\hdots, u_c], V= [ v_1,\hdots, v_c]$, and denoting the rank of the matrix $\ell$ by $\rank \ell$, then
\begin{equation}\label{mat-duality-two}
\rho(\ell)_{2}  = \|\ell\|_{2}\sum\limits_{i=1}^{\rank \ell}u_{i}v_{i}^{T}.
\end{equation}
For $q=\infty$, the duality map $\rho_{\infty}$  sends $\ell$ to a matrix that picks out a maximum in each row: $\rho_{\infty}(\ell)  = \|\ell\|_{\infty}m$
  where $m$ is the $r\times c $ matrix
\begin{equation}\label{mat-duality}
  m_{i,j} = \begin{cases}
  \sgn(\ell_{i,j}) &\text{ if } j = \argmax_{1\leq k \leq c} |\ell_{i,k}|, \\
  0 &\text{ otherwise. }
  \end{cases}
\end{equation}

\end{prop}
The proof of this proposition is in the appendix. 

Next, we construct a duality map for a product space from duality maps on the components.
Recall that in a product vector space 
$Z=X_1\times \hdots \times X_K$,
each linear functional 
$\ell \in \mathcal{L}(Z,\reals)$ uniquely decomposes as 
$
\ell 
= 
(\ell_1,\hdots,\ell_K) 
\in 
\mathcal{L}(X_1,\reals) \times \hdots \times \mathcal{L}(X_K,\reals)$. 
\begin{prop}\label{duality-product}
If $X_1,\hdots,X_K$ are normed spaces, 
carrying duality maps 
$\rho_{X_1},\hdots,\rho_{X_K}$
respectively, and the product $Z=X_1\times \hdots \times X_K$ has norm 
$\|(x_1,\hdots,x_K)\|_{Z} = p_1\|x_1\|_{X_1} + \hdots + p_K\|x_K\|_{X_K}$,
for some positive coefficients $p_1,\hdots,p_K$, then the dual norm for $Z$ is 
\begin{equation}\label{the-dual-norm-gen}
\|(\ell_1,\hdots,\ell_K)\|_{Z}
\,=
\max
\left\{\frac{1}{p_1}\|\ell_1\|_{X_1},\hdots,\frac{1}{p_K}\|\ell_K\|_{X_K}\right\}
\end{equation}
and a duality map is given by
$$\rho_{Z}(\ell_1,\hdots,\ell_K)
=
\left(
0,
\hdots,
\frac{1}{\left(p_{i^{*}}\right)^{2}}\rho_{X_{i^{*}}}(\ell_{i^{*}}),
\hdots,
0
    \right)$$  where  $i^{*} = \argmax_{1\leq i\leq K} \left\{\frac{1}{p_i}\|\ell_i\|_{X_i}\right\}.$
\end{prop}
See the appendix for a proof of
Proposition \ref{duality-product}.
Based on
Proposition \ref{duality-product}, and the definition of the family of norms from \eqref{the-finsler-structure}, the dual norm at a point
$w \in W = W_1 \times \hdots \times W_K$ is
\begin{equation}\label{the-dual-norm}
\|(\ell_1,\hdots, \ell_K)\|_{w} 
= 
\max\limits_{1\leq i\leq K} \frac{1}{p_{i}(w)}\|\ell_i\|_{q}.
\end{equation}
We define the  duality structure on the neural net parameter space as follows:
\begin{enumerate}
\item Each space $W_1,\hdots,W_K$ has duality map $\rho_{q}(\cdot)$, defined by Proposition
  \ref{one-space-dual}.
\item The duality map at each point $w$ is defined according to
  Proposition \ref{duality-product}:
    \begin{equation}\label{the-duality-structure}
      \rho_{w} (\ell_1,\hdots,\ell_K)
      = 
      \left(0,\hdots, 
        \frac{1}{\left(p_{i^{*}}(w)\right)^{2}}\rho_{q}(\ell_{i^{*}}),
        \hdots,0\right) 
    \end{equation}
where $i^{*} = \argmax_{1\leq i\leq K} \left\{\frac{1}{p_i(w)}\|\ell_i\|_{q}\right\}$.
\end{enumerate}

\subsection{Convergence Analysis}
Throughout this section,  we associate with $W$ the family of norms
 $\|\cdot\|_w$ from  \eqref{the-finsler-structure} and duality structure
 $\rho_w$ from \eqref{the-duality-structure}, and the function $f$ is defined as in \eqref{emp-err}.  The convergence analysis of  Algorithm \ref{algo:finsgd} is based on the  idea  that the update performed in the algorithm is exactly equivalent to taking a step in the direction of the duality map \eqref{the-duality-structure} as applied to the derivative of $f$, so the algorithm is simply a special case of Algorithm \ref{algo:finsgd}. 
Recall that the convergence property of Algorithm \ref{algo:finsgd} depends on verifying the generalized smoothness condition set forth in 
Assumption \ref{f-asu}. This smoothness condition is confirmed in the following proposition. 
\begin{lem}\label{lip-prop-nn}
  Let
  Assumptions \ref{nonlin-asu} and \ref{q-asu} hold, and let $q$ be the constant chosen in
  Assumption \ref{q-asu}. Let $f$  defined as in  \eqref{emp-err}. Then  Assumption \ref{f-asu}  is satisfied with $L=1$.
\end{lem}
Now that
Assumption \ref{f-asu} has been established, we can proceed to the analysis of batch and stochastic \textproc{dsgd}. 

\subsection{Batch analysis}
First we consider analysis of Algorithm \ref{algo:finslernn} running in Batch mode. 
Each iteration starts on
Line \ref{bp} by computing the derivatives of the objective function. This is a standard back-propagation step.  Next, on
Line \ref{dual}, for each layer $i$ the polynomials $p_i$ and the $q$-norms of the derivatives $g_i$ are computed. Note that for any $i < K$, computing $p_i$ will require the matrix norms $\|w_{i+1}\|_q,\hdots,\|w_{K}\|_{q}$.
In
Line \ref{findmax}, we identify which layer $i$ has the largest value of $\|g_i(t)\|_q/p_i(w(t))$.
Note that this is equivalent to maximizing $\|g_i(t)\|_q^{2}/2p_i(w(t))^{2}$, which is exactly the lower bound guaranteed by
Proposition \ref{simple-gd-lem}.
Having chosen the layer, in
Lines \ref{update-step} through \ref{donestep} we perform the update of  layer $i^*$, keeping parameters in other layers fixed. 
\LinesNumbered 
\SetAlgoNoLine

\begin{algorithm}[t]

  \caption{Duality structure gradient descent for a multi-layer neural network\label{algo:finslernn}}
  \textbf{input:} Parameter $q \in \{1, 2,\infty\}$, training data $(y_i,z_i)$ for $1\leq i \leq m$, initial point $w(1) \in W$, step-size $\epsilon$,  selection of mode $\operatorname{Batch}$ or $\operatorname{Stochastic}$, and batch-size $b$ (only required for Stochastic mode.)\\ \vspace{0.5em}
  \For{$t=1,2,\hdots$}{
    \vspace{0.5em}
    \uIf{$\mathrm{Mode} = \mathrm{Batch}$}{    \vspace{0.5em}
      $\blacktriangleright$ Compute full derivative $
           g(t) =
                \tfrac{\partial f}{\partial w}(w(t))
                $. \label{bp}
                    \vspace{0.5em}
    }
    \ElseIf{$\mathrm{Mode} = \mathrm{Stochastic}$}{     \vspace{0.5em}
      $\blacktriangleright$ Compute mini-batch derivative
      $ g(t) = \frac{1}{b}\sum\limits_{j \in B(t)}\frac{\partial f_j}{\partial w}(w(t)).$ \label{sbp}
    }
        \vspace{0.5em}
    $\blacktriangleright$ Compute 
    $
    \frac{1}{p_{1}(w(t))}\|g_1(t)\|_{q},
    \hdots,
    \frac{1}{p_{K}(w(t))}\|g_K(t)\|_{q}.
    $
    \hfill (Using \eqref{p-definition} and Prop. \ref{one-space-dual})\label{dual}
    \\     \vspace{0.5em}
    $\blacktriangleright$ Select layer to update: 
    $
    i^{*}
    =
    \argmax_{1\leq i \leq K}\frac{1}{p_{i}(w(t))}\|g_i(t)\|_{q}
   $.    \label{findmax}   \\ \vspace{0.5em}
 $\blacktriangleright$ Update 
    $
    w(t+1)_{i^{*}}
    =
    w(t)_{i^{*}} - \epsilon\frac{1}{p_{i^{*}}(w(t))^2}\rho_{q}\left(g_{i^*}(t)\right)
    $. \hfill (Using Prop. \ref{one-space-dual}) \label{update-step}
    \\ \vspace{0.5em}
    \For{$i \in \{1,2,\hdots, K\} \setminus \{i^*\}$ }{
           \vspace{0.5em}
             $\blacktriangleright$ Copy previous parameter: $w(t+1)_{i} = w(t)_{i}$.
               \vspace{0.5em}
     }\label{donestep} \vspace{0.5em}
  }
\end{algorithm}

\begin{thm}\label{main-thm-app}
  Let the function $f$ be defined as in
  \eqref{emp-err}, let
  Assumptions \ref{nonlin-asu} and \ref{q-asu} hold, and let $q$ be the constant chosen in
  Assumption
  \ref{q-asu}.
  Associate with $W$ the family of norms
 $\|\cdot\|_w$ from  \eqref{the-finsler-structure} and duality structure
  $\rho_w$ from \eqref{the-duality-structure}.
 Consider running
  Algorithm \ref{algo:finslernn}
  in batch mode, using step-size
  $\epsilon = 1$.
	Then 
  $\min_{1\leq t \leq T}\|\frac{\partial f}{\partial w}(w(t))\|^2_{w(t)} \leq \delta$
  when
  $T \geq 2 f(w(1))/\delta$.
\end{thm}
\begin{pf} It is evident that 
  the update performed in
  Algorithm \ref{algo:finslernn}
  running in batch mode is of the form
  $w(t+1) = w(t) - \epsilon \Delta(t)$, where
  $\epsilon =1$ and
  $\Delta(t) = \rho_{w(t)}(\tfrac{\partial f}{\partial w}(w(t)))$.
  Hence the algorithm is a particular case of
  Algorithm \ref{algo:finsgd}.
  We have established  
  Assumption \ref{f-asu} in Lemma  \ref{lip-prop-nn}, and the result follows by
  Corollary \ref{non-asympt-gd}, using $L=1$ and $f^{*}= 0$. 
\end{pf}
To get some intuition for this convergence bound,
 note that the local derivative norm
 may be lower bounded as
\begin{align*}
  \left\|\frac{\partial f}{\partial w}(w(t))\right\|_{w(t)}
  =
  \max_{1\leq i \leq K}\frac{\left\|\frac{\partial f}{\partial w_i}(w(t))\right\|_{q}}{p_i(w(t))} 
  &\geq
  \frac{  \sum\limits_{i=1}^{K}\left\|\frac{\partial f}{\partial w_i}(w(t))\right\|_{q}}{K \sum\limits_{j=1}^{K}p_j(w(t))}.
\end{align*}
Therefore, using a step-size $\epsilon=1$, a consequence of the convergence bound is 
$$\min_{1\leq t \leq T}\frac{\sum\limits_{i=1}^{K}\left\|\frac{\partial f}{\partial w_i}(w(t))\right\|_{q}}{K \sum\limits_{j=1}^{K}p_j(w(t))}
 \leq \delta
 $$
 when
$ T \geq  2 f(w(1))/\delta$. 
In this inequality, the term on the left-hand side is the magnitude of the gradient relevant to a fixed norm independent of the weights $w$, divided  by a term that is an increasing function of the weight norms $\|w(t)\|$. 

\begin{remark}\label{rem}
  Note that in our analysis of \textproc{dsgd} for neural networks, it is important that the abstract theory is not constrained to update schemes based on inner-product norms.
In our case, the family of norms on the parameter space \eqref{the-finsler-structure} is defined so that the corresponding duality structure \eqref{the-duality-structure} generates updates
that are confined to a single layer. This feature will not be present in the duality map for any inner product norm, since the duality map for an inner product norm is always linear. More explicitly, suppose that
$\ell_1$ and $\ell_2$ are linear functionals and $\rho$ is a duality map for an inner product norm.
If
$\rho(\ell_1)$ has non-zero components in only the first layer, and  $\rho(\ell_2)$ only has non-zero components in the second layer, then, due to linearity, $\rho(\ell_1+\ell_2) = \rho(\ell_1) +\rho(\ell_2)$ has non-zero components in both layers. 
\end{remark}

Next, let us consider the setting of mini-batch duality structure stochastic gradient descent.
This corresponds to executing the steps of
Algorithm \ref{algo:finslernn} in with ``stochastic'' mode, where  instead of computing the full derivative at each iteration, approximate derivatives are calculated by averaging the gradient of our loss function over some number of randomly selected instances in our training set.
Formally, this is expressed in
Line \ref{sbp} of 
Algorithm \ref{algo:finslernn}.
We represent $b$ randomly chosen instances  as a random subset  $B(t) \setin \{1,\hdots,m\}^{b}$ and  the gradient estimate $g(t)$ is
\begin{equation}\label{d-estim}
  g(t) = \frac{1}{b}\sum\limits_{j \in B(t)}\frac{\partial f_j}{\partial w}(w(t)).
\end{equation}
We first show that this gradient estimate has a uniformly bounded variance relative to the family of norms \eqref{the-finsler-structure} 
\begin{lem}\label{nn-grad-var}
  Let
  Assumptions \ref{nonlin-asu} and \ref{q-asu} hold and let $q$ be the constant chosen in
  Assumption \ref{q-asu}.
  Let $g(t)$ be as in \eqref{d-estim}
  and define
  $\delta(t) = \frac{\partial f}{\partial w}(w(t)) - g(t)$.
  Then the variance of $g(t)$ is bounded as 
\begin{equation}\label{the-var-cond}
  \mathbb{E}
  \left[
    \|\delta(t)\|^{2}_{w(t)} \, \middle| \, \mc{F}(t-1)
    \right]
  \leq
  \frac{1}{b}\times
   32 K n^{\max\{1+2/q, 4-4/q\}}.
\end{equation}
\end{lem}
\begin{remark}Note that the right-hand side of \eqref{the-var-cond} is bounded independently of $w(1),\hdots,w(t-1)$. This is notable as such a guarantee cannot be made in standard (Euclidean) \textproc{sgd}, a fact we formally prove in Proposition \ref{nn-unbounded-var}.\end{remark}
\noindent 
Now that we have established a bound on the variance of the gradient estimates $g(t)$, we can move to the performance guarantee for stochastic gradient descent. 
\begin{thm}\label{main-thm-app-sgd}
  Let the function $f$ be defined as in  \eqref{emp-err},  let
  Assumptions \ref{nonlin-asu} and \ref{q-asu} hold let $q$ be the constant chosen in
  Assumption \ref{q-asu}.
  Associate with $W$ the family of norms \eqref{the-finsler-structure} and duality structure \eqref{the-duality-structure}.
  Set
  $
  \sigma^{2}
  =
  \frac{32}{b} K n^{\max\{1+2/q,4-4/q\}}.
  $
  Consider running
  Algorithm \ref{algo:finsgd} in stochastic mode, with a batch size $b$ and  step-size 
  $\epsilon =  \frac{1}{4}\frac{\gamma - \sigma^2}{\gamma + \sigma^2}$.
  Then for any
  $\gamma >\sigma^{2}$ if $\tau$, is the stopping time \eqref{taudef}
  it holds that 
  $$
   \mathbb{E}[\tau]
  \leq
    \frac{16G\sigma^2}
    {( \gamma + \sigma^2)^2}
    +
    \frac{16(G+\sigma^2)}{
       \gamma + \sigma^2
    }
    +
    16.
$$
\end{thm}
\begin{proof}The reasoning follows the proof of  Theorem \ref{main-thm-app}:
  Assumption \ref{f-asu} was established in Lemma \ref{lip-prop-nn},
  and Assumption \ref{bias-asu} follows from Lemma \ref{nn-grad-var},  and hence
  the result follows from Corollary \ref{non-asympt-sgd}, using $L=1$, $c=0$, and $f^{*} = 0$. 
\end{proof}
Note that our result requires that $\gamma$ be at least $\sigma^2$. This is in contrast to the deterministic case (Theorem \ref{main-thm-app}) which does not restrict $\gamma$. An interesting avenue of future work would be to see whether this bound can be improved.


\section{Numerical Experiment}\label{sect:num}

The previous section established  convergence guarantees for \textproc{dsgd}, in  both batch
(Theorem \ref{main-thm-app})  and minibatch
(Theorem \ref{main-thm-app-sgd}) 
settings. In this section we investigate the practical efficiency of \textproc{dsgd} with numerical experiments on several machine learning benchmark problems. 
These benchmarks included the MNIST \citep{lecun1998mnist}, Fashion-MNIST \citep{fashionmnist}, SVHN \citep{netzer2011reading}, and CIFAR-10 \citep{Krizhevsky09} image classification tasks.
In our experiments, the networks all had one hidden layer ($K=2$). The hidden layer had $n_1 = 300$ units, and the output layer had $n_2 =10$ units (one for each class).  For the MNIST and Fashion-MNIST datasets, the input size was
$n_{0} = 784$ 
 and,
for the SVHN and CIFAR-10 datasets the input size is 
$n_{0} = 3072$. 
The nonlinearity used in all the experiments was the logistic function
$\sigma(x) = 1/(1+e^{-x})$.
For all datasets, the objective function is defined as in Equation \eqref{emp-err}.
The number of training instances
was
$m = 60,000$ for MNIST and Fashion-MNIST,
$m=50,000$ for CIFAR-10, and
$m = 73,257$ in the SVHN experiment.
In all cases, a training pair 
$(y_n,t_n)$ consists of an image and a $10$ dimensional indicator vector representing the label for the image.

The details of the \textproc{dsgd} procedure are shown in
Algorithm \ref{algo:finslernn}. Note that the algorithm calculates different matrix norms and duality maps depending on the choice of $q$. For instance, when $q=2$, computing the polynomials $p_1$ involves computing the spectral norm of the weight-matrix $w_2$, while computing the norms of $g_1$ and $g_2$ uses the norm dual to the spectral norm, as defined in the second case of
Equation \eqref{mat-dual-norm}. 

For  experiments where \textproc{dsgd} is used in batch mode, the theoretically specified step-size $\epsilon=1$ was used.
In all other cases, the choice of step-size was determined  experimentally using a validation set (details of the validation procedure, as well as weight initialization, are deferred to an appendix.) 
In the batch experiments, the algorithm ran for $20,000$ weight updates. In the stochastic algorithms, each mini-batch had $128$ training examples, and training ran for $500$ epochs. 

\begin{figure}[!t]
\begin{center}
\includegraphics[width=1\linewidth]{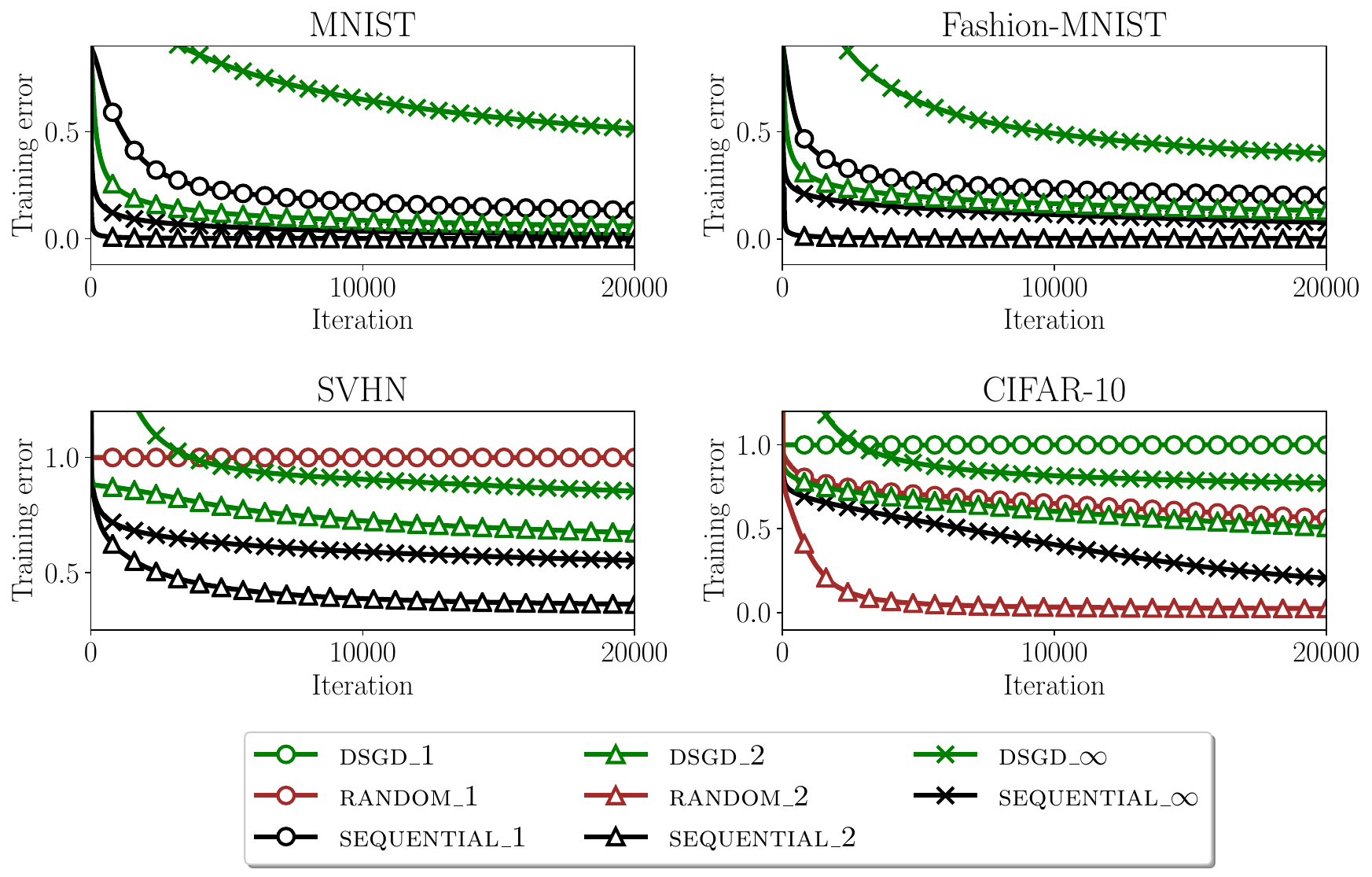}
\end{center}
\caption{
A  comparison of batch \textproc{dsgd} with layer-wise algorithms. For each dataset and choice of $q\in \{1, 2, \infty \}$ we plot the  training error for of \textproc{dsgd} as well as the best layer-wise algorithm among \textproc{random} and \textproc{sequential}.
 Best viewed in color. 
}\label{fig:batch_train}
\end{figure} 
\subsection{Batch DSGD using theoretically specified step-sizes}
We performed several experiments involving  \textproc{dsgd} in batch mode using the theoretically prescribed step-size $\epsilon=1$ from
Theorem \ref{main-thm-app}
in order to understand the practicality of the algorithm. This was achieved by comparing the  performance of \textproc{dsgd} with two  other layer-wise training algorithms termed \textproc{random} and \textproc{sequential}. In the
\textproc{random}  algorithm, the layer to update is chosen uniformly at random at each iteration. In the
\textproc{sequential} algorithm, the layer to update alternates deterministically at each iteration. For both \textproc{random} and \textit{sequential}, the step-sizes are chosen based on performance on a validation set. For each of the three algorithms (\textproc{dsgd}, \textproc{random}, and \textproc{sequential}), we repeated optimization using three different underlying norms $q=1,2,$ or $\infty$. Some of the results are shown in Figure \ref{fig:batch_train}, which indicates the trajectory of the training error over the course of optimization. Note that although \textproc{dsgd} does not have the best performance when measured in terms of final training error, it does carry the benefit of having theoretically justified step-sizes, while the other layer-wise algorithms use step-sizes defined through heuristics. An additional plot featuring the trajectory of testing accuracy for these experiments may be found in the appendix (Figure \ref{fig:batch_test}). 

As the \textproc{dsgd} algorithm selects the layer to update at runtime, based on the trajectory of weights, it may be of interest to consider how these updates are distributed.  This information is presented in Figure \ref{fig:counts} for the MNIST and CIFAR-10 datasets. Interestingly. when using the norm
$\|\cdot\|_1$,
 all updates occur in the second layer. For the norm
$\|\cdot\|_2$,
the rates of updates in each layer remained relatively constant throughout optimization. For
$\|\cdot\|_{\infty}$,
there was a greater range in the rate of updates in each layer as training progressed. 

Let us remark on the runtime performance of \textproc{dsgd} compared with the other layer-wise algorithms. Compared to standard \textproc{dsgd} has the additional step of computing the duality map and norms of the gradients, and the norm of the weights. However, for the batch algorithms the time per epoch  is dominated by  forward and backward passes over the network. The other layer-wise algorithms that were compared against also compute duality maps, but not norms of the weights. Due to this, epochs of \textproc{dsgd} are only about 2\% - 3\% slower than their \textproc{random} and \textproc{sequential} counterparts.   Data sets and code generated during the current study are available from the  author on reasonable request.
\begin{figure}[!t]
\begin{center}
\includegraphics[width=1\linewidth]{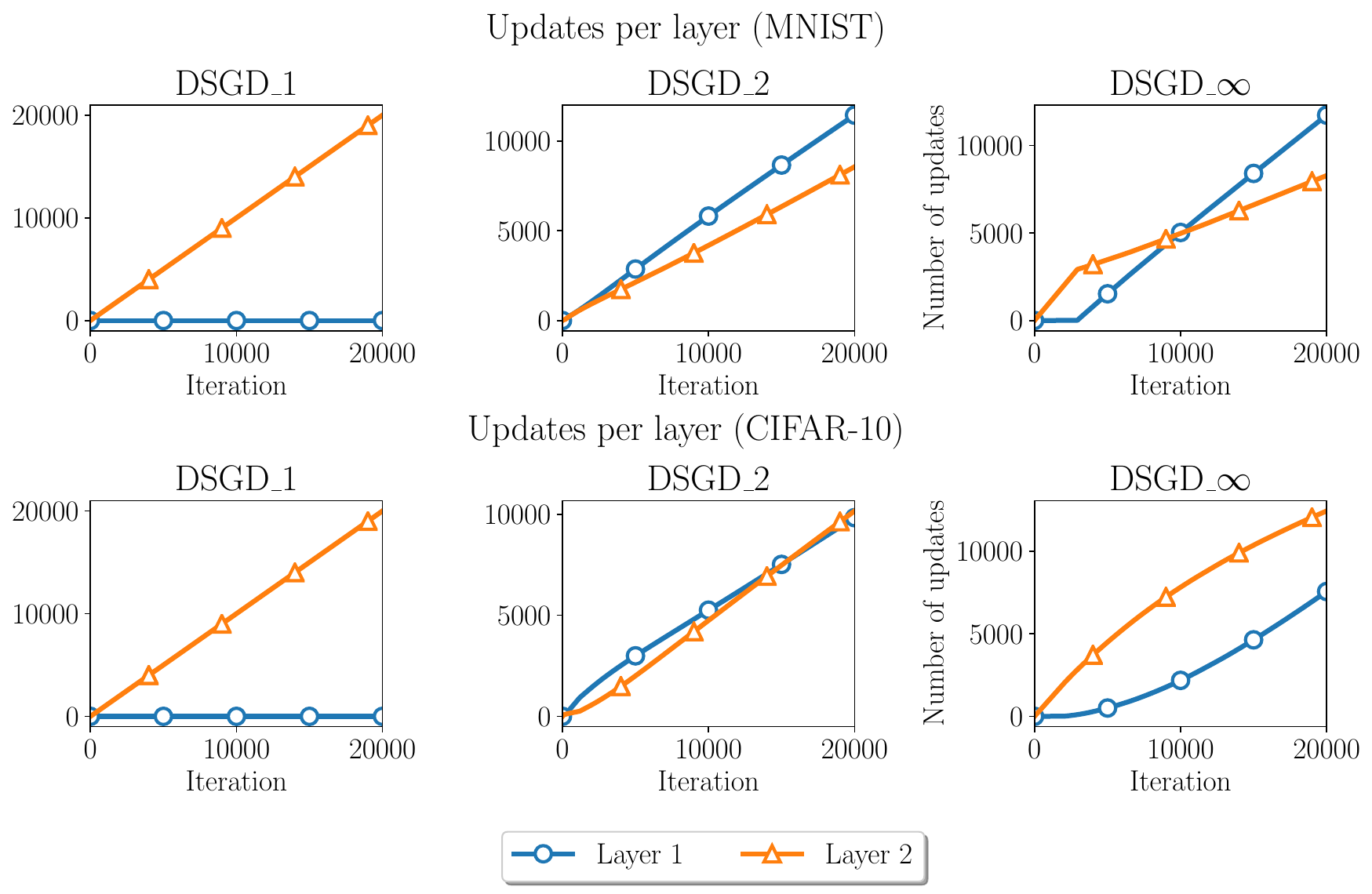}
\end{center}
\caption{
These figures show how frequently each layer was updated in the \textproc{dsgd} algorithm, for the case of batch training on the MNIST and CIFAR-10 datasets. The graphs show the running count of the number of updates by layer. Evidently, there is a range of behaviors depending on the choice of norm. Under the norm $\|\cdot\|_1$, all updates during optimization were confined to the second layer. For the norm $\|\cdot\|_2$, the rate of updates in each layer is more or less constant. For $\|\cdot\|_{\infty}$, there is more variation depending on the dataset and stage of optimization (beginning or end). 
}\label{fig:counts}
\end{figure}

\subsection{Practical variants of DSGD}
In this section we compared a variant of  \textproc{dsgd} with \textproc{sgd}. The variant of \textproc{dsgd} that we consider is termed \textproc{dsgd\_all}. In this algorithm, the step-sizes are computed as in Line \ref{update-step} of Algorithm \ref{algo:finslernn}, but the update is performed in both layers, instead of only one as is done in \textproc{dsgd}. This is to enable a more accurate comparison with algorithms that update both layers. The variants of \textproc{sgd} we considered were  standard \textproc{sgd} using Euclidean updates (\textproc{sgd\_standard}), and \textproc{sgd} using updates corresponding to the $\|\cdot\|_1, \|\cdot\|_2,$ and $\|\cdot\|_{\infty}$ norms. For all the algorithms, the step-size was determined using performance on a validation set, following the protocol set forth in the Appendix. The trajectories of testing accuracy for  the algorithms is shown in Figure \ref{fig:sgd_test}.
We observe that for all the datasets, the variant of \textproc{dsgd\_all} using the norm $\|\cdot\|_2$ performs the best among the \textproc{dsgd} algorithms. However, we also observed that \textproc{sgd\_standard} outperformed \textproc{dsgd}. Corresponding plots for the training error may be found in Figure \ref{fig:sgd_train} in the appendix. 

In terms of performance, \textproc{dsgd} requires more work at each update due to the requirement of computing the matrix norms. In the minibatch scenario, a higher percentage of time is spent on these calculations compared to the batch scenario, where the time-per-epoch was dominated by forward and backward passes over the entire dataset. Because of this, the \textproc{dsgd} algorithm corresponding to $\|\cdot\|_2$ is the slowest among the algorithms, taking about 4 times longer than standard \textproc{sgd}. For the \textproc{dsgd} variants using $\|\cdot\|_1$ or $\|\cdot\|_{\infty}$, calculations of the relevant matrix norms and duality maps can be done very efficiently, and hence these algorithms operate essentially at the same speed as standard \textproc{sgd}. This motivates future work into efficient variations of  \textproc{dsgd}, perhaps using approximate and/or delayed duality map and norm computations.

\begin{figure}[!t]
\begin{center}
\includegraphics[width=1\linewidth]{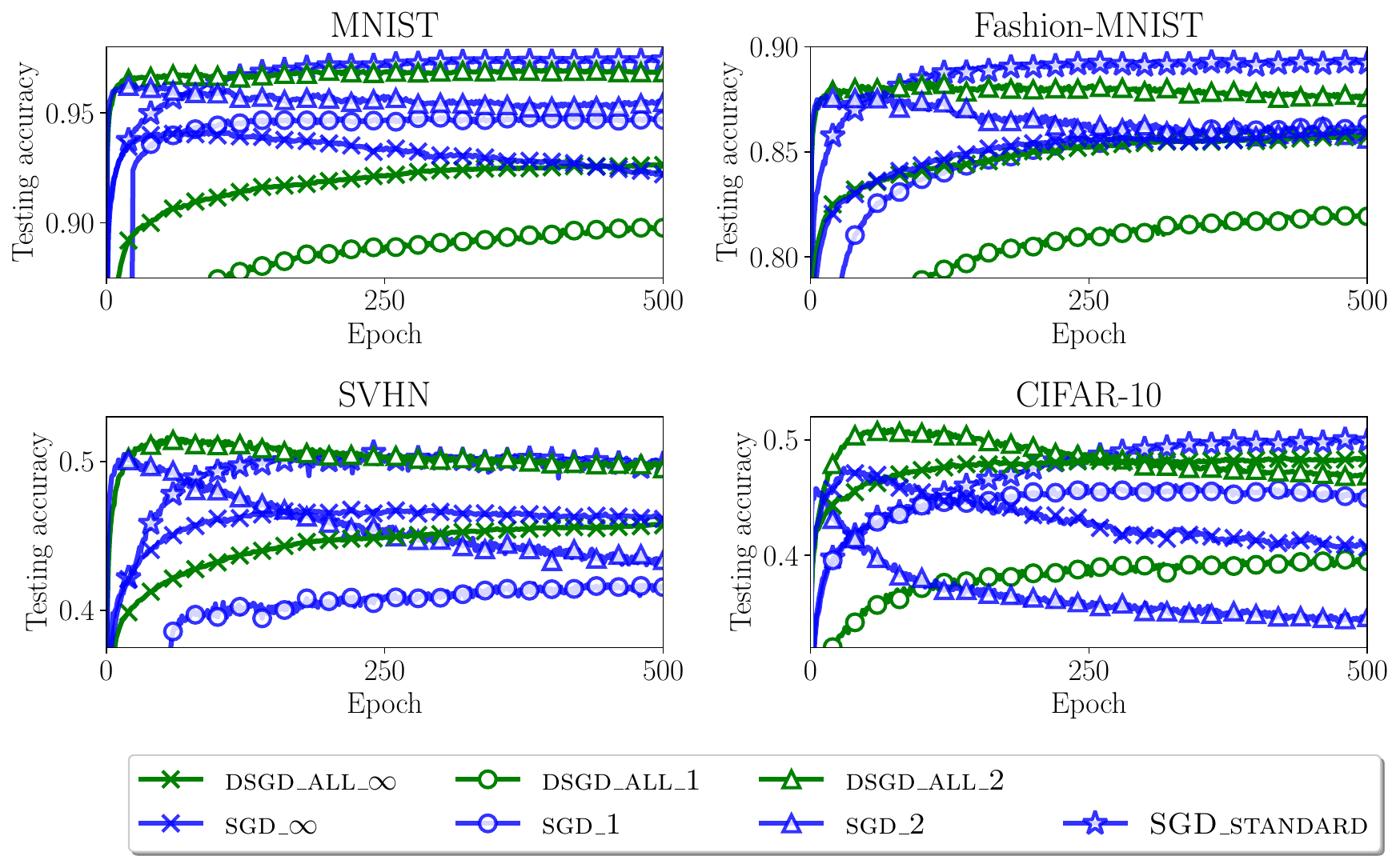}
\end{center}
\caption{
A comparison of \textproc{dsgd} with gradient descent and variants of gradient descent using several different norms.
Each figure plots the value of the testing accuracy  for the dataset. Best viewed in color. 
}\label{fig:sgd_test}
\end{figure}

  \section{Discussion}\label{sect:discuss}
  This work was motivated by the fact that  gradient smoothness assumptions used in certain optimization analyses may be too strict to be applicable in problems involving neural networks. To address this, we sought an algorithm  for training neural networks that is both practical and admits a non-asymptotic convergence analysis. Our starting point was the observation that the empirical error function for a multilayer network has a Layer-wise gradient smoothness property. We showed how a greedy algorithm that updates one layer at a time can be explained with a geometric interpretation involving a family of norms. That is, the steps of the algorithm (choosing one layer at a time in a greedy fashion) flow naturally from the gradient descent procedure, by using a certain family of norms (a geometric construct.) Different variants of the algorithm can be generated by varying the underlying norm on the state-space, and the choice of norm can have a significant impact on the practical efficiency. 

  Our abstract algorithmic framework  can in some cases provide non-asymptotic performance guarantees while making less restrictive assumptions compared to vanilla gradient decent.
In particular, the analysis does not assume that the objective function has a Lipschitz gradient in the usual Euclidean sense. The class of functions that the method applies to includes neural networks with arbitrarily many layers, subject to some mild conditions on the data set (the components of the input and output data should be bounded) and the activation function (boundedness of the derivatives of the activation function.)

Although it was expected that the method would yield step sizes that were too conservative to be competitive with standard gradient descent, this turned out not to be the case. This may be due to that that \textproc{dsgd} integrates more problem structure into the algorithm compared to standard gradient descent. Various problem data was used to construct the family of norms, such as the hierarchical structure of the network, bounds on various derivatives, and bounds on the input.

  \section*{Acknowledgments}
  This material is based upon work supported by the U.S.
Department of Energy, Office of Science, under contract
number DE-0012704.

\bibliography{super}
\clearpage
\appendix
\section*{Appendix}
\subsection*{Further experimental details}

 For algorithms other than batch  \textproc{dsgd},  we withheld 1/6 of the training data as  validation data for tuning step-sizes.  We ran gradient descent on the remaining 5/6 of the dataset, for each choice of
$\epsilon \in \{ 0.001, 0.01, 0.1, 1, 10 \}$, and evaluated the validation error after optimization. The step-size that gave the smallest validation error was used for the full experiments.  In all cases, network weights were initially  uniformly distributed in the interval $[-1,1]$.

\begin{figure}[h]
\begin{center}
  \includegraphics[width=0.8\linewidth]{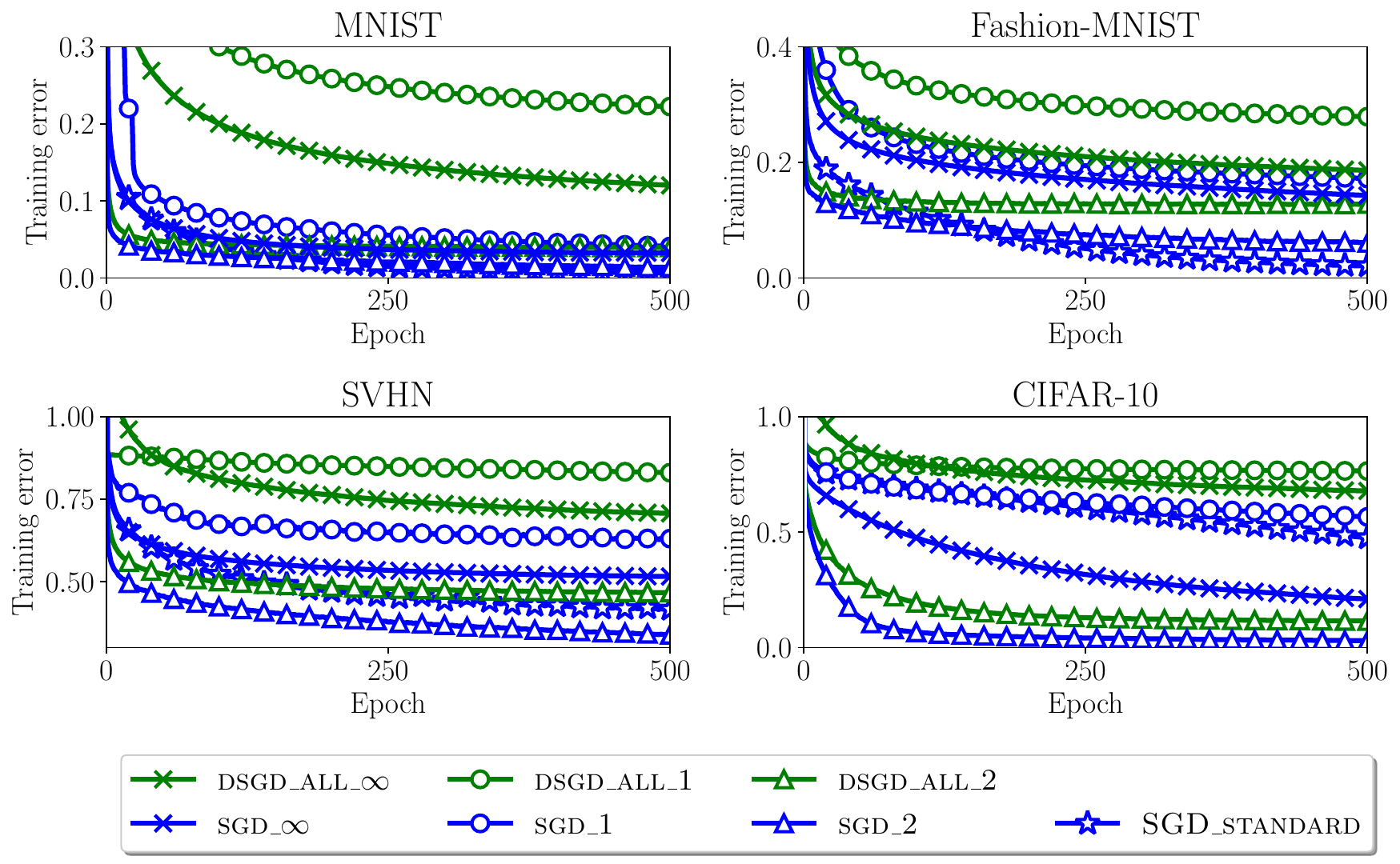}
  \vspace{-1em}
\end{center}
\caption{
A comparison of \textproc{dsgd} with gradient descent (\textproc{gd}) and variants of \textproc{gd} using several different norms.
The plots show the training error  for the dataset. Best viewed in color. 
}\label{fig:sgd_train}
\end{figure}
\vspace{-1em}
\begin{figure}[h]
\begin{center}
\includegraphics[width=0.8\linewidth]{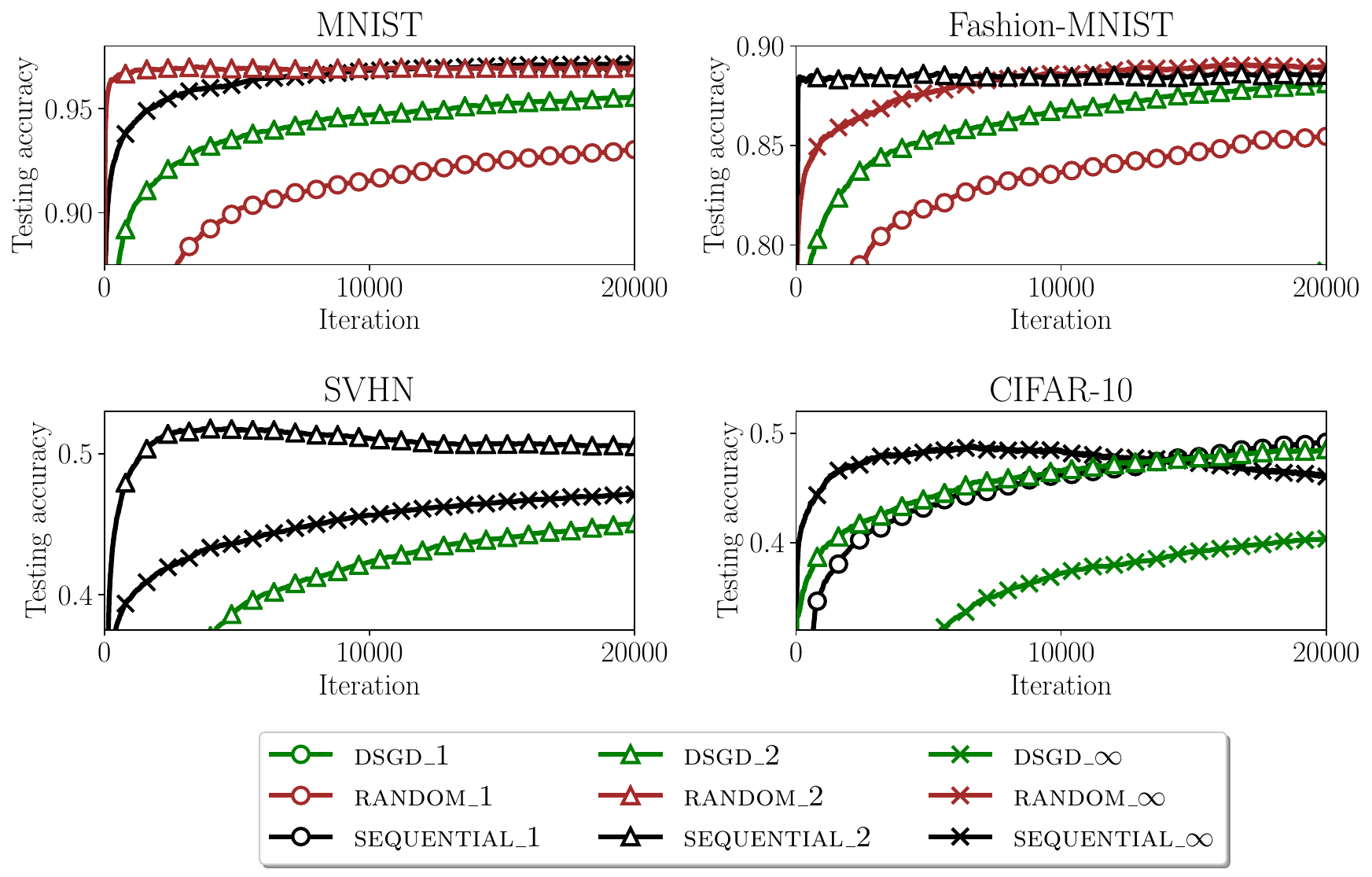}
\vspace{-1em}
\end{center}
\caption{
  A  comparison of batch \textproc{dsgd} with layer-wise algorithms. For each dataset and choice of $q\in \{1,2,\infty\}$ we plot the testing accuracy of \textproc{dsgd} as well as the best layer-wise algorithm among \textproc{random} and \textproc{sequential}.
 Best viewed in color. 
}\label{fig:batch_test}
\end{figure} 
\clearpage

\section*{Proofs of main results}
\renewcommand\thesection{A}

\subsection{Proof of Proposition \ref{prop:nn-hard}}
Let 
$w = (w_1,w_2,w_3,w_4)$  denote a particular choice of parameters.
The chain-rule gives
$$
\frac{\partial^{2} f}{\partial w_1\partial w_3}(w)
=
2y(w;1)\frac{\partial^{2} y}{\partial w_1\partial w_{3}}(w;1)
+2
\frac{\partial y}{\partial w_3}(w;1)
\frac{\partial y}{\partial w_1}(w;1).
$$
The derivatives of $y$ appearing in this equation are as follows:
\begin{subequations}
  \begin{align}
\frac{\partial y}{\partial w_3}(w;1) &=
  \sigma'(w_3 \sigma(w_1) + w_4\sigma(w_2))
  \sigma(w_1),\label{eq:df1} \\
\frac{\partial y}{\partial w_1}(w;1) &=
  \sigma'(w_3 \sigma(w_1) + w_4\sigma(w_2))
  w_3 \sigma'(w_1),
 \\
\begin{split}
  \frac{\partial^{2} y}{\partial w_1\partial w_{3}}(w;1) &= 
     \sigma''(w_3 \sigma(w_1) + w_4\sigma(w_2))
     \sigma(w_1)w_3\sigma'(w_1) \\
     &\quad+
     \sigma'(w_3 \sigma(w_1) + w_4\sigma(w_2))\sigma'(w_1).
\end{split}
  \end{align}
\end{subequations}
Let $z$ be any non-positive number, and define the curve 
$w: [0,\infty) \to \reals$ as 
$$w(\epsilon)= 
\left(1,1,\epsilon,\frac{1}{\sigma(1)}(z-\epsilon\sigma(1))\right).$$ Then
$$
\frac{\partial^{2} f}{\partial w_1\partial w_3}(w(\epsilon))
=
\bigg(2\sigma(z)\sigma''(z)\sigma(1)\sigma'(1) + 2\sigma'(z)\sigma(1)\sigma'(z)\sigma'(1) \bigg)\epsilon + \sigma'(z)\sigma'(1).
$$
Note that since $z$ is non-positive, we guarantee $\sigma''(z) \geq 0$, and therefore the coefficient of $\epsilon$ in this equation is positive. We conclude that $\lim_{\epsilon\to\infty}\frac{\partial^{2} f}{\partial w_1\partial w_3}(w(\epsilon)) = +\infty$.
\qed
\subsection{Proof of Lemma \ref{the-k-lemma}}
        Set
        $\ell_1 = \ell + \delta$ and
        $\ell_2 = -\ell$
        .
        Plugging these values into \eqref{lemma-6-eqn} of Lemma \ref{lemm-dual-ineq} yields
        \begin{equation}\label{uni-1}
          \|\delta \|^2 \geq
          \|\ell + \delta\|^2  -2\ell\cdot\rho(\ell+\delta) + c\|\ell\|^2
        \end{equation}
        Apply \eqref{lemma-6-eqn} again, this time with
        $\ell_1 = \ell$ and
        $\ell_2=\delta$,
        obtaining
        \begin{equation}\label{uni-2}
          \|\ell + \delta\|^2 \geq
          \|\ell\|^2 + 2\delta\cdot \rho(\ell) + c\|\delta\|^2
        \end{equation}
        Combining \eqref{uni-1} and \eqref{uni-2}, then,
        \begin{align*}
          \|\delta\|^2
          &\geq
            \left( \|\ell\|^2 + 2\delta\cdot\rho(\ell) + c\|\delta\|^2 \right)
            -2\ell\cdot\rho(\ell+\delta) + c\|\ell\|^2 \\
          &= (1+c)\|\ell\|^2 + 2\delta\cdot\rho(\ell) + c\|\delta\|^2  -2\ell\cdot\rho(\ell+\delta)  
        \end{align*}
        Rearranging terms and dividing both sides of the equation by two,
        \begin{align*}
          \ell\rho(\ell+\delta)
          &\geq
            \left(\frac{1+c}{2}\right)\|\ell\|^2 + \delta\cdot\rho(\ell) -
            \left(\frac{1-c}{2}\right)\|\delta\|^2
        \end{align*}
        Taking expectations, we obtain \eqref{bias-expect}.

        Next, applying \eqref{uni-cvx-eqn} with $x=2\ell$ and $y=2\delta$ yields
        \begin{equation}\label{beginner}
          \|\ell+\delta\|^2 \leq 2\left(\|\ell\|^2 + \|\delta\|^2\right) - c\|\ell-\delta\|^2
        \end{equation}
        Setting
        $\ell_1 = \ell$ and
        $\ell_2 = -\delta$ in \eqref{lemma-6-eqn} , we get
        \begin{equation}\label{advancer}
          \|\ell-\delta\|^2
          \geq \|\ell\|^2  - \delta\cdot\rho(\ell) + c\|\delta\|^2
        \end{equation}
        Combining \eqref{beginner} and \eqref{advancer},
        \begin{equation}
          \begin{split}
        \|\ell+\delta\|^2
        &\leq
        2\left(\|\ell\|^2  + \|\delta\|^2\right) -
        c\left(\|\ell\|^2  + \delta\cdot\rho(\ell) -c\|\delta\|^2\right) \\
        &=
        (2-c)\|\ell\|^2  + (2-c^2)\|\delta\|^2  + c\delta\cdot\rho(\ell)
        \end{split}
      \end{equation}
      Taking expectations gives \eqref{bias-norm}.
      \qed
\subsection{Proof of Theorem \ref{prop:sgd-prop}}
  By
  Assumption \ref{f-asu}, we know that
  \[
  f(w(t+1))
  \leq
  f(w(t))
  -
  \epsilon \frac{\partial f}{\partial w}(w(t))\cdot\rho_{w(t)}(g(t))
  +
  \epsilon^2\frac{L}{2}\|g(t)\|_{w(t)}^{2}.
  \]
  Using the definition of $\delta(t)$ given in Assumption \ref{bias-asu}, this is equivalent to
  \begin{equation}\label{pt1}
  f(w(t+1))
  \leq
  f(w(t))
  -
  \epsilon
  \frac{\partial f}{\partial w}(w(t))\cdot\rho_{w(t)}\left(\frac{\partial f}{\partial w}(w(t)) + \delta(t)\right)
  +
  \epsilon^2\frac{L}{2}\left\|\frac{\partial f}{\partial w}(w(t)) + \delta(t)\right\|_{w(t)}^{2}.
  \end{equation}                                
  Summing \eqref{pt1} over $t=1,2,\hdots,N$ yields
  \begin{equation}\label{pt1sum}
  \begin{split}
  f(w(N+1))
  \leq
  f(w(1))
  &-
   \epsilon  \sum\limits_{t=1}^{N}
  \frac{\partial f}{\partial w}(w(t))
  \cdot\rho_{w(t)}\left(\frac{\partial f}{\partial w}(w(t)) + \delta(t)\right) \\
  &\quad+
  \epsilon^2\frac{L}{2}\sum\limits_{t=1}^{N}\left\|\frac{\partial f}{\partial w}(w(t)) + \delta(t)\right\|_{w(t)}^{2}.
  \end{split}
  \end{equation}
  Rearranging terms, and noting that $f(w(N+1)) \geq f^{*}$,
  \begin{equation}\label{pt1re}
    \epsilon \sum\limits_{t=1}^{N}
  \frac{\partial f}{\partial w}(w(t))
  \cdot
  \rho_{w(t)}\left(\frac{\partial f}{\partial w}(w(t)) + \delta(t)\right)
  \leq
  G  +
  \epsilon^2\frac{L}{2}
  \sum\limits_{t=1}^{N}
  \left\|\frac{\partial f}{\partial w}(w(t)) + \delta(t)\right\|_{w(t)}^{2}.
  \end{equation}
  According to
  Equation \eqref{bias-expect}, for all $t$ it holds that 
  \begin{equation}\label{pt2}
  \mathbb{E}\left[
  \frac{\partial f}{\partial w}(w(t))
  \cdot
  \rho_{w(t)}\left(\frac{\partial f}{\partial w}(w(t)) + \delta(t)\right) \,\middle|\, \mc{F}(t-1)
  \right]                                      
  \geq
  \left(\frac{1+c}{2}\right)
  \left\|\frac{\partial f}{\partial w}(w(t))\right\|_{w(t)}^{2}
  -
  \left(\frac{1-c}{2}\right)\sigma^{2},
  \end{equation}                         
  while  Equation \eqref{bias-expect} implies
  \begin{equation}\label{pt3}
  \mathbb{E}\left[
  \left\|\frac{\partial f}{\partial w}(w(t)) + \delta(t)\right\|^{2}_{w(t)} \,\middle|\, \mc{F}(t-1)
  \right]
 \leq
  \left(2-c\right)\left\|\frac{\partial f}{\partial w}(w(t))\right\|^{2}_{w(t)}
  +
  \left(2-c^2\right)\sigma^{2}.
  \end{equation}
  For $n\geq 1$ we define the stopping time $ \tau \wedge n$ to be the minimum of $\tau$ and the constant value $n$.
  Applying
  Proposition \ref{mart-like-thm}, inequality \eqref{pt2}, and using the law of total expectation, it holds that for any $n$,
  \begin{equation}\label{mart-cons-1}
    \begin{split}
  \mathbb{E}\left[
  \sum\limits_{t=1}^{\tau \wedge n}
  \frac{\partial f}{\partial w}(w(t))
  \cdot
  \rho_{w(t)}\left(\frac{\partial f}{\partial w}(w(t)) + \delta(t)\right)
  \right]& \\
      &\hspace{-8em}\geq
  \mathbb{E}\left[
  \sum\limits_{t=1}^{\tau \wedge n}
  \left(
  \left( \frac{1+c}{2}\right)
  \left\|\frac{\partial f}{\partial w}(w(t))\right\|_{w(t)}^{2}
  -
  \left(\frac{1-c}{2}\right)\sigma^{2}
  \right)
      \right].
      \end{split}
  \end{equation}
  Applying
  Proposition \ref{mart-like-thm} a second time, in this case to inequality \eqref{pt3}, we see that  
  \begin{equation}\label{mart-cons-2}
  \mathbb{E}\left[
  \sum\limits_{t=1}^{\tau\wedge n }
  \left\|\frac{\partial f}{\partial w}(w(t)) + \delta(t)\right\|_{w(t)}^{2}
  \right]
  \leq
  \mathbb{E}\left[
  \sum\limits_{t=1}^{\tau\wedge n }\left(
  \left(2-c\right)\left\|\frac{\partial f}{\partial w}(w(t))\right\|^{2}_{w(t)}
  +
  \left(2-c^2\right)\sigma^{2}
  \right)\right].
  \end{equation}
Combining \eqref{pt1re} with \eqref{mart-cons-1} and \eqref{mart-cons-2} and rearranging terms,
\begin{equation}\label{rerr}
\begin{split}
\epsilon
  \left(\frac{1+c}{2} - \frac{L}{2}\epsilon(2-c)\right)
  \mathbb{E}\left[
  \sum\limits_{t=1}^{\tau \wedge n}
  \left\|\frac{\partial f}{\partial w}(w(t))\right\|_{w(t)}^{2}
  \right]
  \leq
G + \epsilon
\left(\frac{L}{2}\epsilon\left(2-c^2\right) + \frac{1-c}{2}\right)
\sigma^2\mathbb{E}[\tau \wedge n]
\end{split}
\end{equation}
Next, note that
\begin{equation}\label{note-oner}
\begin{split}
  \sum\limits_{t=1}^{\tau \wedge n}
  \left\|\frac{\partial f}{\partial w}(w(t))\right\|_{w(t)}^{2}
&\geq
  \sum\limits_{t=1}^{\tau \wedge n- 1}
  \left\|\frac{\partial f}{\partial w}(w(t))\right\|_{w(t)}^{2}
\\
&\geq
  \gamma
(\tau\wedge n - 1).
  \end{split}
\end{equation}
Combining \eqref{rerr} with \eqref{note-oner} yields
$$
\epsilon
\left(\frac{1+c}{2} - \frac{L}{2}\epsilon(2-c)\right)
\gamma
\mathbb{E}[ (\tau \wedge n) - 1]
\leq
G
+
\epsilon
\left(\frac{L}{2}\epsilon(2-c^2) + \frac{1-c}{2}\right)
\sigma^2\mathbb{E}[\tau \wedge n].
$$
By \eqref{step-k-cond},  this can be rearranged into
$$
\mathbb{E}[\tau \wedge n]
\leq
\frac{
2G + \epsilon\left(1+c - L\epsilon(2-c)\right)
\gamma
}{
 \epsilon
\left(1+c - L\epsilon(2-c)\right)
\gamma - \epsilon
\left(L\epsilon(2-c^2) + 1-c\right)\sigma^2
}.
$$
Since the right-hand side of this equation is independent of $n$, the claimed inequality \eqref{stoch-lemma} follows by the monotone convergence theorem. 
\qed
\subsection{Proof of Corollary \ref{non-asympt-sgd}}
We first consider case 2 of the corollary. Using the given value of $\epsilon$ in conjunction with \eqref{stoch-lemma}, 
\begin{equation}\label{tau-ineq-general}\begin{split}
  \mathbb{E}[\tau]
&\leq
\frac{
2G + \left(1+c - L\epsilon(2-c)\right)
\gamma
}{
 \epsilon
\left(1+c - L\epsilon(2-c)\right)
\gamma - \epsilon
\left(L\epsilon(2-c^2) + 1-c\right)\sigma^2
  } \\
&=
  \frac{
  4G + 2\left(1+c - L\epsilon(2-c)\right)
  \gamma
  }{
  \epsilon( (1+c)\gamma + (1-c)\sigma^2)
  }  \\
  &=
  \frac{
    \left(8LG + 4L\gamma\left(1+c - L\epsilon(2-c)\right)
  \right)((2-c)\gamma + (2-c^2)\sigma^2)
  }{
    ( (1+c)\gamma + (1-c)\sigma^2)^2
    } .
\end{split}\end{equation}
Note that $1+c-L\epsilon(2-c)  \leq 1+c  \leq 2$.
Therefore
\begin{equation}\label{num-ineq}\begin{split}
\left(8LG + 4L\gamma\left(1+c - L\epsilon(2-c)\right)
\right)((2-c)\gamma + (2-c^2)\sigma^2)
\leq \\
  8LG(2-c^2)\sigma^2 + 8LG(2-c)\gamma +  8L\gamma(2-c)\gamma + 8L\gamma(2-c^2)\sigma^2 
= \\
  8LG(2-c^2)\sigma^2 + 8L(G(2-c) + (2-c^2)\sigma^2)\gamma +  8L(2-c)\gamma^2 
\leq \\
    8LG(2-c^2)\sigma^2 + 8L(2-c^2)(G + \sigma^2)\gamma +  8L(2-c)\gamma^2. 
\end{split}\end{equation}
Combining \eqref{tau-ineq-general} with \eqref{num-ineq},
\begin{align*}
  \mathbb{E}[\tau]
  &\leq
    \frac{8LG(2-c^2)\sigma^2}
    {( (1+c)\gamma + (1-c)\sigma^2)^2}
    +
    \frac{8L(2-c^2)(G+\sigma^2)\gamma}{
    ( (1+c)\gamma + (1-c)\sigma^2)^2
    }
    +
    \frac{8L(2-c)\gamma^2}
    {( (1+c)\gamma + (1-c)\sigma^2)^2}.
\end{align*}
     Note that
    $\gamma \leq (1+c)\gamma + (1-c)\sigma^2$ implies

$$    \mathbb{E}[\tau]
  \leq
    \frac{8LG(2-c^2)\sigma^2}
    {( (1+c)\gamma + (1-c)\sigma^2)^2}
    +
    \frac{8L(2-c^2)(G+\sigma^2)}{
       (1+c)\gamma + (1-c)\sigma^2
    }
    +
    8L(2-c).
    $$

    The result for case 1 (standard \textproc{sgd}) follows by setting $c=1$ in the above equation.
    \qed
    \subsection{Proof of Corollary \ref{non-asympt-gd}}
For $t\geq 0$, set
$\eta(t) = \tfrac{\partial f}{\partial w}(w(t))$. Then
Assumption \ref{f-asu} implies
\begin{equation*}
f(w(t+1))
\leq
f(w(t))
-
\epsilon  \frac{\partial f}{\partial w}(w(t))
\cdot
\rho_{w}\left(
\frac{\partial f}{\partial w}(w(t))
\right)
+
\epsilon^{2}\frac{L}{2}\left\|\frac{\partial f}{\partial w}(w(t))\right\|^2_{w(t)}.
\end{equation*}
Invoking the duality map properties \eqref{dual-map-prop-a} and \eqref{dual-map-prop-b},
\begin{align*}
&\leq f(w(t))
-
\epsilon \left\|\frac{\partial f}{\partial w}(w(t))\right\|^{2}_{w(t)}
+
\frac{L}{2}\epsilon^{2}
\left\|\frac{\partial f}{\partial w}(w(t))\right\|^{2}_{w(t)} \\
&=
f(w(t)) + \epsilon\left(\frac{L}{2}\epsilon - 1\right)
 \left\|\frac{\partial f}{\partial w}(w(t))\right\|^{2}_{w(t)}.
\end{align*}
From the last inequality it is clear that the function decreases at  iteration $t$ unless 
$\frac{\partial f}{\partial w}(w(t))= 0$.
Summing our inequality over $t=1,2,\hdots, T$ yields
\begin{equation}\label{multi-iter-dec-pre}
 f(w(T))
 \leq
 f(w(1)) + 
 \epsilon
 \left(\frac{L}{2}\epsilon - 1\right)
 \sum\limits_{t=1}^{T}\left\|\frac{\partial f}{\partial w}(w(t))\right\|_{w(t)}^{2}.
\end{equation}
Upon rearranging terms and using that $f(w(T)) > f^{*}$, we find that
\begin{equation}\label{multi-iter-dec}
\sum\limits_{t=1}^{T}\left\|\frac{\partial f}{\partial w}(w(t))\right\|_{w(t)}^{2}
\leq \frac{2(f(w(1)) - f^{*})}{\epsilon(2 - L\epsilon ) }.
\end{equation}

Let $w^{*}$ be an accumulation point of the algorithm;  this is defined as a point such that for any
$\gamma>0$ the ball
$\{ w \in \reals^n \mid \|w - w^{*}\| < \gamma\}$ 
is entered infinitely often by the sequence $w(t)$ (any norm $\|\cdot\|$ can be used to define the ball.)
Then there is a subsequence of iterates
$w(m(1)), w(m(2)),\hdots$ with
$m(k) < m(k+1)$ such that
$w(m(k)) \rightarrow w^{*}$.
We know from \eqref{multi-iter-dec} that
$\|\frac{\partial f}{\partial w}(w(t))\|_{w(t)} \rightarrow 0$,
and the same must hold for any subsequence. Hence  
$\|\frac{\partial f}{\partial w}(w(m(k)))\|_{w(m(k))}\rightarrow 0$. 
As  the map $ (w,\ell) \mapsto \|\ell\|_{w}$ is continuous on
$\rn \times \mathcal{L}(\mathbb{R}^{n},\mathbb{R})$,(see for instance \citep[Proposition~27.7]{deimling1985nonlinear}) it must be that
$\|\frac{\partial f}{\partial w}(w^{*})\|_{w^{*}} = 0$  
Since $\|\cdot\|_{w^{*}}$ is a norm, then $\frac{\partial f}{\partial w}(w^{*}) = 0$.
\qed
\subsection{Proof of Proposition \ref{hess-bd-prop}}
Let us first recall  some notation for the composition of a bilinear map with a pair of linear maps:
  if
  $B : U  \times U \to V$ is a bilinear map then 
  $
  B
  (A_1 \oplus A_2)$ is the bilinear map which sends 
  $(z_1,z_2)$ to 
  $B( A_1 z_1, A_2 z_2)$.
  In addition,  if $B : U \times U \to V$  is a bilinear map,
then for any 
$(u_1,u_2) \in U\times U$ the inequality 
$\|B(u_1,u_2)\|_{V} \leq \|B\|\|u_1\|\|u_2\|$ holds. 
It follows that if $A_1 : Z\to U$ and $A_2 : Z\to U$ are any linear maps, then
\begin{equation}\label{bilin-ineq}
  \|B(A_1\oplus A_2)\| \leq \|B\|\|A_1\|\|A_2\|.
\end{equation}

To prove the proposition, it suffices to consider the case of a single input/output pair 
$(x,z) \in \mathbb{R}^{n_0}\times\mathbb{R}^{n_{K}}$.
In this case, we can express the function  $f$ as
\begin{equation}\label{f-def}
f(w) = J(y^{K}(x,w))
\end{equation}
where
$J(y) = \|y-z\|_{2}^{2}$
is the squared distance of a state $y$ to the target $z$ and
$y^{K}(x,w)$
is the output of a $K$-layer neural network with
input $x$. The
output $y^K$ is defined recursively as
\begin{align}\label{net-def}
\begin{split}
y^{i}(x,w) &= \begin{cases} h(y^{i-1}(x,w_{1:i-1}),w_{i})  &\text{ if } 2 \leq i \leq K, \\
  h(x,w_1) &\text{ if } i = 1,
  \end{cases}
\end{split}
\end{align}
where the function
$h(y,w)$ represents the computation performed by a single layer in the network:
\begin{equation}\label{hdef}
h_{k}(y,w) 
= 
\sigma
\bigg(
  \textstyle\sum\limits_{j=1}^{n}
    w_{k,j}y_j
\bigg), \quad  k= 1,2,\hdots, n.
\end{equation}
Taking the 
second derivative of \eqref{f-def} with respect to the weights
$w_i$ for $1 \leq i \leq K$, we find that
\begin{align}\label{df2-dw2} 
  \frac{\partial^{2} f}{\partial w_{i}^{2}}(w)
  &=
  \frac{\partial^{2} J}{\partial x^{2}}(y^{K}(x,w))
  \left(
  \frac{\partial y^{K}}{\partial w_{i}}(x,w)
  \oplus
    \frac{\partial y^{K}}{\partial w_{i}}(x,w)
  \right)
+
\frac{\partial J}{\partial y}(y^{K}(x,w))
  \frac{\partial^{2}y^{K}}{\partial w_{i}^{2}}(x,w).
\end{align}
To find bounds on these terms we will use the following identity:  for $0\leq k \leq i$,
\begin{equation}\label{x-id}
y^{i}( x, w_{1:i}) = y^{i-k}( y^{k}(x,w_{1:k}), w_{k+1:i})
\end{equation}
with the convention that $y^{0}(x) = x$.
Differentiating Equation \eqref{x-id}, with respect to $w_i$ for
$1 \leq i\leq K$ gives
\begin{equation}\label{dx-dw}
\frac{\partial y^{K}}{\partial w_{i}}(x,w_{1:K}) 
= 
\frac{\partial y^{K-i}}{\partial x}(y^{i}(x,w_{1:i}),w_{i+1:K})
\frac{\partial h}{\partial w}(y^{i-1}(x,w_{1:i-1}),w_{i})
\end{equation}
and  differentiating a second time yields
\begin{align}\label{dx2-dw2}
\begin{split}
&\frac{\partial^{2} y^{K}}{\partial w_{i}^{2}}(x,w_{1:K}) =\\&\quad\quad
\frac{\partial^{2} y^{K-i}}{\partial x^{2}}(y^{i}(x,w_{1:i}),w_{i+1:K})
\left( \frac{\partial h}{\partial w}(y^{i-1}(x,w_{1:i-1}),w_{i}) 
      \oplus \frac{\partial h}{\partial w}(y^{i-1}(x,w_{1:i-1}),w_{i})
\right) \\
&\quad+ 
\frac{\partial y^{K-i}}{\partial y}(y^{i}(x,w_{1:i}),w_{i+1:K})
\frac{\partial^{2} h}{\partial w^{2}}(y^{i-1}(x,w_{1:i-1}),w_{i}).
\end{split}
\end{align}
Next, we consider the terms 
$\frac{\partial y^n}{\partial x}$ and 
$\frac{\partial^{2}y^{n}}{\partial x^{2}}$
appearing in the two preceding equations \eqref{dx-dw}, \eqref{dx2-dw2}.
By differentiating equation \eqref{net-def} with respect to the input parameter, we have, for any input $u$ and parameters $a_1,a_2,\hdots,a_{n}$,
\begin{equation}\label{der-x-der-u}
\frac{\partial y^{n}}{\partial x}(u,a_{1:n}) 
= 
\frac{\partial h}{\partial y}\left(x^{n-1}\left(u,a_{1:n-1}\right),a_n\right)
\frac{\partial y^{n-1}}{\partial x}(u,a_{1:n-1}),
\end{equation}
and upon differentiating a second time,
\begin{align}\label{der-x2-der-u2}
\begin{split}
\frac{\partial^{2} y^{n}}{\partial x^{2}}(u,a_{1:n}) &=
\frac{\partial^{2}h}{\partial y^{2}}(y^{n-1}(u,a_{1:n}),a_n)\left(
\frac{\partial y^{n-1}}{\partial x}(u,a_{1:n-1}) \oplus
\frac{\partial y^{n-1}}{\partial x}(u,a_{1:n-1})\right)
\\&\quad+ 
\frac{\partial h}{\partial y}(y^{n-1}(u,a_{1:n-1}),a_n)
\frac{\partial^{2} y^{n-1}}{\partial x}(u,a_{1:n-1}).
\end{split}
\end{align}
We will use some bounds on
$h$ in terms of the norm
$\|\cdot\|_{q}$. 
It follows from
Lemma \ref{nn-bounds}  that the following bounds hold for any
$1 \leq q \leq \infty$:
\begin{align}\label{one-layer-der}
\begin{split}
\left\|\frac{\partial h}{\partial y}(y,w)\right\|_{q} 
&\leq 
\|\sigma'\|_{\infty}\|w\|_{q}, \quad
\left\|\frac{\partial h}{\partial w}(y,w)\right\|_{q} 
\leq 
\|\sigma'\|_{\infty}\|y\|_{q}, \\\\
\left\|\frac{\partial^{2}h}{\partial y^{2}}(y,w)\right\|_{q} 
&\leq 
\|\sigma''\|_{\infty}\|w\|_{q}^{2}, \quad
\left\|\frac{\partial^{2}h}{\partial w^{2}}(x,w)\right\|_{q} 
\leq 
\|\sigma''\|_{\infty}\|y\|_{q}^{2}.
\end{split}
\end{align}
Combining \eqref{der-x-der-u} with \eqref{one-layer-der} we obtain the following inequalities: For 
$n > 1$,
\begin{equation}\label{rec1}
\left\|\frac{\partial y^{n}}{\partial x}(u,a_{1:n})\right\|_{q} 
\leq 
  \begin{cases}
\|\sigma'\|_{\infty}
\|a_{n}\|_{q}
\left\|\frac{\partial y^{n-1}}{\partial x}(u,a_{1:n-1})\right\|_{q}
&\text{ if } n > 1, \\
\|\sigma'\|_{\infty}\|a_{1}\|_{q} &\text{ if } n = 1.
\end{cases}
\end{equation}
Combining the two cases in inequality \eqref{rec1}, and using the definition of $r_n$ we find that, for $n\geq 1$,
\begin{align}\label{r-def}
\begin{split}
\left\|\frac{\partial y^{n}}{\partial x}(u,a_{1:n})\right\|_{q}
&\leq 
r_{n}(\|a_1\|_{q},\hdots,\|a_n\|_{q}).
\end{split}
\end{align}

Now we turn to the second derivative $\frac{\partial^{2} u^{n}}{\partial x^{2}}$. 
Taking norms in Equation
\eqref{der-x2-der-u2}, and applying \eqref{one-layer-der}, \eqref{r-def}, and \eqref{bilin-ineq},
we obtain the following inequalities. If $n>1$,
\begin{equation*}
  \begin{split}
\left\|\frac{\partial^{2} y^{n}}{\partial x^{2}}(u,a_{1:n})\right\|_{q} &\\
&\hspace{-5em}\leq
\|\sigma''\|_{\infty}
\|a_{n}\|_{q}^{2}
\|\sigma'\|_{\infty}^{2(n-1)}\prod\limits_{i=1}^{n-1}\|a_i\|_{q}^{2} 
+ 
\|\sigma'\|_{\infty}\|a_n\|_{q}
    \left\|\frac{\partial^{2} x^{n-1}}{\partial y^{2}}(u,a_{1:n-1})\right\|_{q}.
    \end{split}
\end{equation*}
While for  $n=1$,
$\left\|\frac{\partial^{2} y^{n}}{\partial x^{2}}(u,a_{1:n})\right\|_{q}
\leq  \|\sigma''\|_{\infty}\|a_{1}\|_{q}^{2}$.
By definition of $v_n$, then, for all $n>0$,
\begin{equation}\label{q-def}
\left\|\frac{\partial^{2} y^{n}}{\partial x^{2}}(u,a_{1:n})\right\|_{q}
\leq 
v_{n}(\|a_{1}\|_{q},\hdots,\|a_{n}\|_{q}).
\end{equation}
Combining \eqref{dx-dw}, \eqref{one-layer-der}, and \eqref{r-def},
\begin{align}\label{dx-dw-sigma}
\begin{split}
\left\|\frac{\partial y^{K}}{\partial w_{i}}(x,w_{1:K})\right\|_{q}
&\leq 
\left\|
  \frac{\partial y^{K-i}}{\partial y}(y^{i}(x,w_{1:i}),w_{i+1:K})
\right\|_{q}
  \|\sigma'\|_{\infty}
  \|y^{i-1}(x,w_{1:i-1})\|_{q} \\
&\leq r_{K-i}(\|w_{i+1}\|_{q},\hdots,\|w_{K}\|_{q})\|\sigma'\|_{\infty}c_{q}
\end{split}
\end{align}
where the number $c_{q}$, defined in
Table \ref{table:defs} is the $q$-norm of the vector $(1,1,\hdots,1) \in  \reals^{n_{K}}$.

Combining \eqref{dx2-dw2},  
\eqref{one-layer-der}, \eqref{r-def}, and \eqref{q-def}, 
\begin{align}\label{dx2-dw2-r}
\begin{split}
&  \left\|
    \frac{\partial^{2} y^{K}}{\partial w_{i}^{2}}(x,w_{1:K})
  \right\|_{q}
 \\
&\leq\left\|
  \frac{\partial^{2} y^{K-i}}{\partial x^{2}}
  (y^{i}(x,w_{1:i}),w_{i+1:K})
\right\|_{q}
\|\sigma'\|_{\infty}^{2}c_{q}^{2}
+ 
\left\|
  \frac{\partial y^{K-i}}{\partial x}(y^{i}(x,w_{1:i}),w_{i+1:K})
\right\|_{q}
\|\sigma''\|_{\infty}c_{q}^{2}
\\
&\leq
v_{K-i}(\|w_{i+1}\|_{q},\hdots,\|w_{K}\|_{q})\|\sigma'\|_{\infty}^{2}c_{q}^{2}
+ 
r_{K-i}(\|w_{i+1}\|_{q},\hdots,\|w_{K}\|_{q})\|\sigma''\|_{\infty}c_{q}^{2}.
\end{split}
\end{align}
Now we arrive at bounding the derivatives of the function $f$. 
As shown in
Lemma \ref{euclidean-error} in the appendix, the following inequalities hold:
\begin{subequations}
  \begin{equation}
    \sup_{w,x}
    \left\|
      \frac{\partial J}{\partial y}(y^{K}(x,w))
    \right\|_{q} \leq d_{q,1},\label{dq1prop}
  \end{equation}
  \begin{equation}
    \sup_{w,x} \left\|\frac{\partial^{2} J}{\partial y^{2}}(y^{K}(x,w))\right\|_{q}  = d_{q,2}.\label{dq2prop}
  \end{equation}
\end{subequations}
where $d_{q,1}$, and $d_{q,2}$ are as in
Table \ref{table:defs}.
 Combining \eqref{df2-dw2},\eqref{dx-dw-sigma},\eqref{dx2-dw2-r}, \eqref{dq1prop} and \eqref{dq2prop},
 it holds that for $i=1,\hdots,K$,
\begin{align*}
\left\|\frac{\partial^{2} f}{\partial w_{i}^{2}}(x,w_{1:K})\right\|_{q} &\leq 
d_{q,2}\left\|\frac{\partial y^{K}}{\partial w_{i}}(x,w)\right\|_{q}^{2}
+
d_{q,1}\left\|\frac{\partial^{2}y^{K}}{\partial w_{i}^{2}}(x,w)\right\|_{q}
\\
&\leq
d_{q,2}c_{q}^{2}\|\sigma'\|_{\infty}^{2}r_{K-i}^{2}(\|w_{i+1}\|_{q},\hdots,\|w_{K}\|_{q})  
\\&\quad\quad + 
d_{q,1}c_{q}^{2}\|\sigma'\|_{\infty}^{2} v_{K-i}(\|w_{i+1}\|_{q},\hdots,\|w_{K}\|_{q}) 
\\&\quad\quad +
    d_{q,1}c_{q}^{2}
    \|\sigma''\|_{\infty}r_{K-i}(\|w_{i+1}\|_{q},\hdots,\|w_{K}\|_{q})
  \\
&= s_{K-i}(\|w_{i+1}\|_{q},\hdots,\|w_K\|_{q}) \\
&<  p_{i}(w)^{2}.
\end{align*}
\qed
\subsection{Proof of Proposition \ref{simple-gd-lem}}
  Let for any $w\in\reals^n, \Delta \in \reals^n$ and let $\epsilon >0$.
  Applying the fundamental theorem of calculus, first on the function $f$ and then on its derivative, we have
  \begin{align*}
  f(w-\epsilon\Delta)
  &=
  f(w) - \epsilon \int_{0}^{1}
  \frac{\partial f}{\partial w}(w-\lambda \epsilon \Delta)\cdot \Delta \, d\lambda \, \\
  &=
  f(w) - \epsilon \int_{0}^{1}
  \left[
    \frac{\partial f}{\partial w}(w)\cdot \Delta
    -
    \epsilon
    \int_{0}^{\lambda}
    \frac{\partial^{2} f}{\partial w}(w-u \epsilon\Delta)\cdot( \Delta,\Delta)\,
    du
    \right]\,d\lambda \\
  &=
  f(w) -
  \epsilon \frac{\partial f}{\partial w}(w)\cdot\Delta +
  \epsilon^2
  \int_{0}^{1}
  \int_{0}^{\lambda}
  \frac{\partial^{2} f}{\partial w^{2}}(w + u\epsilon\Delta)\cdot(\Delta,\Delta)
  \, du
  \, d\lambda.
  \end{align*}
  Using our assumption on the second derivative,
  $$
  \leq  f(w) -
  \epsilon \frac{\partial f}{\partial w}(w)\cdot\Delta +
  \epsilon^2\frac{L}{2}\|\Delta\|^{2}.
  $$
  Letting
  $\Delta = \rho(\frac{\partial f}{\partial w}(w))$,
  and using the two defining equations of duality maps (
  \eqref{dual-map-prop-a} and \eqref{dual-map-prop-b}), we see that
  $$
  \leq  f(w) -
  \epsilon \left\|\frac{\partial f}{\partial w}(w)\right\|^{2} +
  \epsilon^2\frac{L}{2}\left\|\frac{\partial f}{\partial w}(w)\right\|^{2}.
  $$
  Combining the terms yields the result.
  \qed
\subsection{Proof of Proposition \ref{one-space-dual}}
Let $\ell$ be given and consider $q=\infty$. For any matrix $A$ with $\|A\|_{\infty}=1$,
\begin{equation}\label{linear-ineq}
\ell(A) = \sum\limits_{i=1}^{r}\sum\limits_{j=1}^{c}\ell_{i,j}A_{i,j} \leq \sum\limits_{i=1}^{r}\max_{1\leq j\leq c}|\ell_{i,j}|
\end{equation}
since each row sum
$\sum\limits_{j=1}^{c}|A_{i,j}|$
is at most $1$. Let $m$ be the matrix defined in equation \eqref{mat-duality}. Clearly this matrix has maximum-absolute-row-sum $1$. Furthermore,
\begin{equation}\label{fmore}
  \ell(m) = \sum\limits_{i=1}^{r}\max_{1\leq j \leq c}|\ell_{i,j}|.
\end{equation}
 Combining equation \eqref{fmore} with the inequality \eqref{linear-ineq} confirms that the dual norm is
$
\|\ell\|_{\infty} 
= 
\sum\limits_{i=1}^{r}\max_{1\leq k\leq c}|\ell_{i,k}|
$.

For $q=2$, that the duality between the spectral norm is the sum of singular values, or trace norm, is well-known, and can be proved for instance as in the proof of Theorem 7.4.24 of \citep{matanal}.

For the duality maps, let the matrix $\rho_{\infty}(\ell)$ be defined as in \eqref{mat-duality}.
Then
$$
\|\rho_{\infty}(\ell)\|_{\infty}
=
\|\ell\|_{\infty}\|m\|_{\infty}
=
\|\ell\|_{\infty}
$$
and
$
\ell \cdot \rho_{\infty}(\ell)
=
\|\ell\|_{\infty}\ell(m)
=
\|\ell\|_{\infty}^{2}.
$

For $q = 2$, let the $\ell = U\Sigma V^{T}$ be the singular value decomposition of $\ell$, and let the matrix $\rho_{2}(\ell)$ be defined as in \eqref{mat-duality-two}. Then
$\ell\cdot \rho_{2}(\ell) = \|\ell\|_{2}\ell(A)$
where $A$ is the matrix
$A  = \sum\limits_{i=1}^{\rank \ell}u_iv_i^{T}$. It remains to show that
$ \ell(A) =   \sum\limits_{i=1}^{\rank \ell}\sigma_{i}(\ell)$:
\begin{align*}
  \ell(A)
  &=
  \tr \left( (U\Sigma V^{T})^{T} A \right)  \\
  &=
  \tr \left( V \Sigma^{T} U^{T} UV^{T} \right) \\
  &=
  \tr
  \left( \left(\sum\limits_{i=1}^{\rank \ell}\sigma_{i}(\ell)v_{i}u_{i}^{T}\right)\left(\sum\limits_{j=1}^{\rank \ell}u_{j}v_{j}^{T}\right)\right) \\
  &=
  \tr\left(
  \sum\limits_{j=1}^{\rank \ell}\sum\limits_{i=1}^{\rank \ell}\sigma_{i}(\ell)v_{i}u_{i}^{T}u_{j}v_{j}^{T}
  \right) \\
  &=
  \tr\left(\sum\limits_{i=1}^{\rank \ell}\sigma_{i}(\ell)v_{i}v_{i}^{T}\right) \\
  &=  \sum\limits_{i=1}^{\rank \ell}\sigma_{i}(\ell). 
\end{align*}
In the second to last inequality we used the fact that the columns of $U$ are orthogonal. In the last inequality we used the linearity of trace together with the fact that the columns of $V$ are unit vectors (that is, $\tr(v_{i}v_i^{T}) = 1$.)
\qed
\subsection{Proof of Proposition \ref{duality-product}}
First we compute the dual norm on $Z$.
For any 
$u = (u_1,\hdots,u_K) \in X_{1} \times \hdots X_K$
with 
$\|(u_1,\hdots,u_K)\|_Z=1$,
we have
\begin{align*}
\ell\cdot u  = 
(\ell_1 \cdot u_1) + \hdots + (\ell_K \cdot u_K) &\leq
\frac{p_1}{p_1}\|\ell_1\|_{X_1} \|u_1\|_{X_1} 
+ \hdots 
+ \frac{p_K}{p_K}\|\ell_K\|_{X_K} \|u_K\|_{X_K} \\&\leq 
(p_1\|u_1\|_{X_1} + \hdots + p_{K}\|u_K\|_{X_{K}})
\max_{1\leq i\leq K}\left\{\frac{1}{p_i}\|\ell_i\|_{X_i}\right\} \\
&=
\max_{1\leq i\leq K}\left\{\frac{1}{p_i}\|\ell_i\|_{X_i}\right\}.
\end{align*}
Therefore,
\begin{equation}\label{ubd}
\|\ell\|_{Z} \leq \max_{1\leq i\leq K}\left\{\frac{1}{p_i}\|\ell_i\|_{X_i}\right\}.
\end{equation}
Define $i^{*}$ as
$$i^{*} = \argmax_{1\leq i\leq K}\left\{\frac{1}{p_i}\|\ell_i\|_{X_i}\right\}.$$
To show that \eqref{ubd} is in fact an equality, consider the vector
$
u$ defined as
$$u
=
(u_1,\hdots,u_K) =
\left(
0,\hdots, \tfrac{1}{\|\ell_{i^{*}}\|p_{i^{*}}}\rho_{X_{i^{*}}}(\ell_{i^{*}}), \hdots, 0
\right).
$$
 The norm of this vector is 
$$
\|u\|_{Z}
=
p_{i^*}\frac{1}{\|\ell_{i^*}\|p_{i^*}}
\|\rho_{X_{i^*}}(\ell_{i^*})\|
=
\frac{1}{\|\ell\|_{i^*}}\|\ell_{i^*}\|
= 1,
$$
and 
$$
\ell\cdot u =
\frac{1}{\|\ell_{i^{*}}\|p_{i^*}}
\ell_{i^*}\cdot\rho_{X_{i^{*}}}(\ell_{i^*})
=
\frac{1}{\|\ell_{i^{*}}\|p_{i^*}}
\|\ell_{i^{*}}\|^{2}
=
\frac{\|\ell_{i^{*}}\|_{X_{i^*}}}{p_{i^*}}
$$
Therefore the dual norm is given by \eqref{the-dual-norm-gen}.

Next, we show that the function $\rho_Z$ defined in Proposition \ref{duality-product} is a duality map, by verifying the conditions \eqref{dual-map-prop-a} and \eqref{dual-map-prop-b}. Firstly,
\begin{align*}
\ell\cdot \rho_{Z}(\ell)
= 
\ell\cdot\left(0, \hdots, \frac{1}{\left(p_{i^{*}}\right)^{2}}\rho_{X_{i^{*}}}(\ell_{i^{*}}),\hdots , 0\right) 
&= 
\frac{1}{\left(p_{i^{*}}\right)^{2}}\ell_{i^{*}}\cdot\rho_{X_{i^{*}}}(\ell_{i^{*}}) 
\\&= 
\frac{1}{\left(p_{i^{*}}\right)^2}\|\ell_{i^{*}}\|_{X_{i^{*}}}^{2} \\
&= \|\ell\|_{Z}^{2}.
\end{align*}
This shows that \eqref{dual-map-prop-b} holds. It remains to show 
$\|\rho_{Z}(\ell)\|_{Z} \, = \|\ell\|_{Z}$.
By definition of $i^{*}$, we have
$$
\rho_{Z}(\ell) = \left(0,\hdots,\frac{1}{\left(p_{i^{*}}\right)^{2}}\rho_{X_{i^{*}}}(\ell_{i^{*}}),\hdots,0\right)
$$
so
$$
\|\rho_{Z}(\ell)\|_{Z} 
= 
p_{i^{*}}\frac{1}{\left(p_{i^{*}}\right)^{2}}\|\rho_{X_{i^{*}}}(\ell_{i^{*}})\|_{X_{i^{*}}} 
=  
\frac{1}{p_{i^{*}}}\|\ell_{i^{*}}\|_{X_{i^{*}}}
= \|\ell\|_{Z}.
$$
\qed
\subsection{Proof of Lemma \ref{lip-prop-nn}}
 Let
  $w\in W$ and
  $\eta \in \mc{L}(W,\reals)$
  be arbitrary
  . 
Let 
$i^{*} 
= 
\argmax_{1\leq i \leq K} 
\left\{\tfrac{1}{p_i(w)}\|\eta_i\|_{q}\right\}
$.
Then
$\rho_w(\eta) $
is of the form
$\rho_w(\eta) = (0, \hdots, \Delta_{i^*},\hdots,0)$,
where
$\Delta_{i^{*}} \in W_{i^{*}}$
is
$\Delta_{i^*} = \frac{1}{p_{i^*}(w)^2}\rho_q(\eta_{i^*})$.
Applying Taylor's theorem, it holds that
\begin{equation}\label{e1}
f(w + \epsilon \rho_w(\eta))
=
f(w)
+\epsilon \frac{\partial f}{\partial w}(w)\cdot\rho_w(\eta)
+ \epsilon^{2}
\int_{0}^{1}\int_{0}^{\lambda}
\frac{\partial^{2} f}{\partial w^{2}}
\left(w + u \epsilon \rho_w(\eta)\right)\cdot\left(\rho_w(\eta),\rho_w(\eta)\right) \,
du \, d\lambda.
\end{equation}
The only components of $\rho_w(\eta)$ that are potentially non-zero are those corresponding to layer $i^{*}$. Then
\begin{equation}\label{e2}
  \frac{\partial^{2} f}{\partial w^{2}}
  \left(w + u\epsilon \rho_w(\eta)\right)\cdot\left(\rho_w(\eta),\rho_w(\eta)\right)  =
  \frac{\partial^{2} f}{\partial w_{i^*}^{2}}
  \left(w + u\epsilon \rho_w(\eta)\right)\cdot
  \left(\Delta_{i^*},\Delta_{i^*}\right).
\end{equation}
According to
Proposition \ref{hess-bd-prop},
\begin{equation}\label{e3}
\left|
\frac{\partial^{2} f}{\partial w_{i^*}^{2}}(w + u\epsilon \rho_w(\eta))
\cdot
     \left(\Delta_{i^{*}},\Delta_{i^{*}}\right)
\right| \leq
p_{i^{*}}(w + u\epsilon\rho_{w}(\eta))^{2}\|\Delta_{i^{*}}\|_q^{2}.
\end{equation}
Since the function $p_{i^{*}}$ only depends on the weights in layers $(i^{*}+1),(i^* + 2),\hdots,K$, it holds that
\begin{equation}\label{e4}
  p_{i^{*}}(w + u\epsilon  \rho_w(\eta)) = p_{i^{*}}(w).
\end{equation}
By the definition of the dual norm \eqref{the-dual-norm}
\begin{equation}\label{e5}
  p_{i^*}(w^{*})\|\Delta_{i^*}\|_q = \|\eta\|_{w}.
\end{equation}
By combining Equations \eqref{e1} - \eqref{e5}, then,
$$
\left|f(w + \epsilon \rho_w(\eta))
-
f(w)
+\epsilon \frac{\partial f}{\partial w}(w)\cdot\rho_w(\eta)
\right|
\leq
\epsilon^{2}
\frac{1}{2}
\|\eta\|_w^{2}.
$$
\qed
\subsection{Proof of Lemma \ref{nn-grad-var}}

  Define the norm  $\|\cdot\|_{1,q}$ on
  $\mc{L}(W,\reals) = \mc{L}(W_1 \times \hdots \times W_K,\reals)$ as
  \begin{equation}\label{weird-norm}
    \|(\ell_1,\hdots,\ell_K)\|_{1,q} = \max_{1\leq i \leq K}\|\ell_{i}\|_{q}.
  \end{equation}
  For each $w\in W$ there is a linear map $A(w(t))$ on $W$ such that the norm
  $\|\cdot\|_{w(t)}$ on the dual space
  $\mc{L}(W,\reals)$ can be represented as
  \begin{equation}\label{repr}
  \|\ell\|_{w(t)} = \|A(w(t))\ell\|_{1,q}.
  \end{equation}
  This can be deduced from inspecting the formula \eqref{the-dual-norm}.
  Although not material for our further arguments, $A(w(t))$ is a block-structured matrix,
  with coefficients $A(w(t))_{i,j} = 0$ whenever $i,j$ correspond to parameters in separate layers, and
  $A(w(t))_{i,j} = \frac{1}{p_{k}(w)}$ when $i,j$ are both weights in layer $k$.

  In general, if   $
  q
  \in
  [
  1,
  \infty
  ]$
  and $A$ is an $n_K\times n_K$ matrix, then, with the convention that $1/\infty = 0$,
  \begin{equation}\label{2q}
    n^{-|1/2-1/q|} \|A\|_{2}
    \leq \|A\|_{q} \leq n^{|1/2 - 1/q|}\|A\|_{2}.
  \end{equation}
  This is well known; see for instance \cite[Section~5.6]{matanal}. Also, the Frobenius norm on $n_K\times n_K$ matrices satisfies
  \begin{equation}\label{2fro}
    \|A\|_{2} \leq \|A\|_{F} \leq n^{1/2}\|A\|_{2}.
  \end{equation}
  Combining \eqref{2q} and \eqref{2fro}, then,
  \begin{equation}\label{fqineq}
    n^{-1/2 -|1/2 - 1/q|} \|A\|_{F} \leq \|A\|_{q} \leq n^{|1/2 - 1/q|}\|A\|_{F}.
    \end{equation}
  It follows from \eqref{fqineq} and
  Proposition \ref{norm-equiv-impl-dual-equiv}  that for any linear functional 
  $\ell \in \mc{L}(\reals^{n_K\times n_K},\reals)$,
  \begin{equation}\label{follower}
    n^{-|1/2 - 1/q|}\|\ell\|_{F} \leq \|\ell\|_{q} \leq n^{1/2 +|1/2 - 1/q|}\|\ell\|_{F}.
  \end{equation}
  Inequality \eqref{follower}, together with the definition \eqref{weird-norm},  means that for any
  $\ell \in \mc{L}(W,\reals)$,
  \begin{equation}\label{cool-equation}
    n^{-|1/2 - 1/q|}
  \max_{1\leq i \leq K}\|\ell_i\|_{F}
  \leq
  \|\ell\|_{1,q}
  \leq  n^{1/2 +|1/2 - 1/q|}\max_{1\leq i \leq K}\|\ell_i\|_{F}.
  \end{equation}
  For any vector $u$ in $\reals^K$ we have,
  \begin{equation}\label{infty-thing}
    K^{-1/2}\|u\|_{2} \leq \|u\|_{\infty} \leq \|u\|_{2}.
  \end{equation}
  Combining \eqref{cool-equation} and \eqref{infty-thing} implies that for all $\ell \in \mc{L}(W,\reals)$,
  \begin{equation}\label{final-thing}
    K^{-1/2}  n^{-|1/2 - 1/q|}\|\ell\|_{2}
    \leq
    \|\ell\|_{1,q}
    \leq
    n^{1/2 + |1/2 - 1/q|}\|\ell\|_{2}.
  \end{equation}
  Let
  $k_3 =  K n^{1 + 2|1 - 2/q|}$.
  Then
  \begin{align*}
  \mathbb{E}\left[\left\|\delta(t)\right\|^{2}_{w(t)} \mid \mc{F}(t-1) \right]
  &= 
  \mathbb{E}\left[\left\|A(w(t))\delta(t)\right\|^{2}_{1,q} \mid \mc{F}(t-1) \right]
    \quad (\text{by } \eqref{repr})\\
  &\hspace{-5em}\leq
    n^{1 + |1 - 2/q|}
    \mathbb{E}\left[\left\|A(w(t))\delta(t)\right\|^{2}_{2} \mid \mc{F}(t-1) \right] 
   \quad (\text{by } \eqref{final-thing})
  \\ &\hspace{-5em}\leq
  \frac{1}{b}    n^{1 + |1 - 2/q|}
  \mathbb{E}\left[
    \left\|
    A(w(t))
    \left(\frac{\partial f}{\partial w}(w(t)) -\frac{\partial f_i}{\partial w}(w(t))\right)
    \right\|^{2}_{2}
       \, \middle|\, \mc{F}(t-1) \right] \\
  &\hspace{-5em}\leq
   \frac{1}{b}   k_3
  \mathbb{E}\left[
    \left\|
    A(w(t))\left(\frac{\partial f}{\partial w}(w(t)) -\frac{\partial f_i}{\partial w}(w(t))\right)
    \right\|^{2}_{1,q}
    \, \middle| \, \mc{F}(t-1)
    \right]   \quad (\text{by } \eqref{final-thing}.)
\end{align*}
In the third step, we used the fact that $b$ items in a mini-batch reduces the (Euclidean) variance by a factor of $b$ compared to using a single instance,
which we have represented with the random index $i\in\{1,2,\hdots,m\}$.

Applying the Equation \eqref{repr} once more, this yields
\begin{equation}\label{var-bd}
  \mathbb{E}\left[ \|\delta(t)\|^{2}_{w(t)} \,\middle|\, \mc{F}(t-1)\right ]
  \leq
        \frac{1}{b}k_{3}
        \mathbb{E}\left[
          \left\|\frac{\partial f}{\partial w}(w(t)) -\frac{\partial f_i}{\partial w}(w(t))\right\|^{2}_{w(t)}
          \, \middle|\, \mc{F}(t-1)
        \right].
  \end{equation}
Next, observe that for any pair $i,j$ in $\{1,2,\hdots,m\}$,
\begin{equation}\label{simple-sum}
\left\|\frac{\partial f_j}{\partial w}(w(t)) -\frac{\partial f_i}{\partial w}(w(t))\right\|_{w(t)}^{2}
\leq
2\left(\left\|\frac{\partial f_j}{\partial w}(w(t))\right\|_{w(t)}^{2}
+
\left\|\frac{\partial f_i}{\partial w}(w(t))\right\|_{w(t)}^{2}\right).
\end{equation}
Applying the chain rule to the function
$f_i$
as defined in
Equation \eqref{the-fi},
and using Inequalities \eqref{dq1prop}, \eqref{dx-dw-sigma},
we see that for all $w$, $i$, and $k$,
\begin{equation}\label{eqbd-pre}
  \begin{split}
  \left\|\frac{\partial f_i}{\partial w_k}(w)\right\|_{q}
  &\leq
  d_{q,1}\left\|\frac{\partial  y^{K}}{\partial w_k}(x_i;w)\right\|_{q} \\
  &\leq d_{q,1}r_{K-k}(\|w_{k+1}\|_{q},\hdots,\|w_{K}\|_{q})\|\sigma'\|_{\infty}c_q \\
  &=
\frac{d_{q,1}}{\sqrt{d_{q,2}}}\sqrt{d_{q,2}}c_{q}\|\sigma'\|_{\infty}r_{K-k}(\|w_{k+1}\|_{q},\hdots,\|w_{K}\|_{q}).
  \end{split}
  \end{equation}
Using the definition of $p_k$ from \eqref{p-definition}, then for any $w$, $i$, and $k$,
\begin{equation}\label{eqbd-almost}
  \left\|\frac{\partial f_i}{\partial w_k}(w)\right\|_{q}
  \leq \frac{d_{q,1}}{\sqrt{d_{q,2}}}p_{k}(w). 
\end{equation}
By examining the ratio $d_{q,1}/\sqrt{d_{q,2}}$ in the four cases presented in Table \ref{table:defs}, we see that
\begin{equation}\label{eqbd}
  \frac{d_{q,1}}{\sqrt{d_{q,2}}} = \sqrt{8}\times n^{\min\{1/2,1-1/q\}}.
  \end{equation}
Combining \eqref{eqbd-almost}, \eqref{eqbd} with the definition of the dual norm at $w$ presented at Equation \eqref{the-dual-norm},
\begin{equation}\label{grad-norm}
\left\|\frac{\partial f_i}{\partial w}(w(t))\right\|_{w(t)}
\leq
\sqrt{8}\times n^{\min\{1/2,1-1/q\}}.
\end{equation}
Using  \eqref{var-bd} and \eqref{simple-sum} together with \eqref{grad-norm},
\begin{align*}
\mathbb{E}
\left[
\|\delta(t)\|_{w(t)}^{2} \mid \mc{F}(t-1)
\right]
&\leq
  \frac{32}{b}k_3n^{\min\{1,2-2/q\}} \\
&=
  \frac{32}{b}Kn^{1+2|1-2/q|}\times n^{\min\{1,2-2/q\}} \\
&= \frac{32}{b} K n^{\max\{1+2/q, 4-4/q\}}.
  \end{align*}
  This confirms \eqref{the-var-cond}.
  \qed
\section*{Auxiliary results}
\begin{prop}\label{norm-equiv-impl-dual-equiv}
  Let $\|\cdot\|_{A}, \|\cdot\|_{B}$ be norms on $\reals^n$, such that
  for all  $u \in \reals^n$ the inequality $\|u\|_{A} \leq K \|u\|_{B}$ holds.
  Then for any  $\ell \in \mc{L}(\reals^n,\reals)$, it holds that
  $\|\ell\|_{B} \leq K \|\ell\|_{A}$.
\end{prop}
\begin{proof}
Given a norm $\|\cdot\|$ on $\reals^n$, the corresponding dual norm on $\mc{L}(\reals^n,\reals)$ can be expressed as $\|\ell\| = \sup_{u\neq 0}\frac{|\ell(u)|}{\|u\|}$. Using this formula, and the assumption on $\|\cdot\|_A, \|\cdot\|_B$, then,
\begin{align*}
\|\ell\|_{B} = \sup_{u\neq 0}\frac{|\ell(u)|}{\|u\|_{B}}  
\leq \sup_{u\neq 0}\frac{|\ell(u)|}{\|u\|_{A}}K 
= K \sup_{u\neq 0}\frac{|\ell(u)|}{\|u\|_{A}} = K\|\ell\|_{A}.
\end{align*}

  \end{proof}

\begin{prop}\label{simple-ex-lip}
  In the context of the neural network model defined by \eqref{nn-fn-comp} and the objective function $f$ of \eqref{opt-obj}, consider the following family of norms on $\mathbb{R}^{4}:$
  $$\|(u_1,u_2,u_3,u_4)\|_w = p_1(w)|u_1| + p_2(w)|u_2| + p_3|u_3| + p_4|u_4|,$$
  where $p_1,...,p_4$ are defined as follows:
  $$
  p_1(w)
  =
  \sqrt{\frac{1}{8}|w_3| + \frac{5}{128}|w_3|^2 + 1},
  $$
  $$
  p_2(w)
  =
  \sqrt{\frac{1}{8}|w_4| + \frac{5}{128}|w_4|^2 + 1 },
  $$
  $$p_3(w) = p_4(w) = \sqrt{\frac{5}{8}}$$
  Then a duality structure for the family of norms $\|\cdot\|_w$ is given by
  $$
  \rho_w(\ell_1,..,\ell_4)
  =
  \left(0,\hdots,\frac{1}{p_{i^*}(w)^2}\ell_{i^*},\hdots,0\right) \text{ where }
  i^*  = \argmax_{1\leq i \leq 4}
  \left\{ \frac{1}{p_i(w)}|\ell_i|\right\}.
  $$
  Additionally, with this family of norms and duality structure, Assumption \ref{f-asu}  holds with $L=1$.
\end{prop}
\begin{pf}
  We provide an abbreviated proof, and for more details see the detailed proof for the general case given in Lemma \ref{lip-prop-nn}.
  Firstly, that $\rho_w$ is indeed a duality structure for the specified family of norms follows from Proposition \ref{duality-product}.
  Next, we must verify the Lipschitz like condition of \eqref{f-asu}.
  The following inequalities, which following from basic calculus, will be essential:
  \begin{equation}\label{nn-dfdw1}
    \left|\tfrac{\partial^2 f}{\partial w_1^2}(w)\right| \leq  \frac{1}{8}|w_3| + \frac{5}{128}|w_3|^2
    \end{equation}
  $$\left|\tfrac{\partial^2 f}{\partial w_2^2}(w)\right| \leq  \frac{1}{8}|w_4| + \frac{5}{128}|w_4|^2$$
  $$\left|\tfrac{\partial^2 f}{\partial w_3^2}(w)\right| \leq \frac{5}{8}  $$
  $$\left|\tfrac{\partial^2 f}{\partial w_4^2}(w)\right| \leq \frac{5}{8}  $$
  Above, we have used the bounds $\|\sigma\|_{\infty} =1, \|\sigma'\|_{\infty}=\frac{1}{4}$, and $\|\sigma''\|_{\infty} \leq \frac{1}{4}$.
  Let $\eta = (\eta_1,\eta_2,\eta_3,\eta_4)$ be an arbitrary vector in $\mathbb{R}^4$.
  Let us  investigate the four possible cases for $\rho_w(\eta)$.
  First, consider the case of $\rho_w(\eta) = (\frac{1}{p_1(w)^2}\eta_1,0,0,0)$. Observe that, by \eqref{nn-dfdw1}, the function
  $\epsilon \mapsto f(w+\epsilon\rho_w(\eta))$ has a $K$-Lipschitz continuous gradient, where $K=\frac{|\eta_1|^2}{p_1(w)^4}\left(\frac{1}{8}|w_3| + \frac{5}{128}|w_3|^2\right) \leq \frac{|\eta_1|^2}{p_1(w)^2}$. Hence
  \begin{equation}
    \begin{split}
\left|
f(w + \epsilon \rho_{w}(\eta)) - f(w) - \epsilon \frac{\partial f}{\partial w}(w)\cdot\rho_{w}(\eta)
\right|
&\leq
\frac{1}{2}\epsilon^2\frac{|\eta_1|^2}{p_1(w)^2} \\
&= \frac{1}{2}\epsilon^2\|\eta\|_w^2
\end{split}
\end{equation}
  The case of $\rho_w(\eta) = (0,\frac{1}{p_2(w)^2}\eta_2,0,0)$ is handled similarly, due to the symmetry between $w_1$ and $w_2$.
  
  If $\rho_w(\eta) = (0,0\frac{1}{p_3(w)^2}\eta_3,0)$ then, similarly, the function
  $\epsilon \mapsto f(w+\epsilon\rho_w(\eta))$ has a $K$-Lipschitz gradient where
  $K=\frac{|\eta_3|^2}{p_3(w)^4}\frac{5}{8} = \frac{|\eta_3|^2}{p_3(w)^2}$, leading to the inequality
    \begin{equation*}
    \begin{split}
\left|
f(w + \epsilon \rho_{w}(\eta)) - f(w) - \epsilon \frac{\partial f}{\partial w}(w)\cdot\rho_{w}(\eta)
\right|
&\leq \frac{1}{2}\epsilon^2\|\eta\|_w^2
\end{split}
\end{equation*}
  Again, by symmetry the case $\rho_w(\eta) = (0,0,0,\frac{1}{p_4(w)^2}\eta_4)$ is handled similarly.  
  \end{pf}
\begin{prop}\label{nn-unbounded-var}
  In the context of the neural network model  defined by  \eqref{nn-fn-comp},
  consider the  objective function
$$
f(w_1,w_2,w_3,w_4)
=
\frac{1}{2}|y(w_1,w_2,w_3,w_4;1)|^{2} +  \frac{1}{2}|y(w_1,w_2,w_3,w_4;0) - 1|^2,
$$
which corresponds to training the network to map input $x=1$ to output $0$, and input $x=0$ to output $1$. 
Define
$\delta_1(w) = 2y(w;1)\tfrac{\partial y}{\partial w}(w;1)$
and
$\delta_2(w) = 2(y(w;0) - 1)\tfrac{\partial y}{\partial w}(w;0)$, so that
$$\tfrac{\partial f}{\partial w}(w) = \frac{1}{2}[\delta_1(w) + \delta_2(w)]$$
Consider the gradient estimator that computes
$\delta_1(w)$ or $\delta_2(w)$ with equal probability.
Then
$$\lim_{w\to\infty}\frac{1}{2}\sum\limits_{i=1}^2\|\delta_i(w) - \tfrac{\partial f}{\partial w}(w)\|^2 = +\infty$$
\end{prop}
\begin{proof}
  It suffices to show that $\|\delta_1(w) - \delta_2(w)\|^2$ is an increasing function of $\|w\|$.
  Focusing on the component of $\delta_1(w) - \delta_2(w)$ corresponding to $w_1$, we have
\begin{align*}
  \delta_{1,1}(w) &= 2y(w;1)\tfrac{\partial y}{\partial w_1}(w;1) \\
  &= 2y(w;1)\sigma'(w_3\sigma(w_1) + w_4\sigma(w_2))w_3\sigma'(w_1)
\end{align*}
and
\begin{align*}
\delta_{2,1}(w) &= 2(y(w;0)-1)\tfrac{\partial y}{\partial w_1}(w;0) \\
&=
2(y(w;0)-1)\sigma'(w_3\sigma(0) + w_4\sigma(0)w_3)\times 0  = 0.
\end{align*}
Hence
$$
|\delta_{1,1}(w) - \delta_{2,1}(w)|^2
=
|2y(w;1)\sigma'(w_3\sigma(w_1) + w_4\sigma(w_2))w_3\sigma'(w_1)|^2.
$$
Let $z$ be any number, and define the curve $w : [0,\infty) \to \reals$ as
$$w(\epsilon) = \left(1,1,\epsilon,\frac{1}{\sigma(1)}(z- \epsilon\sigma(1))\right).$$
Note that
$$
w_{3}\sigma(w_1) + w_{4}\sigma(w_2)
=
\epsilon\sigma(1) +
\frac{1}{\sigma(1)}[z- \epsilon\sigma(1)]\sigma(1)
=
y
$$
and therefore
$$
|\delta_{1,1}(w(\epsilon)) - \delta_{2,1}(w(\epsilon))|^2 
=
|2\sigma(z)\sigma'(z)\sigma'(1)|^2\epsilon^2.
$$
Clearly, the right hand side of this equation tends to $\infty$ as $\epsilon \rightarrow \infty$.
\end{proof}
\begin{lem}[Corollary 4.17 of \citep{chidume}]\label{lemm-dual-ineq}
Let $\|\cdot\|$ be a norm on $\reals^n$ that is $2$-uniformly convex with parameter $c$, and let $\rho$ be a  duality map for $\|\cdot\|$.
Then for any
$\ell_1,\ell_2$ in $\mc{L}(\reals^n,\reals)$,
\begin{equation}\label{lemma-6-eqn}
  \|\ell_1+\ell_2\|^{2}
  \geq \|\ell_1\|^{2} + 2 \ell_2\cdot\rho(\ell_1) + c\|\ell_2\|.
  \end{equation}
\end{lem}
\begin{proof}
  This is follows from Corollary 4.17 of \citep{chidume}, using $p=2$.
\end{proof}
\begin{prop}\label{mart-like-thm}
Let
$\tau$
be a stopping time with respect to a filtration
$\{\mc{F}_t \}_{t=0,1,\hdots}$.
Suppose there is a number
$c <\infty$
such that
$\tau \leq c$
with probability one.
Let
$x_1,x_2,\hdots$
be any sequence of random variables  such that each $x_t$ is $\mc{F}_{t}$-measurable and
$\mathbb{E}[\|x_t\|] < \infty$.
Then
\begin{equation}\label{stopping}
\mathbb{E}\left[\sum\limits_{t=1}^{\tau}x_t\right]
=
\mathbb{E}\left[
  \sum\limits_{t=1}^{\tau}
  \mathbb{E}\left[x_t \,\middle|\, \mc{F}_{t-1} \right]\right].
\end{equation}
\end{prop}

\begin{proof}
We argue that \eqref{stopping} is a consequence of the optional stopping theorem (Theorem 10.10 in \citep{williams1991probability}).
Define $S_0= 0$ and for $t\geq 1$, let
$S_t =\sum\limits_{i=1}^{t}\left(x_i - \mathbb{E}\left[x_i \,\middle|\, \mc{F}_{i-1}\right]\right)$.
Then
$S_0,S_1,\hdots$
is a martingale with respect to the filtration
$\{\mc{F}_{t}\}_{t=0,1,\hdots}$,
and the optional stopping theorem implies
$\mathbb{E}[S_{\tau}] = \mathbb{E}[S_0]$.
But
$\mathbb{E}[S_0] = 0$, and therefore
$\mathbb{E}[S_{\tau}] = 0$,
which is  equivalent to \eqref{stopping}.
\end{proof}

\begin{lem}\label{nn-bounds}
 Let $h:\reals^n\times \reals^{n\times n} \to \mathbb{R}^{n}$ be the function defined by Equation \eqref{hdef}.
 Let $\mathbb{R}^{n}$ have the norm $\|\cdot\|_{q}$ for $1 \leq q \leq \infty$ 
and equip the matrices $\mathbb{R}^{n\times n}$ with the corresponding induced norm.
  Then the following inequalities hold:
$$\left\|\frac{\partial h}{\partial y}(y,w)\right\|_q \leq \|\sigma'\|_{\infty}\|w\|_q,\quad\quad\quad\quad
\left\|\frac{\partial h}{\partial w}(y,w)\right\|_q \leq \|\sigma'\|_{\infty}\|y\|_q,$$
$$\left\|\frac{\partial^{2}h}{\partial y^{2}}(y,w)\right\|_q 
\leq \|\sigma''\|_{\infty}\|w\|_q^{2},\quad\quad\quad\quad
\left\|\frac{\partial^{2}h}{\partial w^{2}}(y,w)\right\|_q \leq \|\sigma''\|_{\infty}\|y\|_q^{2}.$$
\end{lem}
\begin{proof}
  Define the pointwise product of two vectors in $\reals^n$ as $u \odot v = (u_1v_1,\hdots,u_nv_n)$.
  We will rely on the following two properties that are shared by the norms
  $\|\cdot\|_{q}$
  for
  $1\leq q \leq \infty$. Firstly, for all vectors $u,v$,
  \begin{equation}\label{nn-hadamard-asu}
  \|u\odot v\|_q \leq \|u\|_q\|v\|_q.
  \end{equation}
Secondly,   for any $n\times n$  diagonal matrix $D$,
$
    \|D\|_q = \max_{1\leq i \leq n} \,|D_{i,i}|.
$
    Recall that the component functions of $h$ are
    $h_{i}(y,w) = \sigma\left(\sum\limits_{i=k}^{n}w_{i,k}y_{k}\right)$. Then for $1\leq i,j \leq n$,
$$
\frac{\partial h_{i}}{\partial y_{j}}(y,w) =
\sigma'\left(\sum\limits_{k=1}^{n}w_{i,k}y_k\right)w_{i,j}.
$$
Set $D(x,w)$ to be the $n\times n$ diagonal matrix
\begin{equation}\label{nn-d-def}
D(y,w)_{i,i} = \sigma'\left(\sum\limits_{k=1}^{n}w_{i,k}y_k\right).
\end{equation}
Then
$\frac{\partial h}{\partial y}(y,w) = D(y,w)w$.
Hence
\begin{align*}
  \left\|\frac{\partial h}{\partial y}(y,w)\right\|_q
  &= \|D(y,w)w\|_q \\
  &\leq \|D(y,w)\|_q\|w\|_q  
  \\
  &= \sup_{1\leq i \leq n}|D(y,w)_{i,i}| \|w\|_q 
  \\
  &\leq \|\sigma'\|_{\infty}\|w\|_q. 
\end{align*}
Observe that for $1\leq i,j,l \leq n$,
$$\frac{\partial h_{i}}{\partial w_{j,l}}(y,w) = 
\begin{cases}
\sigma'\left(\sum\limits_{k=1}^{n}w_{i,k}y_k\right)y_{l} &\text{ if } j=i, \\
0 &\text{ else. } 
\end{cases}$$
Let
$\Delta$ be an $n\times n$ matrix such that $\|\Delta\|_q=1$.
Then
$\tfrac{\partial h}{\partial w}(y,w)\cdot\Delta$
is a vector in $\reals^n$ with $i$th component 
\begin{align*}
  \left(\frac{\partial h}{\partial w}(y,w)\cdot\Delta\right)_i &=
                                                                 \sum\limits_{j=1}^{n}\sum\limits_{l=1}^{n}
\frac{\partial h_{i}}{\partial w_{j,l}}(y,w)\Delta_{j,l} \\
  &=
  \sum\limits_{l=1}^{n}\frac{\partial h_{i}}{\partial w_{i,l}}(y,w)\Delta_{i,l} \\
  &=
    \sum\limits_{l=1}^{n}
    \sigma'\left(\sum\limits_{k=1}^{n}w_{i,k}y_k\right)y_{\ell}\Delta_{i,l} \\
  &=
    \sigma'\left(\sum\limits_{k=1}^{n}w_{i,k}y_k\right)
    \sum\limits_{l=1}^{n}\Delta_{i,l}y_{\ell} \\
  &=
  (D(x,w)\Delta y)_{i}
\end{align*}
where $D$ is as in \eqref{nn-d-def}. Hence
\begin{align*}
\left\| \frac{\partial h}{\partial w}(y,w)\cdot\Delta \right\|_q &=
\|D(y,w)\Delta y\|_q \\
                                                                 &\leq \|D(y,w)\|_q\|\Delta\|_q\|y\|_q 
  \\
&= \sup_{1\leq i\leq n}\,|D(y,w)_{i,i}|\|\Delta\|_q\|y\|_q 
\end{align*}
and therefore
$
\left\|\frac{\partial h}{\partial w}(y,w)\right\|_q \leq \|\sigma'\|_{\infty}\|y\|_{q}.
$
Observe that
$$\frac{\partial^{2} h_{i}}{\partial y_{j}\partial y_{l}}(y,w) = \sigma''\left(\sum\limits_{k=1}^{n}w_{i,k}y_k\right)w_{i,j}w_{i,l}.$$
That means
$\frac{\partial^{2} h }{\partial y_j\partial y_{\ell}}(y,w)\cdot(u,v)$
is a vector with components
\begin{align*}
\left(  \frac{\partial^{2}h}{\partial y_j\partial y_{\ell}}(x,w)\cdot(u,v)\right)_i
  &=
    \sum\limits_{j=1}^{n}\sum\limits_{\ell=1}^{n} \sigma''\left(
    \sum\limits_{k=1}^{n}w_{i,k}y_k
    \right)w_{i,j}w_{i,l}u_jv_{\ell} \\&
  =
  \sigma''\left(\sum\limits_{k=1}^{n}w_{i,k}y_k\right)
  \left(\sum\limits_{j=1}^{n}w_{i,j}u_j\right)
  \left(\sum\limits_{\ell=1}^{n}w_{i,l}v_{\ell}\right) \\
  &=
  E(y,w)_{i,i}(Wu)_i(Wv)_i
\end{align*}
where $E$ is the diagonal matrix
\begin{equation}\label{nn-e-def}
  E(y,w)_{i,i} = \sigma''\left(\sum\limits_{k=1}^{n}w_{i,k}y_k\right).
  \end{equation}
Using the notation $\odot$ for the entry-wise  product of vectors, then
$$
\frac{\partial^{2}h}{\partial y^{2}}(x,w)\cdot(u,v)
=
E(x,w)\cdot( (wu) \odot (wv) ).
$$
Then
\begin{align*}
\left\|
\frac{\partial^{2} h }{\partial y^{2} }(y,w)\cdot\left(u,v\right)
\right\|_q &=
\|E(y,w)\cdot( wu \odot wv )\|_q \\
&\leq \|E(y,w)\|_q\|(wu)\odot (wv)\|_q \\
&\leq\|E(y,w)\|_q\|wu\|_q\|wv\|_q 
\\
&\leq \|E(y,w)\|_q\|w\|_q^{2}\|u\|_q\|v\|_q 
\end{align*}
Hence
$
\left\|\frac{\partial^{2}h}{\partial y^{2}}(y,w)\right\|_q
\leq \|\sigma''\|_{\infty}\|w\|_{q}^{2}$.
Observe that
$$
\frac{\partial^{2} h_{i}}{\partial w_{j,l}\partial w_{k,m}}(x,w)
= 
\begin{cases}
\sigma''\left(\sum\limits_{k=1}^{n}w_{i,k}y_k\right)y_{l}y_{m} 
&\text{ if } j=i \text{ and } k = i,\\
0 &\text{ else. } 
\end{cases}
$$
Hence for matrices $u,v$,
$\frac{\partial^{2} h}{\partial w^{2}}(y,w)\cdot(u,v)$
is a vector with entries
\begin{align*}
\left(\frac{\partial^{2} h_i}{\partial w^{2}}(y,w)\cdot\left(u,v\right)\right)_i
 &=
\sum\limits_{j=1}^{n}\sum\limits_{l=1}^{n}\sum\limits_{k=1}^{n}\sum\limits_{m=1}^{n}
\frac{\partial^{2} h_i}
     {\partial w_{j,l}\partial w_{k,m}}
     (y,w)u_{j,l}v_{k,m} \\
&=
\sum\limits_{l=1}^{n}\sum\limits_{m=1}^{n}
\frac{\partial^{2} h_i}{\partial w_{i,l}\partial w_{i,m}}
(y,w)u_{i,l}v_{i,m} \\
&=
\sum\limits_{l=1}^n\sum\limits_{m=1}^n
\sigma''\left(\sum\limits_{k=1}^{n}w_{i,k}x_k\right)
y_{\ell}y_{m}u_{i,l}v_{i,m} \\
&=
\sigma''\left(\sum\limits_{k=1}^{n}w_{i,k}y_k\right)
\left(
  \sum\limits_{l}u_{i,l}y_{\ell}
\right)
\left(
\sum\limits_{m}v_{i,m}x_m
\right) \\
&=
E(y,w)_{i,i}(uy)_i(vy)_i
\end{align*}
where $E$ is as in \eqref{nn-e-def}. Hence
$$
\frac{\partial^{2} h}{\partial w^{2}}(y,w)\cdot(u,v)
= E(y,w)\cdot( (uy) \odot (vy) )
$$
which means
\begin{align*}
  \left\|
  \frac{\partial^{2} h}{\partial w^{2}}(y,w)\cdot(u,v)\right\|_q &\leq
\|E(y,w)\|_q\|uy\|_q\|vy\|_q \\
&\leq \|E(y,w)\|_{q}\|u\|_{q}\|v\|_{q}\|y\|_{q}^{2}.
\end{align*}
Therefore  $\left\|\frac{\partial^{2} h}{\partial w^{2}}(y,w)\cdot(u,v)\right\|_q \leq
\|\sigma''\|_{\infty}\|y\|_q^{2}.$
\end{proof}
\begin{lem}\label{euclidean-error}
  Let
  $z\in\mathbb{R}^{n}$
  and define the function $J:\reals^n \to \reals$ as 
  $$
  J(y)
  =
  \sum\limits_{i=1}^{n}(y_i - z_i)^{2}
  $$
  Then  for all $y,z$  in $\reals^n$ such that
  $\|y-z\|_{\infty} \leq 2$, and $1\leq q \leq \infty$,
  \begin{equation}\label{dj-norm-w-z}
    \left\|\frac{\partial J}{\partial y}(y)\right\|_{q}  \leq
    \begin{cases}
    4 &\text{ if } q = 1, \\ 
    4n^{(q-1)/q} &\text{ if } 1 < q < \infty, \\
    4n  & \text{ if } q= \infty,
    \end{cases}
  \end{equation}
and
$$
\left\|\frac{\partial^{2} J}{\partial y^{2}}(y)\right\|_{q}  \leq
\begin{cases}
  2 &\text{ if } 1 \leq q \leq 2, \\
  2n^{(q-2)/q} &\text{ if } 2 < q < \infty, \\
  2n &\text{ if } q = \infty.
  \end{cases}
$$
\end{lem}
\begin{proof}
By direct calculation, the components of the derivative of $J$ are
$\frac{\partial J}{\partial y_i}(y) = 2(y_i-z_i)$,
and therefore
$$\left\|\frac{\partial J}{\partial y_i}(y)\right\|_q = 2\|y_i-z_i\|_{q^*}.$$
where $\|\cdot\|_{q^*}$ represents the norm dual to $\|\cdot\|_q$.
When $q=1$, the dual norm is $\|\cdot\|_{\infty}$, for $1 \leq q < \infty$, the dual of norm $\|\cdot\|_q$ is $\|\cdot\|_{\frac{q}{q-1}}$ and finally
the dual of the norm $\|\cdot\|_{\infty}$ is $\|\cdot\|_1$.
This yields
  \begin{equation}\label{dj-norm}
    \left\|\frac{\partial J}{\partial x}(y)\right\|_{q}  =
    \begin{cases}
    2\|y-z\|_{\infty} &\text{ if } q= 1 \\ 
    2\|y-z\|_{q/(q-1)} &\text{ if } 1 < q < \infty, \\
    2\|y-z\|_1 & \text{ if } q= \infty
    \end{cases}
  \end{equation}
Equation (\ref{dj-norm-w-z}) follows by combining \eqref{dj-norm} with our assumption that $\|y-z\|_{\infty} \leq 2$.
  
For the second derivative, the components are
  $$\frac{\partial^{2}J}{\partial y_i \partial y_j} = \begin{cases}
0 &\text{ if } i\neq j, \\
2 &\text {if } i = j.
\end{cases}$$
Then for any vectors $u,v$,
$$
\frac{\partial J^{2}}{\partial y^2}(y)\cdot(u,v) = 2\sum\limits_{i=1}^{n}u_iv_i.
$$
Therefore
\begin{align*}
  \left\| \frac{\partial J^{2}}{\partial y^2}(y) \right\|_q
  &=2 \sup_{\|u\|_q=\|v\|_q=1}\sum\limits_{i=1}^{n}u_iv_i\\
  &= 2\sup_{\|u\|_q=1}\|u\|_{q^*}
\end{align*}
To bound the final term in the above equation, we consider  four cases.
The first case is that  $q=1$. In this situation, $\|u\|_{q^*} = \|u\|_{\infty} \leq 1$. The second case is that  $1<q<2$. Then $q/(q-1) > q$, and hence $\|u\|_{q^*} \leq \|u\|_q = 1$.
The third case is when  $2 \leq q < \infty$. Here, $q/(q-1) \leq q$, and we appeal to the following  inequality: If $1 \leq r < q$, then
$\|\cdot\|_{r} \leq n^{\frac{1}{r} - \frac{1}{q}}\|\cdot\|_{q}$.
Applying this inequality with $r=q/(q-1)$, we obtain
$\|u\|_{q^*} \leq n^{(q-2)/q}\|u\|_q = n^{(q-2)/q}.$
 Finally, if $q=\infty$ we have $\|u\|_{q^*} = \|u\|_{1} \leq n$.
\end{proof}

\end{document}

%% file: mycmds.tex
\usepackage{pgf, tikz}
\usetikzlibrary{arrows, automata,shapes,snakes}
\newcommand{\reals}{\mathbb{R}}
\newcommand{\rn}{\reals^{n}}
\newcommand{\setin}{\subseteq}

\newcommand{\sgn}{\operatorname{sgn}}

\newcommand{\mc}[1]{\mathcal{#1}}

\newcommand{\rank}{\operatorname{rank}}

\newcommand{\tr}{\operatorname{tr}}
  
 \DeclareMathOperator*{\argmax}{arg\,max}
\renewcommand{\l}{\left}
\renewcommand{\r}{\right}

\let\oldnl\nl
\newcommand{\nonl}{\renewcommand{\nl}{\let\nl\oldnl}}

%% file: paper.bbl
\begin{thebibliography}{55}
\providecommand{\natexlab}[1]{#1}
\providecommand{\url}[1]{\texttt{#1}}
\expandafter\ifx\csname urlstyle\endcsname\relax
  \providecommand{\doi}[1]{doi: #1}\else
  \providecommand{\doi}{doi: \begingroup \urlstyle{rm}\Url}\fi

\bibitem[Abraham et~al.(2012)Abraham, Marsden, and Ratiu]{abraham2012manifolds}
Ralph Abraham, Jerrold~E Marsden, and Tudor Ratiu.
\newblock \emph{Manifolds, tensor analysis, and applications}, volume~75.
\newblock Springer Science \& Business Media, 2012.

\bibitem[Absil et~al.(2009)Absil, Mahony, and Sepulchre]{absil2009optimization}
P-A Absil, Robert Mahony, and Rodolphe Sepulchre.
\newblock \emph{Optimization algorithms on matrix manifolds}.
\newblock Princeton University Press, 2009.

\bibitem[Allen-Zhu and Hazan(2016)]{zeyuan-svrg}
Zeyuan Allen-Zhu and Elad Hazan.
\newblock Variance reduction for faster non-convex optimization.
\newblock In \emph{Proceedings of the 33rd International Conference on Machine
  Learning}, pages 699--707, 2016.

\bibitem[Amari(1998)]{amari1998natural}
Shun-Ichi Amari.
\newblock Natural gradient works efficiently in learning.
\newblock \emph{Neural computation}, 10\penalty0 (2):\penalty0 251--276, 1998.

\bibitem[Bartlett et~al.(2017)Bartlett, Foster, and
  Telgarsky]{bartlettspectrally}
Peter~L Bartlett, Dylan~J Foster, and Matus~J Telgarsky.
\newblock Spectrally-normalized margin bounds for neural networks.
\newblock In I.~Guyon, U.~V. Luxburg, S.~Bengio, H.~Wallach, R.~Fergus,
  S.~Vishwanathan, and R.~Garnett, editors, \emph{Advances in Neural
  Information Processing Systems 30}, pages 6240--6249. 2017.

\bibitem[Bauschke et~al.(2017)Bauschke, Bolte, and Teboulle]{beyondlip}
Heinz~H. Bauschke, Jérôme Bolte, and Marc Teboulle.
\newblock A descent lemma beyond lipschitz gradient continuity: First-order
  methods revisited and applications.
\newblock \emph{Mathematics of Operations Research}, 42\penalty0 (2):\penalty0
  330--348, 2017.

\bibitem[Bauschke et~al.(2019)Bauschke, Bolte, Chen, Teboulle, and
  Wang]{Bauschke2019}
Heinz~H. Bauschke, J{\'{e}}r{\^{o}}me Bolte, Jiawei Chen, Marc Teboulle, and
  Xianfu Wang.
\newblock On linear convergence of non-euclidean gradient methods without
  strong convexity and lipschitz gradient continuity.
\newblock \emph{Journal of Optimization Theory and Applications}, 182\penalty0
  (3):\penalty0 1068--1087, April 2019.
\newblock \doi{10.1007/s10957-019-01516-9}.
\newblock URL \url{https://doi.org/10.1007/s10957-019-01516-9}.

\bibitem[Birgin et~al.(2017)Birgin, Gardenghi, Mart{\'i}nez, Santos, and
  Toint]{Birgin2017}
E.~G. Birgin, J.~L. Gardenghi, J.~M. Mart{\'i}nez, S.~A. Santos, and Ph.~L.
  Toint.
\newblock Worst-case evaluation complexity for unconstrained nonlinear
  optimization using high-order regularized models.
\newblock \emph{Mathematical Programming}, 163\penalty0 (1):\penalty0 359--368,
  May 2017.
\newblock ISSN 1436-4646.

\bibitem[Bolte et~al.(2018)Bolte, Sabach, Teboulle, and Vaisbourd]{bolte}
J\'{e}r\^{o}me Bolte, Shoham Sabach, Marc Teboulle, and Yakov Vaisbourd.
\newblock First order methods beyond convexity and lipschitz gradient
  continuity with applications to quadratic inverse problems.
\newblock \emph{SIAM Journal on Optimization}, 28\penalty0 (3):\penalty0
  2131--2151, 2018.

\bibitem[Boumal et~al.(2018)Boumal, Absil, and Cartis]{boumal2016globalrates}
N.~Boumal, P.-A. Absil, and C.~Cartis.
\newblock Global rates of convergence for nonconvex optimization on manifolds.
\newblock \emph{IMA Journal of Numerical Analysis}, To appear, 2018.

\bibitem[Cartis et~al.(2011)Cartis, Gould, and Toint]{cartis2011adaptive}
Coralia Cartis, Nicholas~IM Gould, and Philippe~L Toint.
\newblock Adaptive cubic regularisation methods for unconstrained optimization.
  part i: motivation, convergence and numerical results.
\newblock \emph{Mathematical Programming}, 127\penalty0 (2):\penalty0 245--295,
  2011.

\bibitem[Chidume(2009)]{chidume}
Charles Chidume.
\newblock \emph{Geometric properties of Banach spaces and nonlinear
  iterations}, volume 1965.
\newblock Springer, 2009.

\bibitem[Curtis et~al.(2017)Curtis, Robinson, and Samadi]{curtis2017trust}
Frank~E Curtis, Daniel~P Robinson, and Mohammadreza Samadi.
\newblock A trust region algorithm with a worst-case iteration complexity of
  $$o(\epsilon^{-3/2}))$$ for nonconvex optimization.
\newblock \emph{Mathematical Programming}, 162\penalty0 (1-2):\penalty0 1--32,
  2017.

\bibitem[Davidon(1959)]{davidon}
W.~C. Davidon.
\newblock Variable metric method for minimization.
\newblock \emph{AEC Research and Development Report ANL-5990 (Rev. TID-4500,
  14th Ed.)}, 11 1959.
\newblock \doi{10.2172/4222000}.

\bibitem[Davidon(1991)]{davidon1991}
William~C Davidon.
\newblock Variable metric method for minimization.
\newblock \emph{SIAM Journal on Optimization}, 1\penalty0 (1):\penalty0 1--17,
  1991.

\bibitem[Davis et~al.(2018)Davis, Drusvyatskiy, and
  MacPhee]{davis2018stochastic}
Damek Davis, Dmitriy Drusvyatskiy, and Kellie~J MacPhee.
\newblock Stochastic model-based minimization under high-order growth.
\newblock \emph{arXiv preprint arXiv:1807.00255}, 2018.

\bibitem[Deimling(1985)]{deimling1985nonlinear}
K.~Deimling.
\newblock \emph{Nonlinear functional analysis}.
\newblock Springer-Verlag, 1985.

\bibitem[Duchi et~al.(2011)Duchi, Hazan, and Singer]{duchi2011adaptive}
John Duchi, Elad Hazan, and Yoram Singer.
\newblock Adaptive subgradient methods for online learning and stochastic
  optimization.
\newblock \emph{Journal of Machine Learning Research}, 12\penalty0
  (Jul):\penalty0 2121--2159, 2011.

\bibitem[Farabet et~al.(2013)Farabet, Couprie, Najman, and LeCun]{farabet2013}
Clement Farabet, Camille Couprie, Laurent Najman, and Yann LeCun.
\newblock Learning hierarchical features for scene labeling.
\newblock \emph{IEEE transactions on pattern analysis and machine
  intelligence}, 35\penalty0 (8):\penalty0 1915--1929, 2013.

\bibitem[Fehrman et~al.(2020)Fehrman, Gess, and Jentzen]{sgdminimum}
Benjamin Fehrman, Benjamin Gess, and Arnulf Jentzen.
\newblock Convergence rates for the stochastic gradient descent method for
  non-convex objective functions.
\newblock \emph{Journal of Machine Learning Research}, 21\penalty0
  (136):\penalty0 1--48, 2020.
\newblock URL \url{http://jmlr.org/papers/v21/19-636.html}.

\bibitem[Ghadimi and Lan(2013)]{ghadimi-lan}
Saeed Ghadimi and Guanghui Lan.
\newblock Stochastic first- and zeroth-order methods for nonconvex stochastic
  programming.
\newblock \emph{SIAM Journal on Optimization}, 23\penalty0 (4):\penalty0
  2341--2368, 2013.

\bibitem[Hinton and Salakhutdinov(2006)]{hintonscience}
G.~E. Hinton and R.~R. Salakhutdinov.
\newblock Reducing the dimensionality of data with neural networks.
\newblock \emph{Science}, 313\penalty0 (5786):\penalty0 504--507, 2006.
\newblock ISSN 0036-8075.

\bibitem[Horn and Johnson(1986)]{matanal}
Roger~A. Horn and Charles~R. Johnson, editors.
\newblock \emph{Matrix Analysis}.
\newblock Cambridge University Press, New York, NY, USA, 1986.

\bibitem[Jin et~al.(2022)Jin, Xing, and He]{jin2022on}
Ruinan Jin, Yu~Xing, and Xingkang He.
\newblock On the convergence of m{SGD} and adagrad for stochastic optimization.
\newblock In \emph{International Conference on Learning Representations}, 2022.
\newblock URL \url{https://openreview.net/forum?id=g5tANwND04i}.

\bibitem[Johnson and Zhang(2013)]{svrg}
Rie Johnson and Tong Zhang.
\newblock Accelerating stochastic gradient descent using predictive variance
  reduction.
\newblock In \emph{Advances in Neural Information Processing Systems 26}, pages
  315--323. 2013.

\bibitem[Kavis et~al.(2022)Kavis, Levy, and Cevher]{kavis2022high}
Ali Kavis, Kfir~Yehuda Levy, and Volkan Cevher.
\newblock High probability bounds for a class of nonconvex algorithms with
  adagrad stepsize.
\newblock In \emph{International Conference on Learning Representations}, 2022.
\newblock URL \url{https://openreview.net/forum?id=dSw0QtRMJkO}.

\bibitem[Kingma and Ba(2014)]{kingma2014adam}
Diederik~P. Kingma and Jimmy Ba.
\newblock Adam: A method for stochastic optimization.
\newblock In \emph{Proceedings of the 3rd International Conference on Learning
  Representations (ICLR)}, 2014.

\bibitem[Krizhevsky(2009)]{Krizhevsky09}
Alex Krizhevsky.
\newblock Learning multiple layers of features from tiny images.
\newblock Technical report, University of Toronto, 2009.

\bibitem[Krizhevsky et~al.(2012)Krizhevsky, Sutskever, and Hinton]{imagenet}
Alex Krizhevsky, Ilya Sutskever, and Geoff Hinton.
\newblock Imagenet classification with deep convolutional neural networks.
\newblock In \emph{Advances in Neural Information Processing Systems 25}, pages
  1106--1114, 2012.

\bibitem[Kurita(1993)]{kurita1993}
Takio Kurita.
\newblock Iterative weighted least squares algorithms for neural networks
  classifiers.
\newblock In Shuji Doshita, Koichi Furukawa, Klaus~P. Jantke, and Toyaki
  Nishida, editors, \emph{Algorithmic Learning Theory: Third Workshop, ALT '92
  Tokyo, Japan, October 20--22, 1992 Proceedings}, pages 75--86, Berlin,
  Heidelberg, 1993. Springer Berlin Heidelberg.

\bibitem[LeCun(1998)]{lecun1998mnist}
Yann LeCun.
\newblock The mnist database of handwritten digits.
\newblock \emph{http://yann. lecun. com/exdb/mnist/}, 1998.

\bibitem[Lee et~al.(2009)Lee, Grosse, Ranganath, and Ng]{lee2009}
Honglak Lee, Roger Grosse, Rajesh Ranganath, and Andrew~Y Ng.
\newblock Convolutional deep belief networks for scalable unsupervised learning
  of hierarchical representations.
\newblock In \emph{Proceedings of the 26th annual international conference on
  machine learning}, pages 609--616. ACM, 2009.

\bibitem[Li and Orabona(2018)]{li2018convergence}
Xiaoyu Li and Francesco Orabona.
\newblock On the convergence of stochastic gradient descent with adaptive
  stepsizes.
\newblock \emph{arXiv preprint arXiv:1805.08114}, 2018.

\bibitem[Lian et~al.(2017)Lian, Zhang, Zhang, Hsieh, Zhang, and
  Liu]{decentralized}
Xiangru Lian, Ce~Zhang, Huan Zhang, Cho-Jui Hsieh, Wei Zhang, and Ji~Liu.
\newblock Can decentralized algorithms outperform centralized algorithms? a
  case study for decentralized parallel stochastic gradient descent.
\newblock In \emph{Advances in Neural Information Processing Systems 30}, pages
  5330--5340. 2017.

\bibitem[Liang et~al.(2019)Liang, Poggio, Rakhlin, and Stokes]{liang2019fisher}
Tengyuan Liang, Tomaso Poggio, Alexander Rakhlin, and James Stokes.
\newblock Fisher-rao metric, geometry, and complexity of neural networks.
\newblock In \emph{The 22nd International Conference on Artificial Intelligence
  and Statistics}, pages 888--896, 2019.

\bibitem[Lieb et~al.(1994)Lieb, Ball, and Carlen]{Lieb1994}
E.H. Lieb, Keith Ball, and E.A. Carlen.
\newblock Sharp uniform convexity and smoothness inequalities for trace norms.
\newblock \emph{Inventiones mathematicae}, 115\penalty0 (3):\penalty0 463--482,
  1994.

\bibitem[Lu et~al.(2018)Lu, Freund, and Nesterov]{relativelysmooth}
H.~Lu, R.~Freund, and Y.~Nesterov.
\newblock Relatively smooth convex optimization by first-order methods, and
  applications.
\newblock \emph{SIAM Journal on Optimization}, 28\penalty0 (1):\penalty0
  333--354, 2018.

\bibitem[Mai and Johansson(2021)]{mai21a}
Vien~V. Mai and Mikael Johansson.
\newblock Stability and convergence of stochastic gradient clipping: Beyond
  lipschitz continuity and smoothness.
\newblock In Marina Meila and Tong Zhang, editors, \emph{Proceedings of the
  38th International Conference on Machine Learning}, volume 139 of
  \emph{Proceedings of Machine Learning Research}, pages 7325--7335. PMLR,
  18--24 Jul 2021.
\newblock URL \url{https://proceedings.mlr.press/v139/mai21a.html}.

\bibitem[Moulines and Bach(2011)]{bachmoulines11}
Eric Moulines and Francis~R. Bach.
\newblock Non-asymptotic analysis of stochastic approximation algorithms for
  machine learning.
\newblock In J.~Shawe-Taylor, R.~S. Zemel, P.~L. Bartlett, F.~Pereira, and
  K.~Q. Weinberger, editors, \emph{Advances in Neural Information Processing
  Systems 24}, pages 451--459. Curran Associates, Inc., 2011.

\bibitem[Nesterov(2013)]{nesterov2013introductory}
Yurii Nesterov.
\newblock \emph{Introductory lectures on convex optimization: A basic course},
  volume~87.
\newblock Springer Science \& Business Media, 2013.

\bibitem[Netzer et~al.(2011)Netzer, Wang, Coates, Bissacco, Wu, and
  Ng]{netzer2011reading}
Yuval Netzer, Tao Wang, Adam Coates, Alessandro Bissacco, Bo~Wu, and Andrew~Y
  Ng.
\newblock Reading digits in natural images with unsupervised feature learning.
\newblock In \emph{NIPS Workshop on Deep Learning and Unsupervised Feature
  Learning}, 2011.

\bibitem[Neyshabur et~al.(2015)Neyshabur, Tomioka, and Srebro]{Neyshabur15}
Behnam Neyshabur, Ryota Tomioka, and Nathan Srebro.
\newblock Norm-based capacity control in neural networks.
\newblock In Peter Grünwald, Elad Hazan, and Satyen Kale, editors,
  \emph{Proceedings of The 28th Conference on Learning Theory}, volume~40 of
  \emph{Proceedings of Machine Learning Research}, pages 1376--1401, Paris,
  France, 03--06 Jul 2015. PMLR.

\bibitem[Nguyen(2018)]{sgd-and-hogwild}
Lam et.~al Nguyen.
\newblock Sgd and hogwild: Convergence without the bounded gradients
  assumption.
\newblock In \emph{ICML}, 2018.

\bibitem[Ollivier(2015)]{yann}
Yann Ollivier.
\newblock Riemannian metrics for neural networks i: feedforward networks.
\newblock \emph{Information and Inference: A Journal of the IMA}, 4\penalty0
  (2):\penalty0 108, 2015.

\bibitem[Reddi et~al.(2016)Reddi, Hefny, Sra, Poczos, and
  Smola]{reddi2016stochastic}
S.~Reddi, A.~Hefny, S.~Sra, B.~Poczos, and A.~Smola.
\newblock Stochastic variance reduction for nonconvex optimization.
\newblock In \emph{International conference on machine learning}, 2016.

\bibitem[Salakhutdinov and Hinton(2007)]{nlnca}
Ruslan Salakhutdinov and Geoffrey~E. Hinton.
\newblock Learning a nonlinear embedding by preserving class neighbourhood
  structure.
\newblock In \emph{International Conference on Artificial Intelligence and
  Statistics}, pages 412--419, 2007.

\bibitem[Schaul et~al.(2013)Schaul, Zhang, and LeCun]{pesky}
Tom Schaul, Sixin Zhang, and Yann LeCun.
\newblock No more pesky learning rates.
\newblock In \emph{Proceedings of the 30th International Conference on
  International Conference on Machine Learning - Volume 28}, pages
  III--343--III--351, 2013.

\bibitem[Srivastava and Salakhutdinov(2014)]{srivastava2014multimodal}
Nitish Srivastava and Ruslan Salakhutdinov.
\newblock Multimodal learning with deep boltzmann machines.
\newblock \emph{The Journal of Machine Learning Research}, 15\penalty0
  (1):\penalty0 2949--2980, 2014.

\bibitem[Tieleman and Hinton(2012)]{Tieleman2012}
T.~Tieleman and G.~Hinton.
\newblock {Lecture 6.5---RmsProp: Divide the gradient by a running average of
  its recent magnitude}.
\newblock COURSERA: Neural Networks for Machine Learning, 2012.

\bibitem[Ward et~al.(2019)Ward, Wu, and Bottou]{ward19a}
Rachel Ward, Xiaoxia Wu, and Leon Bottou.
\newblock {A}da{G}rad stepsizes: Sharp convergence over nonconvex landscapes.
\newblock In Kamalika Chaudhuri and Ruslan Salakhutdinov, editors,
  \emph{Proceedings of the 36th International Conference on Machine Learning},
  volume~97 of \emph{Proceedings of Machine Learning Research}, pages
  6677--6686, Long Beach, California, USA, 09--15 Jun 2019. PMLR.

\bibitem[Williams(1991)]{williams1991probability}
D.~Williams.
\newblock \emph{Probability with Martingales}.
\newblock Cambridge University Press, 1991.

\bibitem[Xiao et~al.(2017)Xiao, Rasul, and Vollgraf]{fashionmnist}
Han Xiao, Kashif Rasul, and Roland Vollgraf.
\newblock Fashion-mnist: a novel image dataset for benchmarking machine
  learning algorithms, 2017.

\bibitem[Zhang et~al.(2016)Zhang, J.~Reddi, and Sra]{riemannian-svrg}
Hongyi Zhang, Sashank J.~Reddi, and Suvrit Sra.
\newblock Riemannian svrg: Fast stochastic optimization on riemannian
  manifolds.
\newblock In \emph{Advances in Neural Information Processing Systems 29}, pages
  4592--4600. 2016.

\bibitem[Zhang et~al.(2020)Zhang, He, Sra, and Jadbabaie]{Zhang2020Why}
Jingzhao Zhang, Tianxing He, Suvrit Sra, and Ali Jadbabaie.
\newblock Why gradient clipping accelerates training: A theoretical
  justification for adaptivity.
\newblock In \emph{International Conference on Learning Representations}, 2020.

\bibitem[Zhang and He(2018)]{zhang2018convergence}
Siqi Zhang and Niao He.
\newblock On the convergence rate of stochastic mirror descent for nonsmooth
  nonconvex optimization.
\newblock \emph{arXiv preprint arXiv:1806.04781}, 2018.

\end{thebibliography}
